\documentclass{article}

\usepackage{amsmath,amsfonts,bm}

\def\eqref#1{equation~(\ref{#1})}
\def\Eqref#1{Equation~(\ref{#1})}

\def\1{\bm{1}}

\def\rd{{\textnormal{d}}}

\def\rp{{\textnormal{p}}}
\def\rq{{\textnormal{q}}}
\def\rr{{\textnormal{r}}}

\def\ry{{\textnormal{y}}}

\def\rvu{{\mathbf{i}}}

\def\rvu{{\mathbf{u}}}

\def\rvx{{\mathbf{x}}}
\def\rvy{{\mathbf{y}}}
\def\rvz{{\mathbf{z}}}

\def\rmA{{\mathbf{A}}}

\def\rmC{{\mathbf{C}}}
\def\rmD{{\mathbf{D}}}
\def\rmE{{\mathbf{E}}}
\def\rmF{{\mathbf{F}}}
\def\rmG{{\mathbf{G}}}

\def\rmI{{\mathbf{I}}}

\def\rmM{{\mathbf{M}}}

\def\rmS{{\mathbf{S}}}

\def\rmX{{\mathbf{X}}}

\def\rmZ{{\mathbf{Z}}}

\def\vzero{{\bm{0}}}

\def\vmu{{\bm{\mu}}}
\def\vtheta{{\bm{\theta}}}

\def\ve{{\bm{e}}}

\def\vx{{\bm{x}}}

\def\evtheta{{\theta}}

\def\mA{{\bm{A}}}

\def\mI{{\bm{I}}}

\def\mX{{\bm{X}}}

\DeclareMathAlphabet{\mathsfit}{\encodingdefault}{\sfdefault}{m}{sl}
\SetMathAlphabet{\mathsfit}{bold}{\encodingdefault}{\sfdefault}{bx}{n}

\def\gN{{\mathcal{N}}}
\def\gO{{\mathcal{O}}}

\def\gS{{\mathcal{S}}}

\def\gX{{\mathcal{X}}}
\def\gY{{\mathcal{Y}}}

\def\sP{{\mathbb{P}}}

\newcommand{\E}{\mathbb{E}}

\newcommand{\R}{\mathbb{R}}

\DeclareMathOperator*{\argmin}{arg\,min}
\DeclareMathOperator*{\argtop}{arg\,top}

\DeclareMathOperator{\sign}{sign}
\DeclareMathOperator{\Tr}{Tr}

\newcommand{\bzero}{\bm{0}}
\newcommand{\cov}{\bm{\Sigma}}
\newcommand{\rveps}{\bm{\varepsilon}}

\newcommand{\hry}{\hat{\ry}}
\newcommand{\hrvy}{\hat{\rvy}}
\newcommand{\hvtheta}{\hat{\vtheta}}
\newcommand{\hevtheta}{\hat{\evtheta}}

\newcommand{\tk}{\tilde{k}}
\newcommand{\try}{\tilde{\ry}}
\newcommand{\trvy}{\tilde{\rvy}}

\newcommand{\tvtheta}{\tilde{\vtheta}}
\newcommand{\tlambda}{\tilde{\lambda}}

\newcommand{\bgS}{\bar{\gS}}

\def\ubar#1{\underline{#1}}
\def\abar#1{\overline{#1}}

\newcommand{\crvy}{\check{\rvy}}

\newcommand{\svtheta}{\vtheta^\star}
\newcommand{\sevtheta}{\evtheta^\star}
\newcommand{\sk}{k^\star}

\newcommand{\sevthetas}{\evtheta^{\star 2}}
\newcommand{\svthetat}{\vtheta^{\star \top}}
\newcommand{\hvthetap}{\hvtheta_{\tn{post}}}

\newcommand{\deq}{\overset{\mathrm{d}}{=}}

\newcommand{\XX}{\rmX\rmX^\top}
\newcommand{\XXinv}{\pts{\XX}^{-1}}
\newcommand{\XXXty}{\rmX^\top\XXinv\trvy}
\newcommand{\XXXhy}{\rmX^\top\XXinv\hrvy}
\newcommand{\XXXcovXXX}{\XXinv \rmX \cov \rmX^\top \XXinv}

\newcommand{\ZcovZ}{\rmZ \cov \rmZ^\top}

\newcommand{\ZcovtZ}{\rmZ \cov^2 \rmZ^\top}
\newcommand{\yynorm}{\norm{\hrvy-\trvy}_2^2}
\newcommand{\var}{\|\cov^{\frac{1}{2}}\svtheta\|_2^2}
\newcommand{\varbar}{\|\cov^{\frac{1}{2}}\bar{\vtheta}\|_2^2}
\newcommand{\tserror}{\E_\rvx \pts{\rvx^\top\hvtheta-\rvx^\top\tvtheta}^2}
\newcommand{\rtserror}{\E_{\svtheta,\rvx} \pts{\rvx^\top\hvtheta-\rvx^\top\tvtheta}^2}

\newcommand{\SU}{\mathsf{SU}}
\newcommand{\CN}{\mathsf{CN}}

\newcommand{\eg}{\textit{e.g.,} }
\newcommand{\ie}{\textit{i.e.,} }
\newcommand{\iid}{\textit{i.i.d.}~}

\newcommand{\pts}[1]{\left(#1\right)}
\newcommand{\mts}[1]{\left[#1\right]}
\newcommand{\bts}[1]{\left\{#1\right\}}
\newcommand{\cond}[4]{\left\{ \begin{matrix}
      #1 &\; #2 \\
      #3 &\; #4
    \end{matrix}\right.\,}
\newcommand{\condd}[6]{\left\{ \begin{matrix}
      #1 &\; #2 \\
      #3 &\; #4 \\
      #5 &\; #6
    \end{matrix}\right.\,}
\newcommand{\condds}[6]{\left\{ \begin{matrix}
      #1 &\; #2 \\[10pt]
      #3 &\; #4 \\[10pt]
      #5 &\; #6
    \end{matrix}\right.\,}
\newcommand{\conddd}[8]{\left\{ \begin{matrix}
      #1 &\; #2 \\
      #3 &\; #4 \\
      #5 &\; #6 \\
      #7 &\; #8
    \end{matrix}\right.\,}

\newcommand{\sgn}[1]{\sign\left(#1\right)}

\newcommand{\tn}[1]{\textnormal{#1}}
\newcommand{\diag}[1]{\textnormal{diag}\left(#1\right)}

\newcommand{\norm}[1]{\left\|#1\right\|}
\newcommand{\risk}[1]{L(#1)}

\newcommand{\myquad}[1][1]{\hspace*{#1em}\ignorespaces}

\usepackage[utf8]{inputenc} %
\usepackage[T1]{fontenc}    %
\usepackage{algorithm}      %
\usepackage{algpseudocode}  %
\usepackage{amsfonts}       %
\usepackage{amsmath}        %
\usepackage{amsthm}         %
\usepackage{bbding}         %
\usepackage{bbm}            %
\usepackage{booktabs}       %
\usepackage{bm}             %
\usepackage{caption}        %
\usepackage{colortbl}       %
\usepackage{comment}        %
\usepackage{enumitem}       %
\usepackage{footnote}       %
\usepackage{fullpage}       %
\usepackage{graphicx}       %
\usepackage{mathtools}      %
\usepackage{microtype}      %
\usepackage{minitoc}        %
\usepackage{multirow}       %
\usepackage{nicefrac}       %
\usepackage{pifont}         %
\usepackage{soul}           %
\usepackage{subcaption}     %
\usepackage[table, x11names]{xcolor} %
\usepackage{tcolorbox}      %
\usepackage{url}            %
\usepackage{wrapfig}        %

\usepackage[style=authoryear,maxcitenames=2,maxbibnames=99,backend=biber,sorting=nyt,natbib=true,backref=true,uniquelist=false]{biblatex}
\usepackage[colorlinks, citecolor=blue]{hyperref} %
\usepackage{footnotebackref}

\newtheorem{theorem}{Theorem}
\newtheorem{assumption}[theorem]{Assumption}
\newtheorem{corollary}[theorem]{Corollary}
\newtheorem{definition}[theorem]{Definition}
\newtheorem{lemma}[theorem]{Lemma}

\makesavenoteenv{tabular}

\makeatletter
\newcommand\footnoteref[1]{\protected@xdef\@thefnmark{\ref{#1}}\@footnotemark}
\makeatother

\begin{document}

\doparttoc %
\faketableofcontents %

\begin{center}
    {\bf{\LARGE{Task Shift: From Classification to Regression\\in Overparameterized Linear Models}}}
    
    \vspace*{.2in}
    
    \renewcommand{\thefootnote}{\fnsymbol{footnote}}
    {\large{
    \begin{tabular}{ccc}
    Tyler LaBonte$^{1}$\footnote{\label{foot:foot1}Equal contribution; co-first author.} & Kuo-Wei Lai$^2$\footnoteref{foot:foot1} & Vidya Muthukumar$^{2,1}$
    \end{tabular}}}
    
    \vspace*{.2in}
    
    \renewcommand{\thefootnote}{\arabic{footnote}}
    \setcounter{footnote}{0}
    \begin{tabular}{c}
        $^1$H. Milton Stewart School of Industrial and Systems Engineering, Georgia Institute of Technology \\
        $^2$School of Electrical and Computer Engineering, Georgia Institute of Technology
        \\ \texttt{\{tlabonte, klai36, vmuthukumar8\}@gatech.edu}
    \end{tabular}
\vspace*{.2in}
\end{center}

\begin{abstract}
Modern machine learning methods have recently demonstrated remarkable capability to generalize under \emph{task shift}, where latent knowledge is transferred to a different, often more difficult, task under a similar data distribution.
We investigate this phenomenon in an overparameterized linear regression setting where the task shifts from classification during training to regression during evaluation.
In the zero-shot case, wherein no regression data is available, we prove that task shift is impossible in both sparse signal and random signal models for any Gaussian covariate distribution.
In the few-shot case, wherein limited regression data is available, we propose a simple postprocessing algorithm which asymptotically recovers the ground-truth predictor.
Our analysis leverages a fine-grained characterization of individual parameters arising from minimum-norm interpolation which may be of independent interest.
Our results show that while minimum-norm interpolators for classification cannot transfer to regression \emph{a priori}, they experience surprisingly structured attenuation which enables successful task shift with limited additional data.
\end{abstract}

\section{Introduction}
The fields of modern statistics and machine learning aim to develop models which generalize to a plethora of application-specific tasks.
For example, tasks in computer vision range from classifying images into discrete categories to object detection~\citep{ren2015faster}, segmentation~\citep{ronneberger2015unet}, and pose estimation~\citep{cao2017realtime}, while tasks in language modeling could be as basic as next-token prediction~\citep{vaswani2017attention}, or involve summarization~\citep{liu2019text} or machine translation~\citep{bahdanau2015neural}.
In statistics, basic estimation tasks involve either classification or regression; in the latter we wish to predict real-valued quantities and performance is measured via a continuous error metric.
The traditional perspective on task shift establishes a clear hierarchy in difficulty, \eg a statistical estimator which achieves a certain error rate for a regression task will typically achieve an equal or better rate on the corresponding classification task\footnote{This is most directly used in applying logistic regression procedures to classification tasks, but also works for, \eg least-squares regression~\citep{kline2005revisiting,rifkin2003regularized}.}.
Similarly, in empirical machine learning, the most difficult task is considered to be representation learning.
Indeed, learned representations are commonly finetuned on simpler downstream tasks and observed to generalize in a \emph{zero-shot} or \emph{few-shot} sense, \ie when finetuning data is unavailable or limited, respectively~\citep{bengio2013representation}.

Perhaps more surprising are recent trends in modern machine learning which appear to go in the other direction: using models trained on an ``easier'' task to successfully solve a ``harder'' task.
Specifically, large language models (LLMs) have shown a remarkable ability to generalize \emph{in-context} --- without explicit finetuning --- to completing prompt-response pairs despite being trained only on the more basic next-token prediction task~\citep{brown2020language}.
From a statistical perspective, a particularly intriguing observation is that LLMs trained on next-token prediction can successfully solve linear regression tasks by computing the ordinary least-squares estimate in-context~\citep{zhang2024trained}.
With this high-level motivation, we propose the following statistical learning problem formulation:

\emph{Can estimators trained on a classification task generalize, in a zero-shot or few-shot sense, to the regression task on the same data distribution?}

\begin{figure}
    \centering
    \begin{subfigure}[b]{0.22\textwidth}
        \hspace{1.5mm}
        \includegraphics[width=0.9\linewidth]{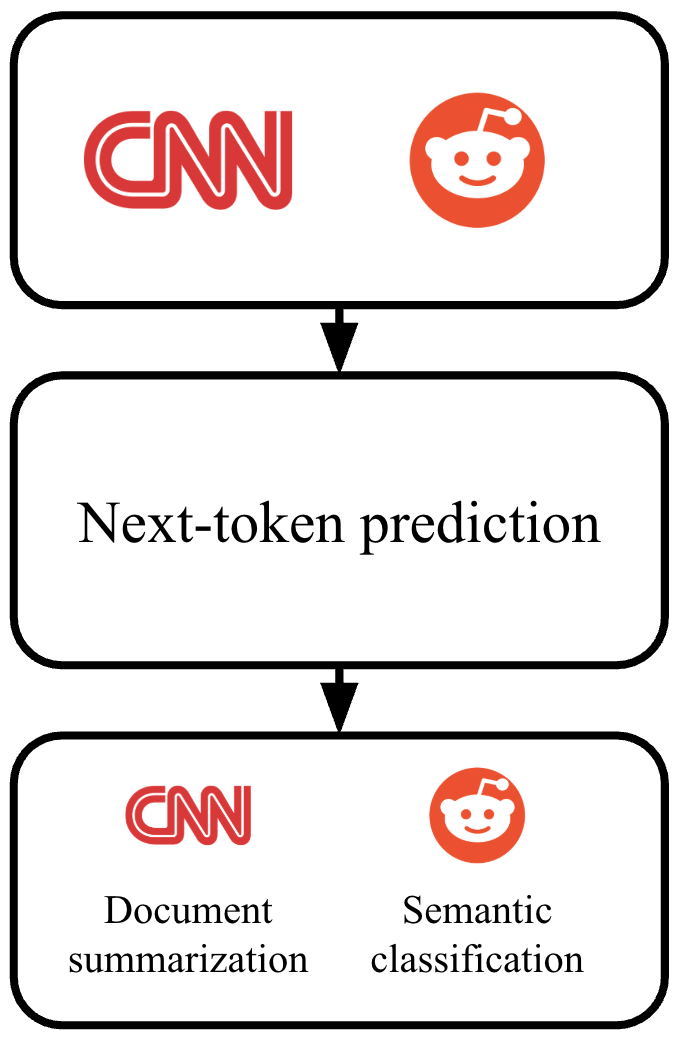}
        \subcaption{Language modeling}
    \end{subfigure}
    \begin{subfigure}{0.22\textwidth}
        \hspace{1.5mm}
        \includegraphics[width=0.9\linewidth]{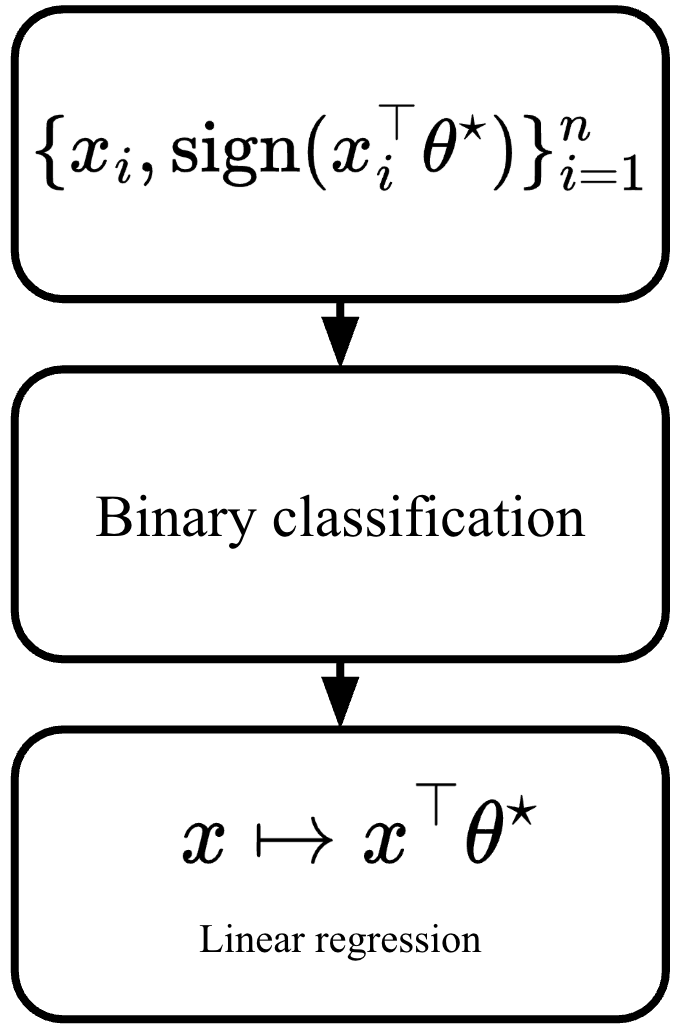}
        \subcaption{Statistical estimation}
    \end{subfigure}
    \caption{\textbf{Task shift in language modeling and statistical estimation.} In our task shift setting, latent knowledge is transferred between tasks under a similar conditional distribution or ground-truth signal. Task shift is compelling when the aim is to shift to a fundamentally harder task, with little to no data available from the new task.
    }
    \label{fig:taskshift}
\end{figure}

\paragraph{Our contributions.} We consider linear binary classification on data $\bts{\pts{\rvx_i, \hat{y}_i = \sgn{\rvx_i^\top \svtheta}}}_{i=1}^n$ and investigate whether an estimator trained on the classification task can generalize to the corresponding regression task, \ie predict the regression label $\rvx^\top \svtheta$ of a new datum $\rvx$.
We consider the \emph{overparameterized} regime wherein the dimensionality $d$ of the data greatly exceeds the number of training examples $n$, and we study the minimum $\ell_2$-norm interpolator (MNI) on the binary labels $\{\hat{y}_i\}_{i=1}^n$, which we denote by $\hvtheta$.
We define the \emph{task shift error} as the difference between the regression risk of the classification MNI $\hvtheta$ and the regression risk of the regression MNI, which we denote by $\tvtheta$.
We show the following results with high probability over the training data:
\begin{itemize}
    \item Classification data attenuates the signal $\svtheta$ even in the most favorable possible situation for either task in minimum $\ell_2$-norm interpolation, \ie maximally anisotropic data covariance, sublinearly sparse signal, and the existence of benign overfitting of noise. 
    Therefore, the classification MNI $\hvtheta$ \emph{does not successfully generalize in a zero-shot sense to regression data}, except when effective signal magnitude is equal to a specific, pre-defined constant.
    See Theorem \ref{thm:zero_shot_general_vacuous} for a formal statement of this result.
    \item We also produce an ``ansatz'' prediction of task shift error for more general signal models under a simplifying assumption on the regression labels (Theorem \ref{thm:benign}).
    Corollary~\ref{thm:benign_lower} then shows a fundamental \emph{tradeoff between regression bias and task shift error} --- they cannot be simultaneously statistically consistent.
    Moreover, we show in Theorem \ref{thm:benign_dense} that while a ``dense'' random signal is known to suffer from poor bias, it achieves vanishing task shift error if the covariance matrix has large effective rank compared to $n$, \ie its eigenvalues decay sufficiently slowly.
    \item Finally, we consider a $t$-sparse ground truth $\svtheta$ and propose a simple postprocessing algorithm utilizing few-shot regression data.
    We show the attenuation of the classification MNI is surprisingly structured, culminating in Theorem \ref{thm:support_identification} which proves the support of $\svtheta$ can be recovered simply by the $t$ largest elements (in absolute value) of $\hvtheta$.
    Our postprocessing algorithm ensures $\gO\pts{\frac{t}{m}}$ regression error with $m$ noisy regression examples or exact recovery from $t$ noiseless examples --- in other words, \emph{successful few-shot generalization from classification to regression}.
\end{itemize}

Our techniques build on the recent literature on benign overfitting of minimum $\ell_2$-norm interpolators in both regression~\citep{bartlett2020benign} and classification tasks~\citep{muthukumar2021classification,wang2023benign}.
We are especially inspired by the separation in statistical consistency derived in~\cite{muthukumar2021classification}, which showed that for certain anisotropic ensembles, classification may generalize while regression does not.
We substantially develop their tools to provide a fine-grained characterization of the individual magnitudes $\{|\hat{\theta}_j|\}_{j = 1}^d$.
Surprisingly, this characterization can enable the success of our few-shot algorithm even when the minimum $\ell_2$-norm interpolator would not generalize for either classification or regression.

\subsection{Related work}\label{sec:relatedwork}

The formulation of task shift in this paper, particularly our focus on the shift from classification to regression tasks, shares both similarities and differences with several prominent areas of research in machine learning. Compared to transfer learning~\citep{pan2009survey}, task shift similarly aims to generalize knowledge from one ``job'' to another. However, while transfer learning emphasizes preserving useful features for similar or downstream tasks, task shift focuses on generalizing from a simpler task (\eg classification) to a more complex one (\eg regression). Task shift is also related to the concept of distribution shift~\citep{moreno2012unifying}. In task shift, the source data distribution remains unchanged, but the conditional distributions of the labels differ at test time. Furthermore, our classification-to-regression setup is closely connected to the one-bit compressive sensing problem~\citep{boufounos2008one}. However, while one-bit compressive sensing focuses on optimal estimators or algorithms based on a known measurement process, our work emphasizes unique properties of the $\ell_2$-inductive bias.

The theoretical analysis in this work builds directly on the literature on benign overfitting. \cite{bartlett2020benign}~\cite{tsigler2023benign} characterized benign overfitting for regression estimators, while~\cite{muthukumar2021classification} provided a survival and contamination analysis for sparse signals in both classification and regression settings. Our analysis substantially develops insights from these works to estimate the support of a sparse signal even when neither classification nor regression generalizes.

We discuss additional related work in Appendix~\ref{app:expanded_related_work}.

\section{Preliminaries} \label{sec:preliminaries}
We now detail our classification-to-regression task shift setting and introduce our proposed estimators and assumptions on data covariance.
We also recap the definitions of \emph{effective dimension} quantities extensively used in analysis of interpolating estimators, \eg by~\cite{bartlett2020benign}.

\paragraph{Notation.}
We use uppercase bold symbols to denote matrices (\eg $\mX$), lowercase bold symbols to denote vectors (\eg $\vx$), and lowercase italicized symbols to denote scalars (\eg $x$).
We denote random variables using non-italicized symbols, \eg $\rmX$ for a random matrix and $\rvx$ for a random vector.
Let $\sP_\rvx$ and $\E_\rvx$ denote a probability and expectation with respect to a random vector $\rvx$, respectively.

Let $\deq$ denote equality in distribution.
Let $\gS^\mathsf{c}$ denote the complement of a set $\gS$.
Let $\rvx\sim\gN(\vmu, \cov)$ denote that $\rvx$ is sampled from a multivariate Gaussian distribution with mean $\vmu$ and covariance matrix $\cov$.
Let $\diag{\alpha_1,\dots,\alpha_n}$ denote an $n\times n$ matrix with $\alpha_1,\dots,\alpha_n$ on the diagonal and zeroes elsewhere.
Let $\mu_1(\mA),\mu_2(\mA),\dots$ denote the eigenvalues of a matrix $\mA$ in non-increasing order.
Let $\norm{\mA}\coloneqq \mu_1(\mA)$ and $\Tr(\mA)\coloneqq \sum_i \mu_i(\mA)$ denote the operator norm and trace of $\mA$.
For clarity, we use the shorthand $[t]\coloneqq \{1,2,\dots,t\}$ for $t\in\mathbb{N}$.
Finally, we use $C,c, c_1,\ldots$ to denote constants not depending on $n,d$ which can change from line to line.
\subsection{Minimum-norm interpolating estimators}
We consider a noiseless linear regression problem where data undergoes centered Gaussian featurization\footnote{We expect many of our results to generalize to the case of independent, sub-Gaussian features by building on~\cite{cao2021risk,tsigler2023benign}.} in $d$ dimensions such that $\rvx \sim \gN(\bzero, \cov)$.
We denote the ground-truth regressor, also called \emph{signal}, by $\svtheta\in \R^d$, so the regression and classification label for a datum $\rvx$ are $\try \coloneqq \rvx^\top\svtheta \in \R$ and $\hry \coloneqq \sgn{\rvx^\top\svtheta} \in \bts{\pm 1}$, respectively.

We assume noiseless labels in the classification dataset for ease of exposition,  but our zero-shot lower bounds and few-shot upper bounds can be easily extended to handle classification label noise\footnote{For example, if $\sP\big(\hat{y} =  \sgn{\rvx^\top \svtheta}\mid \rvx \big)= 1 - \nu^\star$, we require $\nu^\star < 0.5$, \ie the sign is preserved on average.}.

We assume access to a classification dataset $\bts{(\rvx_i, \hry_i)}_{i=1}^n$ where $d \gg n$ and $\lim_{n \to \infty} \frac{d}{n} = \infty$, \ie the data is heavily overparameterized.
We denote the data matrix by $\rmX \coloneqq [\rvx_1, \dots, \rvx_n]^\top \in \R^{n \times d}$ and the regression and classification label vectors by $\trvy \coloneqq [\try_1,\dots,\try_n]^\top \in \R^n$ and $\hrvy \coloneqq [\hry_1,\dots,\hry_n]^\top \in \bts{\pm 1}^n$.
We study the \emph{minimum-norm interpolator} (MNI), which is obtained directly via the implicit bias of gradient descent on the squared loss and also enjoys close links to the implicit bias of popular classification losses~\citep{hsu2022proliferation}.
The MNIs on regression and classification labels are defined by $\tvtheta \coloneqq \argmin \bts{\norm{\vtheta}_2:\rmX\vtheta=\trvy}$ and $\hvtheta \coloneqq \argmin \bts{\norm{\vtheta}_2:\sgn{\rmX\vtheta}=\hrvy}$, respectively.
Since $\rmX$ is full rank almost surely under Gaussian design, both estimators have simple closed forms: $\tvtheta = \XXXty$ and $\hvtheta = \XXXhy$.

The excess risk of any linear estimator $\vtheta\in \R^d$ is given by $\risk{\vtheta} \coloneqq \E_{\rvx} \pts{\rvx^\top\vtheta - \rvx^\top\svtheta}^2.$
The central goal of this paper is to bound the excess \emph{regression} risk of the MNI \emph{classifier} $\risk{\hvtheta}$ --- and, later, postprocessed variants of this classifier.
We say $\hvtheta$ \emph{achieves task shift} if it is regression-consistent, \ie$\lim_{n,d\to\infty}\risk{\hvtheta}=0$.
\subsection{Effective rank and covariance structure}
As shown in~\cite{bartlett2020benign, muthukumar2021classification, tsigler2023benign}, the performance of the regression MNI on regression labels or classification MNI on classification labels can be characterized by two notions of \emph{effective rank} of the spectrum of the data covariance matrix $\cov$. We let $\lambda_1,\dots,\lambda_d$ denote the eigenvalues of $\cov$ in non-increasing order (\ie $\lambda_j\coloneqq \mu_j(\cov)$ for all $j \in [d]$).
\begin{definition}[Effective rank]\label{def:effective_rank}
    For an index $k\geq 0$, two notions of the \emph{effective rank} of $\cov$ are
    \begin{equation*}
    r_k(\cov) \coloneqq \frac{\sum_{j=k+1}^d \lambda_j}{\lambda_{k+1}}
     \quad \tn{and}\quad 
    R_k(\cov) \coloneqq \frac{\pts{\sum_{j=k+1}^d \lambda_j}^2}{\sum_{j=k+1}^d \lambda_j^2}.
    \end{equation*}
\end{definition}
These are essentially two notions of effective dimension of a ``tail'' component of the covariance matrix restricted to eigenvalues $\lambda_{k+1},\ldots,\lambda_d$.
For a constant $b>1$, denote $\sk \coloneqq \min \{k\geq 0: r_k(\cov) \geq bn\}$
where the minimum of the empty set is defined as $\infty$.
In other words, $\sk$ is the minimal index after which the spectrum has large (first) effective rank compared to $n$.
We will make the following assumptions on $\cov$.
\begin{assumption}\label{asm:rank}
    We assume that $\cov$ is diagonal and positive definite for any $d, n < \infty$. We also assume that $\sk \leq \frac{n}{c}$ for some universal constant $c > 0$.
\end{assumption}
The diagonal assumption is without loss of generality for our zero-shot lower bounds but required for our few-shot upper bounds.
The assumption on effective rank essentially requires a ``long tail'' on the covariance and ensures that the data is actually high dimensional in nature.

We situate our results in two covariance ensembles considered in~\cite{bartlett2020benign, muthukumar2021classification} which enable us to state precise rates.
\begin{definition}[Spiked covariance matrix]\label{def:spiked}
    A \emph{spiked} covariance matrix $\cov$ is parameterized by $p \in (1,\infty)$, $r \in [0,1)$, and $q \in (0, p-r)$.
    We set the data dimension to $d = n^p$, the length of the ``spike'' to $s = n^r$, and $a = n^{-q}$ to be a parameter controlling the ratio of the eigenvalues. 
    Then, the ensemble is defined by
    \begin{equation*}
        \lambda_j \coloneqq
        \begin{cases}
            \frac{ad}{s} & j\in[s] \\
            \frac{(1-a)d}{d-s} & \tn{otherwise.}
        \end{cases}
    \end{equation*}
    This ensemble satisfies Assumption~\ref{asm:rank} with $\sk = s$ when $q<1-r$ and $\sk=0$ when $q>1-r$.
\end{definition}

\begin{definition}[Polynomial decay covariance matrix]\label{def:polynomial}
    A \emph{polynomial decay} covariance matrix $\cov$ is parameterized by $p\in(1, \infty)$ and $u, v\geq 0$ such that
    \begin{equation*}
        \lambda_j \coloneqq j^{-u} \ln^{-v}(j+1)
    \end{equation*}
    and $d=n^p$. We consider two versions of the covariance. We set $u=1$, $v=2$ to study an instance of the regime wherein benign overfitting is achieved, and we set $u\in[0,1)$, $v=0$ to study a case wherein it is not achieved~\citep{bartlett2020benign}.
    These parameterizations satisfy Assumption~\ref{asm:rank}, but unlike the spiked covariance matrix, characterizing $\sk$ is rather delicate.
    In particular, if $\sk$ is nonzero we may only be able to characterize its order (see Appendix~\ref{app:poly_cov}).
\end{definition}

Formally, we take limits over a sequence of covariance ensembles $\{\cov_n\}_{n=1}^\infty$ (\eg Theorem~\ref{thm:zero_shot_general_vacuous}, Corollaries~\ref{cor:diff_cov} and \ref{thm:benign_lower}), but we drop the subscript $n$ for clarity.

\section{Zero-shot task shift} \label{sec:zeroshot}

In this section, we study task shift performance in the zero-shot setting, wherein no regression data is available.
We show that task shift is impossible under sparse and random signal models with maximally favorable data covariance.
While our negative results are perhaps expected --- since the magnitude of the regression labels is irrevocably lost in the classification task --- our analysis leads to some unexpected conclusions.
Specifically, in Section \ref{sec:zeroshot_sparse}, we show that the attenuation of the classification MNI is surprisingly structured (which enables recovery of the true signal up to a magnitude factor), and in Section \ref{sec:zeroshot_dense}, we show that the nature of the failure of zero-shot task shift is closely linked to the index $\sk$.
\subsection{The case of sparse signal} \label{sec:zeroshot_sparse}
We first consider a signal with \emph{sublinear sparsity}.
\begin{assumption}[$t$-sparse signal model]\label{asm:k_sparse}
    Denote the support of the signal $\svtheta$ by $\gS\coloneqq\{j\in[d]\mid\sevtheta_j \neq 0\}$.
    We assume $\svtheta$ is t\emph{-sparse}, that is $|\gS| = t \ll n^{\frac{1}{2}-\epsilon} \ll d$, where $\epsilon\in(0, \frac{1}{4})$. We write $\sevtheta_j\coloneqq\frac{a_j}{\sqrt{\lambda_j}}$ for $j \in \gS$ with $a_j \neq 0$.
    Moreover, we assume the total signal strength $\var=\sum_{j\in\gS} a_j^2$ is constant for all $n$.\footnote{The assumption on total signal strength is only necessary for limits over over $\{\cov_n\}_{n=1}^\infty$, \ie it is not required for our few-shot results in Section~\ref{sec:fewshot}.}
\end{assumption}
We note that the sparsity assumption is justified since in its absence, regression generalization is information-theoretically impossible in the overparameterized regime~\citep{wainwright2009information,aeron2010information}.
For the regression MNI,~\cite{tsigler2023benign} showed one must have approximate sparsity in the direction of the top eigenvalues to ensure low bias, \ie sparsity is required for statistical consistency.

Rather than the standard bias-variance decomposition, we directly investigate the relative preservation of the true signal (\emph{survival}) and the pollution of false signal (\emph{contamination}), which were shown in~\cite{muthukumar2021classification} to tightly characterize regression and classification tasks.
We define these quantities formally below, using shorthand notation $\rmA\coloneqq\XX = \sum_{j=1}^d \lambda_j \rvz_j\rvz_j^\top$ where $\rvz_j \coloneqq \rmX\ve_j/\sqrt{\lambda_j}$ and $\rvz_j\sim\gN(\bzero, \rmI_n)$ is an independent isotropic Gaussian vector for $j\in[d]$.
\begin{definition}[Survival and contamination]\label{def:su_cn}
    Under Assumptions~\ref{asm:rank} and~\ref{asm:k_sparse}, we define the \emph{survival} $\SU_j$ for all $j \in \gS$ and \emph{contamination} $\CN$ by
    \begin{align*}
        \SU_j &\coloneqq \frac{\hevtheta_j}{\sevtheta_j} = \frac{\ve_j^\top \XXXhy}{\frac{a_j}{\sqrt{\lambda_j}}} = \frac{\lambda_j}{a_j}\rvz_j^\top\rmA^{-1}\hrvy,\\
        \CN &\coloneqq \sqrt{\sum_{j\in\gS^\mathsf{c}}\lambda_j\hevtheta_j^2}.
    \end{align*}
\end{definition}
Intuitively, we desire survival to be close to one and contamination to be small.
In the remainder of this section, we substantially generalize the analysis in~\cite{muthukumar2021classification,wang2023benign}, which assumed $t=1$ and spiked covariance, to $t$-sparse and general covariance models.
Using Assumption~\ref{asm:k_sparse} and Gaussian design, we may write
\begin{align*}
    \trvy &\coloneqq \rmX\svtheta = \sum_{j\in\gS}\sqrt{\lambda_j}\sevtheta_j\rvz_j = \sum_{j\in\gS} a_j\rvz_j, \\
    \hrvy &\coloneqq \sgn{\rmX\svtheta} = \text{sign}\Big(\sum_{j\in\gS} a_j\rvz_j\Big).
\end{align*}
This expression enables us to perform careful leave-$t$-out analyses and decouple the survival and contamination terms in the following key lemma.
\begin{lemma}\label{lem:k_sparse_SU_CN_bounds}%
Define $\bgS \coloneqq \gS \cup [\sk]$ and denote by $\{\tlambda_j\}_{j=1}^{d - |\bgS|}$ the diagonal entries of the matrix $\cov_{-\bgS}$, \ie $\cov$ with rows and columns indexed by $\bgS$ left out, and $\rmA_{-\bgS} = \sum_{j=1}^{d- |\bgS|} \tlambda_j \rvz_j\rvz_j^\top$.
Under Assumptions~\ref{asm:rank} and~\ref{asm:k_sparse} and for large enough $n$, we have for all $j \in \gS$,
     \begin{align*}
         \SU_j &\leq \sqrt{\frac{2}{\pi\var}} \pts{\frac{ \pts{1 + \frac{c_1t}{n^{\frac{1}{2}-\epsilon}}}\lambda_{j}\Tr\pts{\rmA_{-\bgS}^{-1}}}{1 + \pts{1 - \frac{c_2}{n^{\frac{1}{2}-\epsilon}}}\lambda_{j}\Tr\pts{\rmA_{-\bgS}^{-1}}}} \leq \sqrt{\frac{2}{\pi\var}} \cdot \frac{\lambda_j\pts{\frac{c_3n}{\tlambda_{1}r_{0}(\cov_{-\bgS})}}}{1+\lambda_j\pts{\frac{n}{c_4\tlambda_{1}r_{0}(\cov_{-\bgS})}}}, \\
         \SU_j &\geq \sqrt{\frac{2}{\pi\var}} \pts{\frac{ \pts{1 - \frac{c_5t}{n^{\frac{1}{2}-\epsilon}}}\lambda_{j}\Tr\pts{\rmA_{-\bgS}^{-1}}}{1 + \pts{1 + \frac{c_6}{n^{\frac{1}{2}-\epsilon}}}\lambda_{j}\Tr\pts{\rmA_{-\bgS}^{-1}}}} \geq \sqrt{\frac{2}{\pi\var}} \cdot \frac{\lambda_j\pts{\frac{n}{c_7\tlambda_{1}r_{0}(\cov_{-\bgS})}}}{1+\lambda_j\pts{\frac{c_8n}{\tlambda_{1}r_{0}(\cov_{-\bgS})}}},
     \end{align*}
    \color{black}
    with probability at least $1-cte^{-n^{2\epsilon}}$. When $n,d\to\infty$, the limit converges as
    \begin{align*}
        \lim_{n,d\to\infty} \SU_j = \sqrt{\frac{2}{\pi\var}}\cdot \frac{\lambda_j\Tr\pts{\rmA_{-\bgS}^{-1}}}{1 + \lambda_j\Tr\pts{\rmA_{-\bgS}^{-1}}}
    \end{align*}
    almost surely.
    Moreover, we have
    \begin{align*}
        \CN &\leq c t \sqrt{\pts{\frac{\sk}{n} + \frac{n}{R_0\pts{\cov_{-\bgS}}}}\ln\pts{n}}
    \end{align*}
    with probability at least $1-\frac{ct}{n}$. 
\end{lemma}
Lemma~\ref{lem:k_sparse_SU_CN_bounds} is proved in Appendix~\ref{app:su_cn_bound_proof}.
In addition to being useful for the results of this section, Lemma~\ref{lem:k_sparse_SU_CN_bounds} is also used for our postprocessing algorithm which has access to few-shot data in Section~\ref{sec:fewshot}.
Equipped with bounds of survival and contamination, we can now relate the classification MNI excess risk with $\{\SU_j\}_{j \in \gS}$ and $\CN$. Under Assumptions~\ref{asm:rank} and~\ref{asm:k_sparse}, we can write the excess risk of the classification MNI as
\begin{align}
    \risk{\hvtheta} &= \E_{\rvx}\mts{\pts{\rvx^\top\hvtheta - \rvx^\top\svtheta}^2}\nonumber \\
    &=\sum_{j=1}^d\lambda_j\pts{\hevtheta_j - \sevtheta_j}^2\nonumber\\
    &=\sum_{j\in\gS}\lambda_j\sevthetas_j\pts{\frac{\hevtheta_j}{\sevtheta_j} - 1}^2 + \sum_{j\in \gS^\mathsf{c}} \lambda_j\hevtheta_j^2\nonumber\\
    &= \sum_{j\in\gS} a_j^2\pts{\SU_j - 1}^2 + \CN^2\label{eq:risk_su_cn},
\end{align}
where we substitute the expressions for $\sevtheta_j$ from Assumption~\ref{asm:k_sparse} and $\{\SU_j\}_{j \in \gS}$ and $\CN$ from Definition~\ref{def:su_cn}. 
Note that this decomposition is valid for any estimator of $\svtheta$ (not just the classification MNI $\hvtheta$).
Substituting the bounds on survival and contamination from Lemma~\ref{lem:k_sparse_SU_CN_bounds} yields our main result of this section.
\begin{theorem}\label{thm:zero_shot_general_vacuous}
    Define $b_j \coloneqq \underset{n,d\to\infty}{\lim}\lambda_j\Tr\pts{\rmA_{-\bgS}^{-1}}$.
    Under Assumptions~\ref{asm:rank} and~\ref{asm:k_sparse}, for any covariance $\cov$ satisfying $\underset{n,d\to\infty}{\lim}\frac{t^2\cdot \sk \cdot \ln\pts{n}}{n}=\underset{n,d\to\infty}{\lim}\frac{t^2\cdot n \cdot \ln\pts{n}}{R_0\pts{\cov_{-\bgS}}}=0$, we have
    \begin{align*}
        &\lim_{n,d\to\infty}\risk{\hvtheta} = \sum_{j\in\gS}a_j^2\pts{\sqrt{\frac{2}{\pi \var}}\frac{b_j}{1+b_j} - 1}^2
    \end{align*}
    almost surely.
\end{theorem}
Theorem~\ref{thm:zero_shot_general_vacuous} is proved in Appendix~\ref{app:proof_zero_shot_general_vacuous}. It shows that, even for data covariances that satisfy benign overfitting of noise in linear regression~\citep{bartlett2020benign}, perfect survival of signal is required for consistent task shift.
Notably, it is not possible for the classification MNI $\hvtheta$ to satisfy regression consistency for all possible magnitudes of the ground truth, \ie for all possible values of $\var$. 
This is because the coefficients $\{b_j\}_{j\in\gS}$ in Theorem~\ref{thm:zero_shot_general_vacuous} clearly do not depend on $\svtheta$.
We can, however, ask the more specialized question (posed in one-bit compressive sensing, \eg\cite{plan2012robust}) of whether it is possible to generalize on all signals of a specific magnitude.
Theorem~\ref{thm:zero_shot_general_vacuous} shows this will be the case if and only if $b_j \to \infty$ for all $j \in \gS$ and $\var = \frac{2}{\pi}$: a positive result in the flavor of one-bit compressive sensing.

As a corollary of Theorem~\ref{thm:zero_shot_general_vacuous}, we present characterizations of zero-shot task shift for the spiked covariance model (Definition~\ref{def:spiked}).
We present similar results for the more delicate polynomial decay covariance model (Definition~\ref{def:polynomial}) in Appendix~\ref{app:poly_cov}~(Corollary~\ref{cor:poly_cov}).
\begin{corollary}\label{cor:diff_cov} %
    Suppose Assumptions~\ref{asm:rank} and~\ref{asm:k_sparse} are satisfied with $t \ll \min\Big\{\sqrt{\frac{n}{\sk\cdot\ln\pts{n}}}, \sqrt{\frac{n^{p-1}}{\ln\pts{n}}}\Big\}$.
    Under the spiked covariance model (Definition~\ref{def:spiked}),
    \begin{itemize}
         \item For $q < 1-r$, we have
         \begin{equation*}
              \lim_{n,d\to\infty}\risk{\hvtheta} = \sum_{j\leq s,j\in\gS}a_j^2\pts{\sqrt{\frac{2}{\pi \var}} - 1}^2+\sum_{j>s,j\in\gS}a_j^2.
         \end{equation*}
        This implies regression consistency \emph{if and only if} the signal magnitude is fixed at $\var = \frac{2}{\pi}$ and $a_j = 0$ for all $j \in \gS \cap \{s+1,\dots,d\}$, \ie the signal is only supported within the ``spike''.
        The latter condition is also required for regression consistency of the regression MNI~\citep{tsigler2023benign}.
         \item For $q > 1-r$, we have
         \begin{align*}
             \lim_{n,d\to\infty}\risk{\hvtheta} = \sum_{j\in\gS}a_j^2.
         \end{align*}
        As in the case of the regression MNI~\citep{muthukumar2021classification}, regression consistency would not be possible unless we had zero signal, \ie $\svtheta = \bm{0}$.
    \end{itemize}
\end{corollary}
Corollary~\ref{cor:diff_cov} is proved in Appendix~\ref{app:diff_cov}. 

\subsection{The case of random signal} \label{sec:zeroshot_dense}
We now provide results for \emph{random signal models} which may be dense.
In this section, we study a general random signal model and introduce a simplifying ansatz which enables upper and lower bounding the task shift error.
In Appendix~\ref{app:dense} (Theorem~\ref{thm:benign_dense}), we show a more specific ``dense'' random signal model --- wherein $\svtheta$ has similar magnitude in all dimensions --- which does not require the simplifying ansatz.

The key idea is to interpret classification labels as regression labels under a dependent noise model and explicitly connect to characterizations of the regression MNI~\citep{bartlett2020benign,tsigler2023benign}.
For regression labels, one usually has additive sub-Gaussian noise, which means that $\try - \rvx^\top\svtheta$ is sub-Gaussian and conditionally independent given $\rvx$.
But for classification labels, we have
 \begin{equation*}
     \hry - \rvx^\top\svtheta \coloneqq \sgn{\rvx^\top\svtheta}-\rvx^\top\svtheta = \frac{\sgn{\rvx^\top\svtheta}-\rvx^\top\svtheta}{\rvx^\top\svtheta} \rvx^\top \svtheta,
 \end{equation*}
which is clearly dependent on $\rvx$. We will write $\hrvy-\trvy\coloneqq \rmD\rmX \svtheta$, where $\rmD\coloneqq\diag{\tn{d}_1,\dots,\tn{d}_n}$ and $\tn{d}_i\coloneqq\frac{\sgn{\rvx^\top\svtheta}-\rvx^\top\svtheta}{\rvx^\top\svtheta}$, to make this relationship explicit.

We begin with a decomposition of the regression risk of the classification MNI, proved in Appendix \ref{app:decomposition}.
\begin{lemma} \label{lem:decomposition}
    The regression risk of the classification MNI $\hvtheta$ can be decomposed as
     \begin{equation*}
         \risk{\hvtheta} = \risk{\tvtheta} + \tserror.
     \end{equation*}
\end{lemma}
Clearly, $\risk{\tvtheta}$ is the \emph{regression error} of the regression MNI: when there is no noise in regression labels this is equivalent to the \emph{bias}.
Likewise, we refer to $\E_\rvx (\rvx^\top\hvtheta-\rvx^\top\tvtheta)^2$ as the \emph{task shift error}, which can be interpreted as the ``variance'' under our dependent noise model.
 
The bias is a standard term, characterized as follows. 
\begin{lemma}\label{lem:bias}
    For any $\cov$, there exists a constant $c\geq 1$ such that the following hold.
    
    \paragraph{Upper bound.}~\emph{\citep[Lemma 35]{bartlett2020benign}}. For any $\svtheta$ (not necessarily random), we have
    \begin{equation*}
        \risk{\tvtheta} \leq \sqrt{c} \norm{\cov} \norm{\svtheta}_2^2 \max\pts{\sqrt{\frac{r_0(\cov)}{n}}, \frac{r_0(\cov)}{n}}
    \end{equation*}
    with probability at least $1-e^{-\frac{n}{c}}$.

    \paragraph{Lower bound.}~\emph{\citep[Lemma 8]{tsigler2023benign}}. Suppose random signal $\svtheta$ is generated from the ground truth $\bar{\vtheta}$ by $\theta^\star_j = \tn{r}_j \bar{\theta}_j$ where each $\tn{r}_j$ is an independent Rademacher random variable.
    We have 
    \begin{equation*}
        \E_{\svtheta} \risk{\tvtheta} \geq \frac{1}{c} \sum_{j=1}^d \frac{\lambda_j \bar{\theta}_j^2}{\pts{1+\frac{n\lambda_j}{\sum_{k=1}^d \lambda_k}}^2}
    \end{equation*}
    with probability at least $1 - ce^{-\frac{n}{c}}$.
\end{lemma}

We now provide a characterization of the task shift error under the simplifying ansatz that $\rmD = \alpha \mI$ for some $\alpha \neq 0$.
The interpretation for this assumption is that all regression labels have the same magnitude (say, equal to $R$), which would result in $\rd_i = \frac{1}{R} - 1 =: \alpha$; clearly $\alpha \neq 0$ except in the special case where $R = 1$.
Thus, we are considering regression problems that are, in essence, a scaled version of classification.
From a technical perspective, it is difficult to obtain closed-form bounds on the task shift error without the simplifying ansatz, as dependencies which arise for general $\rmD$ may invalidate certain concentration arguments.
Nevertheless, as we expect generic regression problems to be even harder than scaled classification, providing a lower bound even for this simpler setting is meaningful.
Our next theorem does precisely this, via an extension of benign overfitting techniques to our dependent noise model.
\begin{theorem} \label{thm:benign}
    For any $\cov$, there exist constants $c,c_1\geq 1$ such that the following hold.
    
    \paragraph{Upper bound.} If $\sk<\frac{n}{c_1}$, then for any $\svtheta$ (not necessarily random), we have
    \begin{equation*}
         \tserror \leq \\ cn\pts{\frac{\sk}{n}+\frac{n}{R_{\sk}(\cov)}}\var
    \end{equation*}
    with probability at least $1-18e^{-\frac{n}{c}}$.
    
    \paragraph{Lower bound.} Suppose $\rmD = \alpha \mI$ and $\svtheta$ is any random signal such that $\E\theta^{\star 2}_j\geq \sigma^2$ for all $j\in [d]$.
    If $\sk<\frac{n}{c_1}$, then
     \begin{equation*}
         \rtserror \geq \frac{\alpha^2\sigma^2}{c} \pts{\sum_{j=1}^{\sk} \lambda_j + \frac{n}{R_{\sk}(\cov)}\sum_{j=\sk+1}^d\lambda_j}
     \end{equation*}
    with probability at least $1-14e^{-\frac{n}{c}}$.
    On the other hand, if $\sk\geq \frac{n}{c_1}$, then
    \begin{equation*}
        \rtserror \geq \frac{\alpha^2\sigma^2}{c}
    \end{equation*}
    with probability at least $1-10e^{-\frac{n}{c}}$.
\end{theorem}
Theorem~\ref{thm:benign} is proved in Appendix~\ref{app:benign}.
An interesting consequence of Theorem~\ref{thm:benign} is that there do exist covariance ensembles for which the task shift error decays to zero --- implying that the classification and regression MNIs would generalize equivalently on a regression task!
However, these are ensembles for which $\sk=0$, and therefore regression bias would stay constant.
Ultimately, our results imply a fundamental tradeoff between bias error and task shift error for random signals, stated below.
\begin{corollary} \label{thm:benign_lower}
    For any sequence $\{\cov_n\}_{n=1}^\infty$, denote $\kappa\coloneqq\lim_{n,d\to\infty} \sk$.
    Suppose random signal $\svtheta$ is generated from the ground truth $\bar{\vtheta}$ by $\theta^\star_j = \tn{r}_j \bar{\theta}_j$, where each $\tn{r}_j$ is an independent Rademacher random variable, such that $\varbar$ is constant for all $n$.
    Assume that $\bar{\theta}_j^2\geq 1$ for all $1\leq j \leq d$.
    Then, under the same conditions as the lower bound of Theorem~\ref{thm:benign}, the almost sure limits of bias and task shift error are characterized in two distinct regimes:
    \begin{enumerate}
        \item $\kappa = 0$: %
        the limiting bias is nonzero, \ie
        \begin{equation*}
            \lim_{n,d\to\infty} \E_{\svtheta} \risk{\tvtheta} \geq \frac{\varbar}{c}.
        \end{equation*}
        \item $\kappa > 0$: %
        the limiting task shift error is nonzero, \ie
        \begin{equation*}
            \lim_{n,d\to\infty} \rtserror \geq \frac{\alpha^2}{c}.
        \end{equation*}
    \end{enumerate}
\end{corollary}
The proof of Corollary \ref{thm:benign_lower} is in Appendix \ref{app:benign_lower}.

\section{Few-shot task shift} \label{sec:fewshot}

In the previous sections, we demonstrated that task shift from classification to regression without any regression information is generally unachievable. 
Therefore, in this section, we investigate task shift in the \emph{few-shot} setting, where limited regression information is available.
We propose a simple two-step approach to recover a sparse signal $\svtheta$. In the first step, we leverage the structured attenuation of the classification MNI to recover the support of $\svtheta$.
Second, we perform least-squares regression with reduction to the dimensionality of the support to recover the magnitude of $\svtheta$.
Our results require a diagonal covariance matrix (Assumption~\ref{asm:rank}).

\begin{algorithm}[t]
\caption{Support recovery}\label{alg:recover_support}
\begin{algorithmic}
\Require $\rmX$, $\hrvy$, $(t \text{ or } \{\lambda_j\}_{j=1}^d)$
\State $\hvtheta \gets \XXXhy$
\State $\gS \gets \emptyset$
\If {$t$ is known}
    \State $\gS \gets \argtop t(|\hvtheta|)$%
\ElsIf {$\cov$ is known}
    \For {$j\in [d]$}
        \If {$|\hevtheta_j| \geq \lambda_j^{-\frac{1}{2}}$}
        \State $\gS \gets \gS\cup \{j\}$
        \EndIf
    \EndFor
\EndIf
\\\Return $\gS$
\end{algorithmic}
\end{algorithm}
\begin{algorithm}[h!]
\caption{Least-squares on recovered support}\label{alg:post_processing}
\begin{algorithmic}
\Require $\rmX$, $\hrvy$,  $\{\rvx^\prime_i,y^\prime_i\}_{i=1}^m$
\State $\hvtheta \gets \XXXhy$
\State $\gS \gets$ from Algorithm~\ref{alg:recover_support}
\State $\hvtheta_{\text{post}} \gets \argmin_{\vtheta\in\R^d} \sum_{i=1}^m\pts{\rvx_i^{\prime \top}\vtheta - y^\prime_i}^2$\\
\myquad[5]\text{  s.t. } $\evtheta_j = 0\quad  \forall j\in\gS^\mathsf{c}$\\ 
\end{algorithmic}
\end{algorithm}

\subsection{Support recovery via attenuation}
The survival and contamination bounds of Lemma~\ref{lem:k_sparse_SU_CN_bounds} show that while the classification MNI $\hvtheta$ attenuates the sparse signal $\svtheta$, it does so in a highly structured manner.
This suggests it is possible to distinguish the support components of $\svtheta$ using the relative magnitudes of entries of $\hvtheta$.
If the true signal is supported within the top $\sk$ indices of the covariance spectrum, the survival is bounded below by a constant; in contrast, contamination decays to zero with $n$ at a faster rate than the survival terms~\citep{muthukumar2021classification}.
Surprisingly, even when the signal is supported outside the top $\sk$ indices, its decay rate may still be slower than the non-support components.

With this in mind, we propose Algorithm~\ref{alg:recover_support} assuming we either know the sparsity level $t$ or the covariance spectrum $\{\lambda_j\}_{j=1}^d$ (equivalently $\cov$ by assumption).
Below, we state the general-purpose support recovery guarantee for Algorithm \ref{alg:recover_support}.

\begin{theorem}\label{thm:support_identification}
    Under Assumptions~\ref{asm:rank} and~\ref{asm:k_sparse}, suppose $\gS\subseteq [\sk]$, \ie $\svtheta$ is only supported in the top $\sk$ indices of the covariance spectrum.
    Denote by $\{\tlambda_j\}_{j=1}^{d - \sk}$ the diagonal entries of the matrix $\cov_{-[\sk]}$, \ie $\cov$ with the first $\sk$ rows and columns left out.
    Algorithm \ref{alg:recover_support} recovers the support of the true regressor with probability at least $1-ctde^{-n^{2\epsilon}}$ if either (1) $t$ is known and the additional conditions $\lambda_j \ll \frac{\lambda_q \cdot n^{1-2\epsilon}}{t^2}$ and $\lambda_j \ll \frac{\pts{\tlambda_1 r_0\pts{\cov_{-[\sk]}}}^2}{t^2\cdot\pts{n^{1+2\epsilon}}\cdot\lambda_\ell}$ hold for all $j\in\gS$, $q \in \gS^\mathsf{c} \cap [\sk]$, $\ell >\sk$, or (2) $\cov$ is known.
\end{theorem}

\begin{figure*}[t]
    \begin{subfigure}[b]{\textwidth}
        \centering
        \begin{subfigure}[b]{0.32\textwidth}
            \includegraphics[scale=0.33]{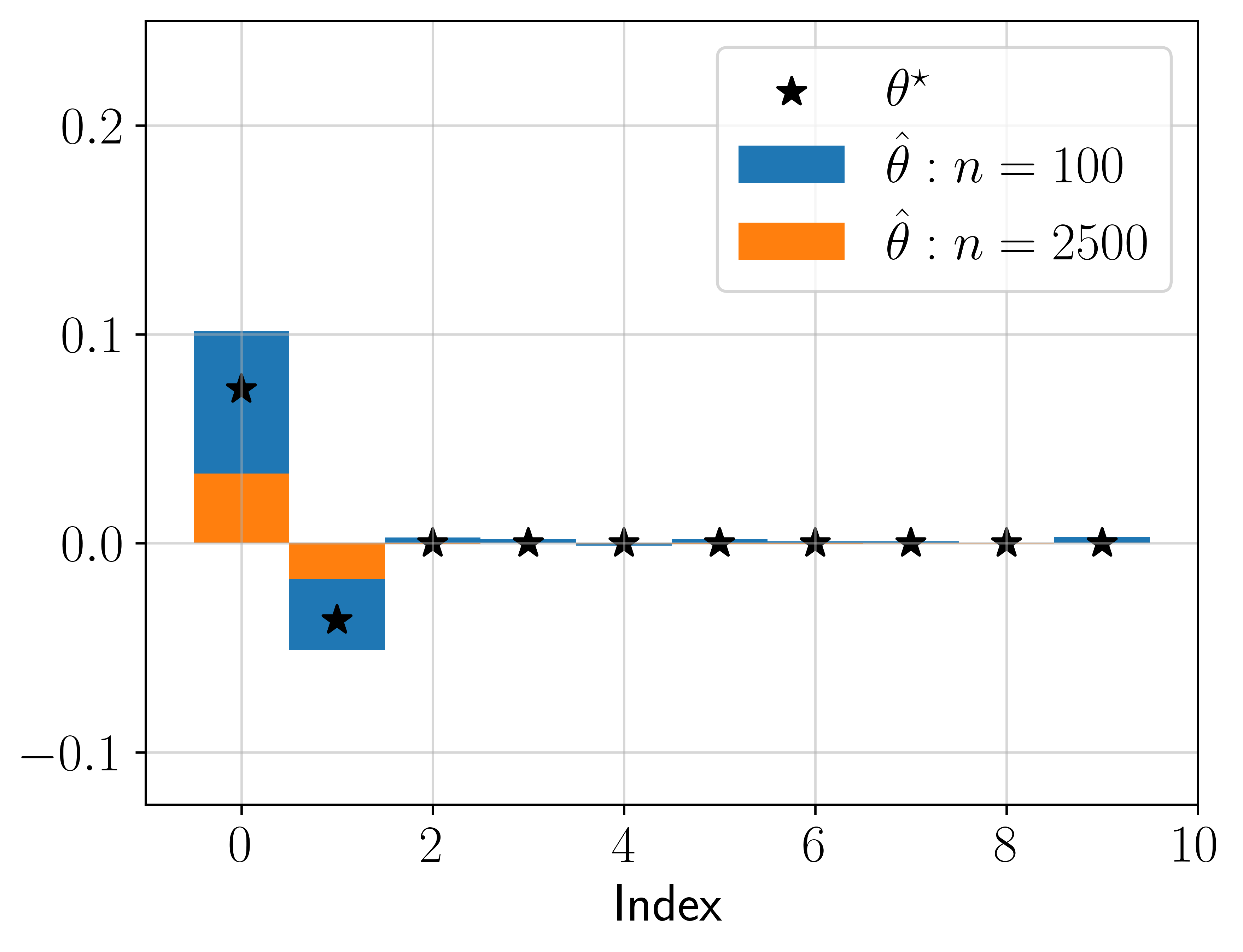}
            \subcaption{Support recovery}
        \end{subfigure}
        \hfill
        \begin{subfigure}[b]{0.32\textwidth}
            \includegraphics[scale=0.33]{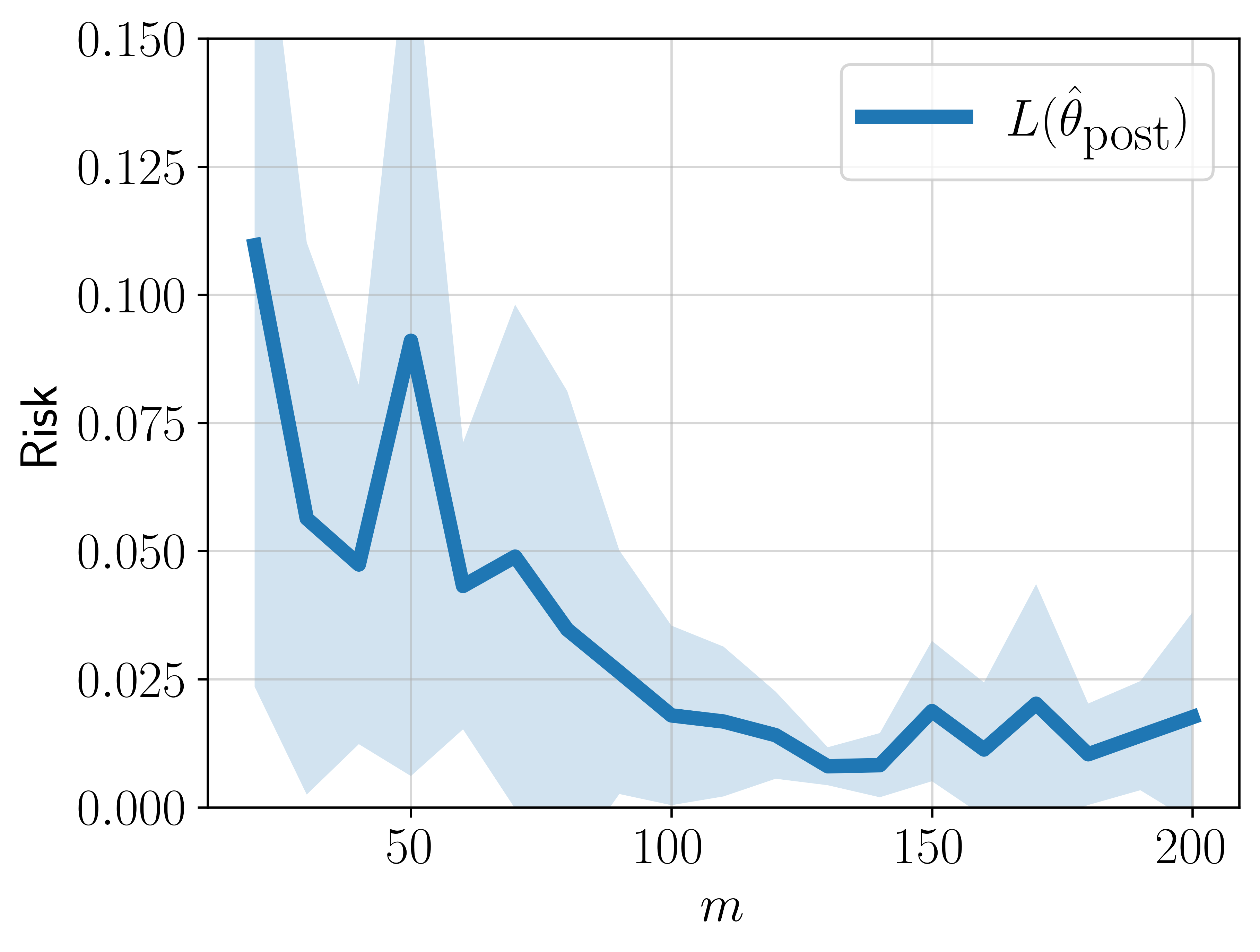}
            \subcaption{Few-shot least squares}
        \end{subfigure}
        \hfill
        \begin{subfigure}[b]{0.32\textwidth}
            \includegraphics[scale=0.33]{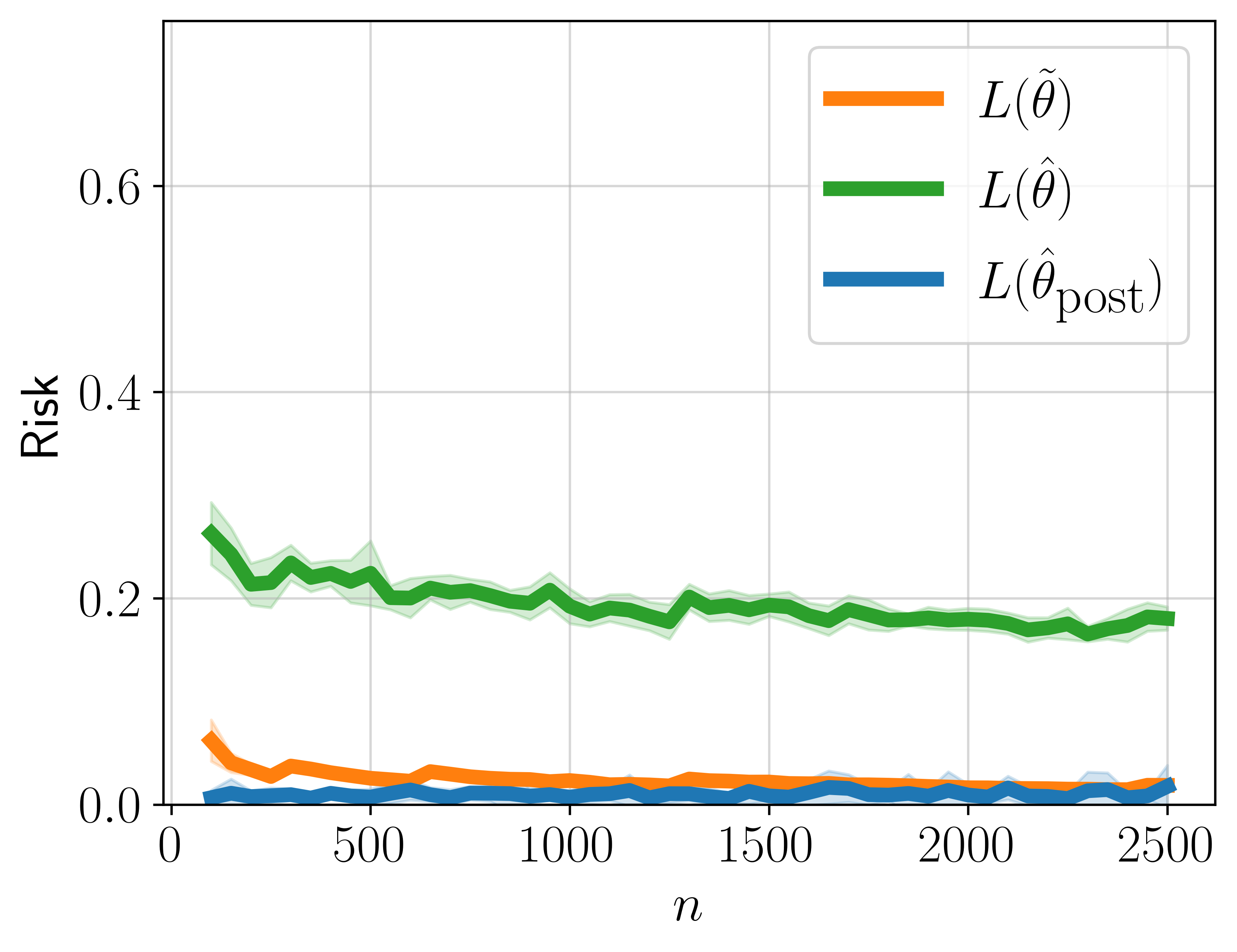}
            \subcaption{Regression risk}
        \end{subfigure}
        \hfill
        \caption*{\textbf{(i) Task shift for spiked covariance when regression generalizes.} We set $p=1.5$, $q=0.5$, and $r=0.25$ so that $q< 1-r$. In this regime, both our task shift estimator $\hvthetap$ and the regression MNI $\tvtheta$ generalize.}
        \label{fig:spiked_r25}
    \end{subfigure}
    \addtocounter{subfigure}{-3}
    \begin{subfigure}[b]{\textwidth}
        \centering
        \begin{subfigure}[b]{0.32\textwidth}
            \includegraphics[scale=0.33]{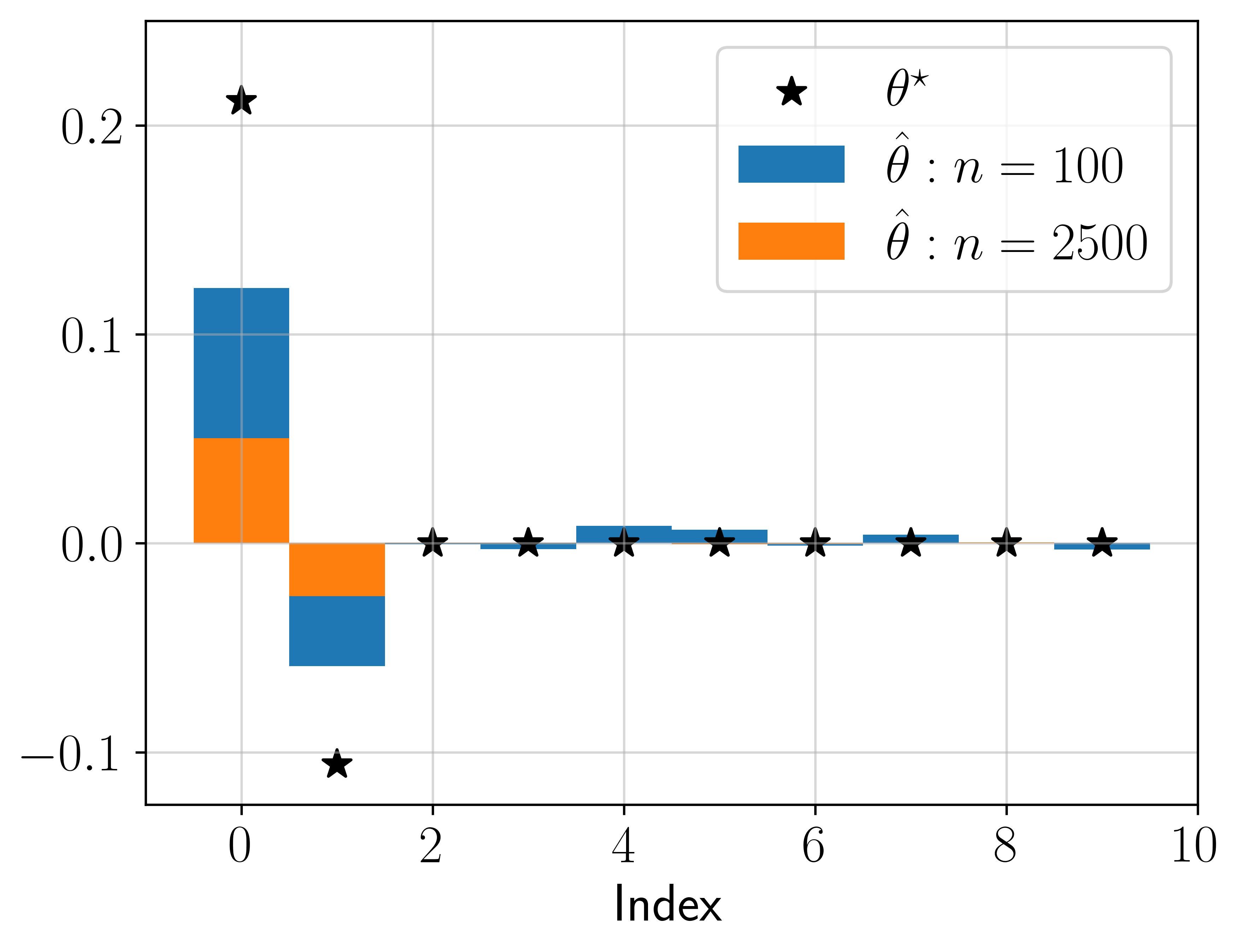}
            \subcaption{Support recovery}
        \end{subfigure}
        \hfill
        \begin{subfigure}[b]{0.32\textwidth}
            \includegraphics[scale=0.33]{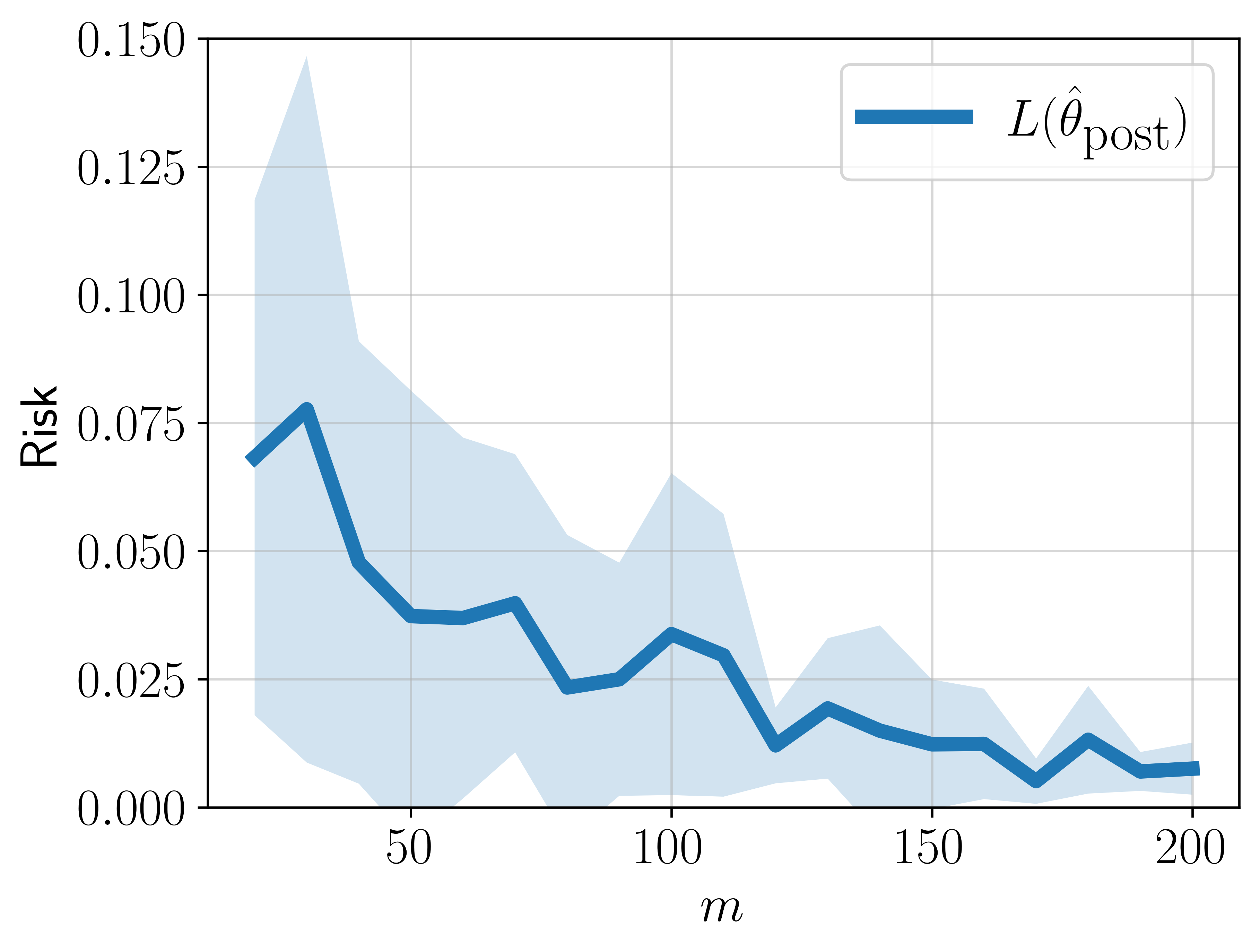}
            \subcaption{Few-shot least squares}
        \end{subfigure}
        \hfill
        \begin{subfigure}[b]{0.32\textwidth}
            \includegraphics[scale=0.33]{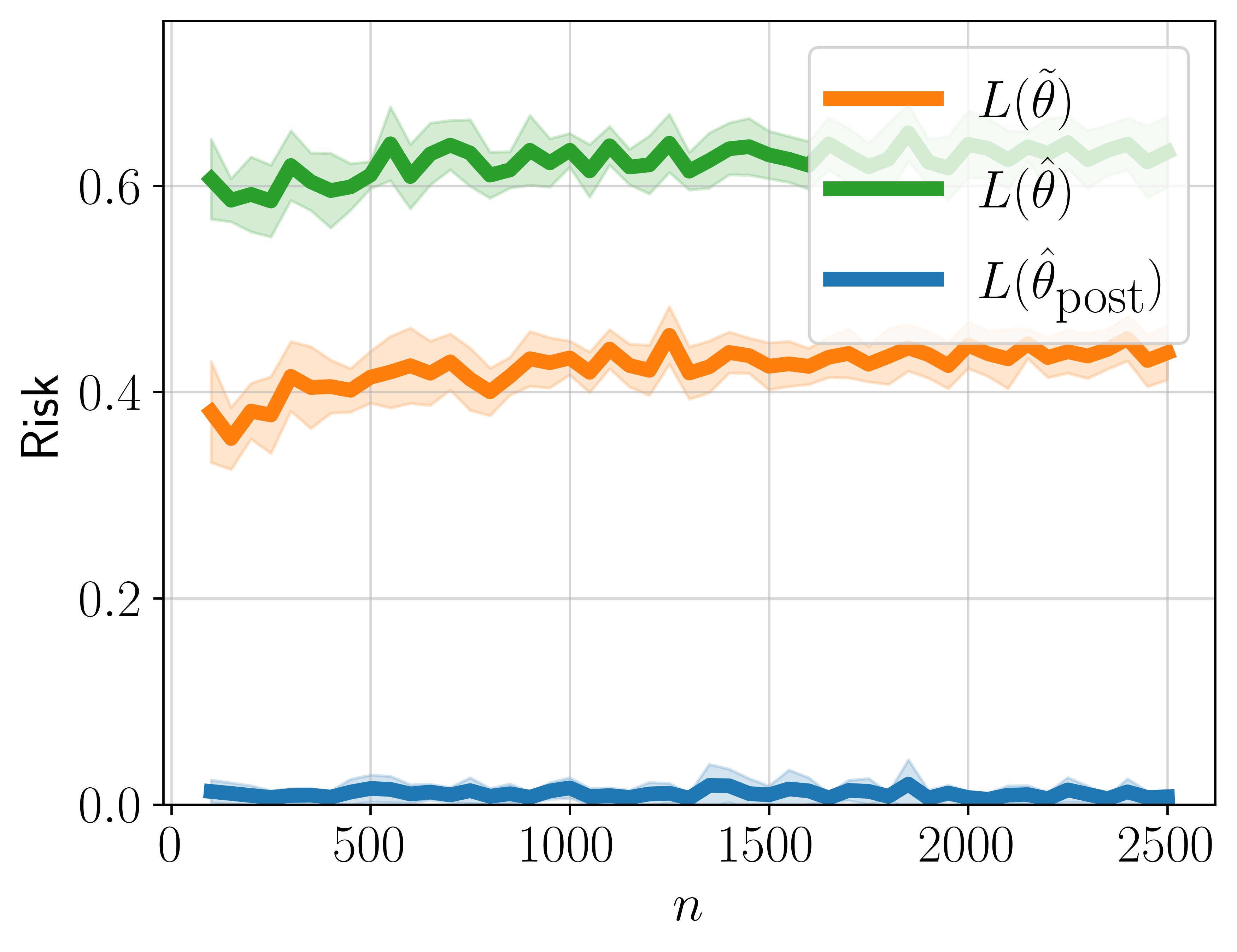}
            \subcaption{Regression risk}
        \end{subfigure}
        \hfill
        \caption*{\textbf{(ii) Task shift for spiked covariance when regression does not generalize.} We set $p=1.5$, $q=0.5$, and $r=0.55$ so that $q> 1-r$. In this regime, our task shift estimator $\hvthetap$ generalizes, but the regression MNI $\tvtheta$ does not. Note that $\sk=0$ here, so $\svtheta$ is necessarily supported outside the ``spike''.}
        \label{fig:spiked_r55}
    \end{subfigure}
    \caption{\textbf{Postprocessing achieves task shift even when minimum $\ell_2$-norm interpolation fails for both classification and regression.} The left column demonstrates the survival of $t$-sparse signal support components in the classification MNI $\hvtheta$ while non-support components decay quickly. The middle column shows the $\gO\pts{\frac{t}{m}}$ regression error of least-squares with reduction to $t$ dimensions using $m$ regression samples under standard Gaussian noise. Finally, the right column displays the regression risk of the classification MNI, regression MNI, and our postprocessed predictor. The signal $\svtheta$ is $2$-sparse with $a_1=1$ and $a_2=-0.5$ (see Assumption~\ref{asm:k_sparse}). The middle column fixes $n=2500$. We plot the mean and standard deviation over $10$ draws of the training dataset $\rmX$. See Appendix \ref{app:simulation} for additional simulations.
    }
    \label{fig:spike}
\end{figure*}

The first case of Theorem~\ref{thm:support_identification} utilizes Lemma~\ref{lem:k_sparse_SU_CN_bounds} to ensure the lower bound of support components of $\hvtheta$ is larger than the upper bound of non-support components of $\hvtheta$.
The two additional conditions in the first case of Theorem~\ref{thm:support_identification} are necessary to avoid scenarios wherein the support components of $\svtheta$ decay very quickly with $n$ (recall that $\sevtheta_j\coloneqq a_j\lambda_j^{-\frac{1}{2}}$, so a larger $\lambda_j$ implies faster decay).
Roughly, the first condition states that support eigenvalues should not be much larger than non-support eigenvalues in the top $\sk$ indices, and the second condition states that the top $\sk$ eigenvalues should not be much larger than the squared sum of the ``tail'' eigenvalues.
We analyze these conditions for specific covariance ensembles in Appendix~\ref{app:recover_support_proof}.
On the other hand, the second case of Theorem~\ref{thm:support_identification} does not require any additional assumptions, as we show $|\hevtheta_j|$ is lower bounded by $\lambda_j^{-\frac{1}{2}}$ if $j\in \mathcal{S}$ but decays at a faster rate if $j\notin \mathcal{S}$.

We provide the proof of Theorem~\ref{thm:support_identification} in Appendix~\ref{app:recover_support_proof}. We include extensions to cases wherein $\svtheta$ is supported outside the top $\sk$ indices of the covariance spectrum for spiked and polynomial decay models --- surprisingly including isotropic covariance, despite isotropy not being conducive to generalization even in classification tasks~\citep{muthukumar2021classification}.

\subsection{Least-squares on recovered support}
In this section, we leverage few-shot regression data $\{(\rvx_i^\prime,y^\prime_i)\}_{i=1}^m$ to recover the magnitude of the $t$-sparse signal $\svtheta$.
Since the support of $\svtheta$ has already been recovered, we employ a straightforward least-squares estimation technique considering only the $t$ components of each regression datum which lie in the support. 
Algorithm \ref{alg:post_processing} describes this method in detail.

We recall that the few-shot regression dataset is allowed to be noisy, \ie for some $\sigma^2\geq 0$, we may have $y^\prime=\rvx_i^{\prime \top} \svtheta + \rveps$ where $\rveps\sim\gN(\vzero,\sigma^2\mI)$.
Therefore, provided that $m \gg t$, Algorithm \ref{alg:post_processing} enjoys the standard least-squares guarantee of $\gO(\frac{t\sigma^2}{m})$ regression risk, or zero regression error with only $t$ noiseless samples.
Either of these imply the desired regression consistency, \ie task shift is achieved.

In Figure \ref{fig:spike}, we demonstrate the performance of our postprocessing procedure, combining Algorithm \ref{alg:recover_support} and Algorithm \ref{alg:post_processing}.
Notably, our task shift estimator $\hvtheta_{\tn{post}}$ generalizes even for covariance ensembles wherein minimum $\ell_2$-norm interpolation fails, \ie the regression MNI is statistically inconsistent with respect to regression labels.
In fact, Appendix~\ref{app:simulation} demonstrates that this success persists even when the classification MNI is inconsistent for classification tasks --- including the case of isotropic covariance.

A remaining question is whether we can recover the magnitudes of the support of $\hvtheta$ without \emph{any} few-shot regression data.
Lemma \ref{lem:k_sparse_SU_CN_bounds} implies that if the effective signal strength $\var$ is known, then a simple scaling procedure is sufficient.
Specifically, for all $j \in \gS$ we may set $\hevtheta_{\tn{post},j} = \hevtheta_j \sqrt{\frac{\pi}{2}\var}$, and $\hevtheta_{\tn{post},j} = 0$ otherwise.
Lemma \ref{lem:k_sparse_SU_CN_bounds} then directly implies that $\hvtheta_{\tn{post}}\to\svtheta$ as $n,d\to\infty$ as long as $b_j \to 1$ for all $j \in \gS$.
But while this approach is specialized to Gaussian covariates and noiseless classification data, Algorithm~\ref{alg:post_processing} is more robust to modeling assumptions, and we believe it can handle even unknown label noise and sub-Gaussian covariates by building on~\cite{cao2021risk,wang2023benign}.

\section{Discussion}
Our results paint a pessimistic picture for zero-shot task shift (perhaps as expected), but an optimistic one for the few-shot case.
Our key insight is that the attenuation of the classification MNI is surprisingly structured, which suggests one can get more ``mileage'' out of the MNI than previously known, including for few-shot task shift to regression.
A principal open question is whether there exist alternative formulations (\eg shifting from multiclass classification to regression) more conducive to zero-shot task shift.
More close-ended questions include providing a successful few-shot procedure in the absence of any sparsity or data covariance assumption and studying task shift for minimum $\ell_p$-norm estimators where $p \neq 2$.

While we use task shift in large language models (LLMs) only as a motivating example for our theoretical investigation, our work could inform future analyses of of LLMs and in-context learning. In particular, our survival and contamination analysis may be extended to the neural tangent kernel regime via recent frameworks for kernel interpolation, \eg\cite{mallinar2022benign,tsigler2023benign, kaushik2024new}.
Furthermore, recent work has characterized linear attention as high-dimensional linear regression under a specific data embedding~\citep{lu2024asymptotic}, which could be analyzed in our framework to explain few-shot task shift in linear Transformers.

\paragraph{Acknowledgments.}
We thank Jacob Abernethy for the compute assistance and anonymous reviewers for helpful feedback.
T.L.~acknowledges support from the DoD NDSEG Fellowship. K.L.~acknowledges support from an ARC-ACO Fellowship provided by Georgia Tech.
V.M.~acknowledges support from the NSF (awards
CCF-223915 and IIS-2212182), Adobe Research, and Amazon Research.

\clearpage

\printbibliography

\clearpage

\onecolumn
\appendix
\addcontentsline{toc}{section}{Appendix} %
\part{Appendix} %
\parttoc %

\newpage

\section{Expanded related work}\label{app:expanded_related_work}

We organize related work under four verticals.
\paragraph{Task shift \emph{vis-a-vis} transfer learning.}
Our problem formulation differs from the popular \emph{transfer learning} paradigm, which utilizes a pretrained representation on downstream tasks, \eg via finetuning or knowledge transfer.
Numerous works have analyzed the generalization of transfer learning, including for high-dimensional linear regression~\citep{dar2022deep} and general function classes~\citep{tripuraneni2020transfer}.
In the transfer learning literature, improved sample complexity guarantees are often provided as compared to learning each of the tasks from scratch; our few-shot results in Section \ref{sec:fewshot} also have this flavor at a high level.
In our \emph{task shift} setting, the ground-truth signal does not change between the tasks --- only the nature of the task in terms of its metric, \ie test loss function, changes.
Task shift is particularly compelling when the aim is to shift to a fundamentally harder task, as in this work where we transfer a classification estimator to a regression task.
Conversely, recent work formulating regression as a multi-class classification problem showed that despite discretization, the classification loss can aid in feature learning~\citep{stewart2023regression}.
Earlier, multi-label classification problems (\ie prediction of a vector-valued discrete output) were modeled as binary classification utilizing a sparsity assumption on the labels~\citep{hsu2009multi}.

\paragraph{Task shift \emph{vis-a-vis} distribution shift.}
Though ultimately very different, task shift shares some similarities with \emph{distribution shift}, wherein the training distribution $p$ and test distribution $q$ over the feature space $\gX$ and label space $\gY$ differ.
Task shift is not directly related to the common settings of \emph{covariate shift}, wherein $p(x)\neq q(x)$ but $p(y\mid x)=q(y\mid x)$, or \emph{label shift}, wherein $p(y)\neq q(y)$ but $p(x\mid y)=q(x\mid y)$.
In particular, recent results leveraging benign overfitting and random matrix theory to analyze covariate shift~\citep{tripuraneni2021overparameterization, mallinar2024minimum, lejeune2024monotonic} are generally inapplicable in our setting.
Task shift is more closely related to \emph{concept drift}, where $p(x) = q(x)$ but $p(y\mid x) \neq q(y\mid x)$, though concept drift is typically studied in the context of temporal changes~\citep{moreno2012unifying}.
Task shift also has similarities to generalized settings including the ``generalized target shift'' of~\cite{zhang2013domain}, in which $p(y)\neq q(y)$ and $p(x\mid y)$ changes with constraints, and the ``generalized label shift'' of~\cite{tatchet2020domain}, in which $p(y)\neq q(y)$ and $p(g(x)\mid y) = q(g(x)\mid y)$ for some representation $g$.
While these works assume that $\gX$ and $\gY$ remain constant during the distribution shift, in task shift we study different or even disjoint label spaces $\gY_{\tn{train}}$ and $\gY_{\tn{test}}$.
Similar models have been proposed, such as Disjoint Label Space Transfer Learning~\citep{chang2019disjoint}, wherein $\gY_{\tn{train}}$ and $\gY_{\tn{test}}$ are completely disjoint but share a common representation, and Open Set Label Shift~\citep{garg2022open}, wherein $p(y)\neq q(y)$ and new classes may arrive during test-time as long as $p(x\mid y)$ is constant for existing classes,
Yet, our task shift formulation goes one step further in that the objective function also changes.

\paragraph{One-bit compressive sensing.}
Our problem formulation for zero-shot estimation is deeply connected to the \emph{one-bit compressive sensing} problem~\citep{boufounos2008one}.
In this setting, the data matrix $\rmX$ is designed, the ground-truth signal $\svtheta \in \R^d$ is unknown but $t$-sparse, and the objective is signal recovery from $n \ll d$ quantized measurements of the form $\hat{y}_i = \sgn{\rvx_i^\top\svtheta}$. 
The most prominent difference between our frameworks is that the focus in one-bit compressive sensing is designing the \emph{optimal estimation procedure} with knowledge of the measurement process.
This procedure involves either solving a convex program to maximize the average margin on training data subject to a $\ell_1$-norm constraint~\citep{plan2012robust,awasthi2016learning,chinot2022adaboost} or combinatorial optimization routines~\citep{gopi2013one}.
In contrast, we have no control over the design of the estimator and assume \emph{de facto} the $\ell_2$-inductive bias, which due to the implicit bias of gradient descent is one of the most commonly observed in machine learning~\citep{soudry2018implicit,ji2019implicit}.
Accordingly, while one-bit compressive sensing tends to consider isotropic or near-isotropic ensembles, we consider a gamut of anisotropic ensembles which can be more favorable to the $\ell_2$-inductive bias\footnote{While, surprisingly, our few-shot guarantees hold even for the isotropic ensemble, the sample complexity is suboptimal compared to one-bit compressive sensing, requiring $n = d^{\frac{1}{p}}$ for some $p > 1$.}.
Moreover, one-bit compressive sensing allows the design of biases which leak information about not only the signal direction but also its magnitude.
For example, one can design \emph{known} biases $\bm{\tau}\in \R^n$ and modify the measurements to $\hat{y}_i = \sgn{\rvx_i^\top\svtheta + \tau_i}$~\citep{knudson2016one,dirksen2021non}.
We have no such flexibility in our framework --- any error terms that may arise would be of the form of unknown regression or classification label noise, and would therefore worsen our zero-shot lower bounds.
Nevertheless, our results in Theorem~\ref{thm:zero_shot_general_vacuous} show that even under the weaker $\ell_2$-inductive bias, positive results in the flavor of one-bit compressive sensing are possible: specifically, we can estimate the signal correctly \emph{only if} its total magnitude is known and given to be $\var = \frac{2}{\pi}$.

\paragraph{The $\ell_2$-inductive bias and benign overfitting.}
The minimum $\ell_2$-norm interpolator has been shown, for certain ``effectively high dimensional'' data covariances, to overfit noise in a \emph{benign} manner, meaning that the extra error arising from such interpolation can decay to zero as $n, d \to \infty$~\citep{bartlett2020benign, tsigler2023benign}.
Despite the possible absence of noise in our setting, we find effective dimension useful --- indeed, we propose an interpretation of the difference between regression and classification labels as effective noise.
The classical shortcoming with minimum $\ell_2$-norm solutions is not overfitting noise, but instead their propensity to \emph{attenuate} signal~\citep{chen2001atomic}; even for $1$-sparse $\svtheta = \ve_1$ and $d \gg n$ we would have $\hat{\theta}_1 \to 0$ for isotropic data covariance~\citep{hastie2022surprises}.
Statistically speaking, minimum $\ell_2$-norm solutions suffer from a constant bias and therefore inconsistent regression error even when trained on regression labels.
Despite this,~\cite{muthukumar2021classification, cao2021risk, wang2022binary} showed that for sufficiently anisotropic ensembles, one can achieve \emph{classification consistency on classification labels} when \emph{regression would not be consistent on regression labels}.
The key insight, developed primarily for $1$-sparse signal and spiked covariance models, was that the relative magnitude of the true feature $\hat{\theta}_1$ (quantified through a metric called \emph{survival}) is preserved with respect to the total magnitude of the ``false'' features $\{\hat{\theta}_j\}_{j = 2}^d$ (quantified through a metric called \emph{contamination}).

We substantially develop this insight to show that the relative feature magnitudes can be used to estimate the support of a sparse signal even when neither classification nor regression generalizes --- including isotropic covariance, the worst-case data model for minimum $\ell_2$-norm solutions --- which in turn enables few-shot regression consistency.
Our support recovery procedure is also generally applicable; it only requires a diagonal covariance and either the sparsity level $t$ or or the covariance spectrum $\{\lambda_j\}_{j=1}^d$ to be known.

\section{Key lemma: general survival and contamination bounds}\label{app:su_cn_bound_proof}
In this section, we provide the proof of Lemma~\ref{lem:k_sparse_SU_CN_bounds}: our extension of survival and contamination bounds to $t$-sparse signal $\svtheta$ and general covariance matrices $\cov$. In contrast to the $1$-sparse $\svtheta$ result of~\cite{muthukumar2021classification}, we consider a more general $t$-sparse $\svtheta$ setting. To overcome this technical obstacle, we utilize the leave-$k$-out technique of~\cite{wang2023benign} to complete the proof. For bounds on the survival term $\SU_{s_\ell}$, where $s_\ell$ is the $\ell^{th}$ element in $\gS$ for $\ell\in[t]$, we show the result when $\ell=1$ without loss of generality. Before we begin the proof of Lemma~\ref{lem:k_sparse_SU_CN_bounds}, we define the following notation for $Q_{s_\ell}$ and $\tilde{Q}_{s_\ell}$:
\begin{align}
    Q_{s_\ell} &\coloneqq \rvz_{s_1}^\top \rmA_{-s_1:s_\ell}^{-1}\hrvy\label{eq:Q_def}  \\
    \tilde{Q}_{s_\ell} &\coloneqq \rvz_{s_1}^\top \rmA_{-s_1:s_\ell}^{-1}\rvz_{s_1}\label{eq:tQ_def},
\end{align}
 where we denote shorthand notation $\rmA\coloneqq \XX$ and the \emph{leave-$\ell$-out} matrices corresponding to $\rmA_{-s_1:s_\ell} \coloneqq \sum_{j=1,j\neq s_1,\cdots,s_\ell}^d\lambda_j\rvz_j\rvz_j^\top$.  We also need the following auxiliary lemmas which are generalized from lemmas in~\cite{wang2023benign}. Note that lemmas in~\cite{wang2023benign} only apply to leave-\emph{top $k$}-out matrices, while we generalize them into leave-\emph{discrete $t$}-out matrices. Proofs of these lemmas can be found in Appendix~\ref{app:su_cn_auxiliary}.
\begin{lemma}[Generalization of Lemma 15 in~\cite{wang2023benign}]\label{lem:Q_k_bounds} Under Assumptions~\ref{asm:rank} and~\ref{asm:k_sparse}, for large enough $n$, we have
    \begin{align*}
        Q_{s_t} &\leq \sgn{a_{s_1}}\sqrt{\frac{2\beta_{s_1}}{\pi}}\Tr\pts{\rmA_{-s_1:s_t}^{-1}} + 2c_1\norm{\rmA_{-s_1:s_t}^{-1}}_2 \cdot n^{\frac{1}{2}+\epsilon}, \\
        Q_{s_t} &\geq \sgn{a_{s_1}}\sqrt{\frac{2\beta_{s_1}}{\pi}}\Tr\pts{\rmA_{-s_1:s_t}^{-1}} - 2c_1\norm{\rmA_{-s_1:s_t}^{-1}}_2 \cdot n^{\frac{1}{2}+\epsilon},
    \end{align*}
    where $\beta_{s_\ell} \coloneqq \frac{\lambda_{s_\ell}\sevthetas_{s_\ell}}{\sum_{j\in\gS} \lambda_j\sevthetas_j} = \frac{a_{s_\ell}^2}{\sum_{j\in\gS} a_j^2}$ for $\ell\in[t]$, with probability at least $1-2e^{-n^{2\epsilon}}$.
\end{lemma}
\begin{lemma}[Generalization of Lemma 16 in~\cite{wang2023benign}]\label{lem:Q_1_bounds} We have that $Q_{s_1}$ is tight up to an additive factor in $Q_{s_t}$ as
    \begin{align*}
        \pts{1 - \sgn{a_{s_1}}\frac{c_4 t}{c_3n^{\frac{1}{2}-\epsilon}-1}}Q_{s_t} \leq Q_{s_1} \leq  \pts{1 + \sgn{a_{s_1}}\frac{c_4 t}{c_3n^{\frac{1}{2}-\epsilon}-1}}Q_{s_t}
    \end{align*}
    with probability at least $1-cte^{-n^{2\epsilon}}$.
\end{lemma}
\begin{lemma}[Generalization of Lemma 21 in~\cite{wang2023benign}]\label{lem:cn_const}
    Define $\crvy_{s_\ell} \coloneqq \crvy_{s_{\ell-1}} - \frac{\lambda_{s_\ell}\rvz_{s_\ell}^\top\rmA_{-s_1:s_\ell}^{-1}\crvy_{s_{\ell-1}}}{1+\lambda_{s_\ell}\rvz_{s_\ell}^\top\rmA_{-s_1:s_\ell}^{-1}\rvz_{s_\ell}}\rvz_{s_\ell}$ and $\crvy_{s_0} = \hrvy$ for $\ell\in[t]$. Then,
    \begin{align*}
        \left|\frac{\lambda_{s_\ell}\rvz_{s_\ell}^\top\rmA_{-s_1:s_\ell}^{-1}\crvy_{s_{\ell-1}}}{1+\lambda_{s_\ell}\rvz_{s_\ell}^\top\rmA_{-s_1:s_\ell}^{-1}\rvz_{s_\ell}}\right| \leq c
    \end{align*}
    with probability at least $1-ct^2e^{-n^{2\epsilon}}$.
\end{lemma}
Now we can prove Lemma~\ref{lem:k_sparse_SU_CN_bounds}.
\begin{proof}(Lemma~\ref{lem:k_sparse_SU_CN_bounds})
    We start with the survival upper and lower bound. By the survival definition in Definition~\ref{def:su_cn}, we can write $\SU_{s_1}$ as
    \begin{align}
        \SU_{s_1} &= \frac{\lambda_{s_1}}{a_{s_1}}\rvz_{s_1}^\top\rmA^{-1}\hrvy\nonumber\\
        &= \frac{\lambda_{s_1}}{a_{s_1}}\rvz_{s_1}^\top\pts{\rmA_{-s_1:s_1}^{-1} - \frac{\lambda_{s_1}\rmA_{-s_1:s_1}^{-1}\rvz_{s_1}\rvz_{s_1}^\top\rmA_{-s_1:s_1}^{-1}}{1+ \lambda_{s_1}\rvz_{s_1}^\top\rmA_{-s_1:s_1}^{-1}\rvz_{s_1}}}\hrvy\nonumber\\
        &= \frac{\lambda_{s_1}}{a_{s_1}}\frac{\rvz_{s_1}^\top\rmA_{-s_1:s_1}^{-1}\hrvy}{1+\lambda_{s_1}\rvz_{s_1}^\top\rmA_{-s_1:s_1}^{-1}\rvz_{s_1}} \nonumber \\
        &= \frac{\lambda_{s_1}}{a_{s_1}}\frac{Q_{s_1}}{1+\lambda_{s_1}\tilde{Q}_{s_1}} \nonumber \\
        &= \frac{\lambda_{s_1}}{|a_{s_1}|}\frac{\sgn{a_{s_1}}Q_{s_1}}{1+\lambda_{s_1}\tilde{Q}_{s_1}}\label{eq:su_term},
    \end{align}
    where we apply the Sherman-Morrison-Woodbury identity in the second equality. Next, recall that $\tilde{Q}_{s_\ell} \coloneqq \rvz_{s_1}^\top \rmA_{-s_1:s_\ell}^{-1}\rvz_{s_1}$ in~\Eqref{eq:tQ_def}, and by the Hanson-Wright inequality (Lemma~\ref{lem:hanson_wright_ineq}), we have 
    \begin{align*}
        \tilde{Q}_{s_1} &\leq \Tr\pts{\rmA_{-s_1:s_1}^{-1}} + c_1\norm{\rmA_{-s_1:s_1}^{-1}}_2 \cdot n^{\frac{1}{2}+\epsilon}\\
        &\leq \Tr\pts{\rmA_{-\pts{s_1:s_t}\cup[\sk]}^{-1}} + c_1\norm{\rmA_{-\pts{s_1:s_t}\cup[\sk]}^{-1}}_2 \cdot n^{\frac{1}{2}+\epsilon},
    \end{align*}
    where the second inequality follows because $\rmA_{-s_1:s_1}\succeq\rmA_{-\pts{s_1:s_t}\cup[\sk]}$ and therefore $\rmA_{-s_1:s_1}^{-1} \preceq \rmA_{-\pts{s_1:s_t}\cup[\sk]}^{-1}$. On the other hand,
    \begin{align*}
        \tilde{Q}_{s_1} &\geq \Tr\pts{\rmA_{-s_1:s_1}^{-1}} - c_1\norm{\rmA_{-s_1:s_1}^{-1}}_2 \cdot n^{\frac{1}{2}+\epsilon}\\
        &\geq \pts{1-\frac{c}{n}}^{\sk+t-1}\Tr\pts{\rmA_{-\pts{s_1:s_t}\cup[\sk]}^{-1}} - c_1\norm{\rmA_{-\pts{s_1:s_t}\cup[\sk]}^{-1}}_2 \cdot n^{\frac{1}{2}+\epsilon},
    \end{align*}
    where the second inequality follows from the trace lower bound in Lemma~\ref{lem:trace_lower_bound}. Next, we need to derive the bounds for $\sgn{a_{s_1}}Q_{s_1}$ term in~\Eqref{eq:su_term}. Note that we need to adjust the bounds of $Q_{s_1}$ in Lemma~\ref{lem:Q_1_bounds} according to the sign of $a_{s_1}$. Considering the sign of $a_{s_1}$, from Lemma~\ref{lem:Q_1_bounds}, we can have
    \begin{align*}
        \pts{1 - \frac{c_4 t}{c_3n^{\frac{1}{2}-\epsilon}-1}}\sgn{a_{s_1}}Q_{s_t} \leq \sgn{a_{s_1}}Q_{s_1} \leq  \pts{1 + \frac{c_4 t}{c_3n^{\frac{1}{2}-\epsilon}-1}}\sgn{a_{s_1}}Q_{s_t}.
    \end{align*}
    From Lemma~\ref{lem:Q_k_bounds}, we can further bound the $\sgn{a_{s_1}}Q_{s_t}$ term as
    \begin{align*}
        \sgn{a_{s_1}}Q_{s_t} &\leq \sqrt{\frac{2\beta_{s_1}}{\pi}}\Tr\pts{\rmA_{-s_1:s_t}^{-1}} + 2c_1\norm{\rmA_{-s_1:s_t}^{-1}}_2 \cdot n^{\frac{1}{2}+\epsilon}, \\
        \sgn{a_{s_1}}Q_{s_t} &\geq \sqrt{\frac{2\beta_{s_1}}{\pi}}\Tr\pts{\rmA_{-s_1:s_t}^{-1}} - 2c_1\norm{\rmA_{-s_1:s_t}^{-1}}_2 \cdot n^{\frac{1}{2}+\epsilon}.
    \end{align*}
    As a result, we can apply the bounds of $\tilde{Q}_{s_1}$ and the bounds of $\sgn{a_{s_1}}Q_{s_1}$ above to get the upper and lower bound of $\SU_{s_1}$ from~\Eqref{eq:su_term}. We then obtain the $\SU_{s_1}$ upper bound as
    \begin{align}
        (\ref{eq:su_term}) &\leq \pts{\frac{\lambda_{s_1}}{|a_{s_1}|}} \frac{ \pts{1 + \frac{c_4t}{c_3n^{\frac{1}{2}-\epsilon}-1}}\pts{\sqrt{\frac{2\beta_{s_1}}{\pi}}\Tr\pts{\rmA_{-s_1:s_t}^{-1}} + 2c_1\norm{\rmA_{-s_1:s_t}^{-1}}_2 \cdot n^{\frac{1}{2}+\epsilon}}}{1 + \lambda_{s_1}\pts{\pts{1-\frac{c}{n}}^{\sk+t-1}\Tr\pts{\rmA_{-\pts{s_1:s_t}\cup[\sk]}^{-1}} - c_1\norm{\rmA_{-\pts{s_1:s_t}\cup[\sk]}^{-1}}_2 \cdot n^{\frac{1}{2}+\epsilon}}}\nonumber\\
        &= \sqrt{\frac{2}{\pi\var}} \pts{\frac{ \pts{1 + \frac{c_4t}{c_3n^{\frac{1}{2}-\epsilon}-1}}\lambda_{s_1}\pts{\Tr\pts{\rmA_{-s_1:s_t}^{-1}} + c_2\norm{\rmA_{-s_1:s_t}^{-1}}_2 \cdot n^{\frac{1}{2}+\epsilon}}}{1 + \lambda_{s_1}\pts{\pts{1-\frac{c}{n}}^{\sk+t-1}\Tr\pts{\rmA_{-\pts{s_1:s_t}\cup[\sk]}^{-1}} - c_1\norm{\rmA_{-\pts{s_1:s_t}\cup[\sk]}^{-1}}_2 \cdot n^{\frac{1}{2}+\epsilon}}}},\label{eq:su_upper_bound}
    \end{align}
    and the lower bound of $\SU_{s_1}$ as
    \begin{align}
        (\ref{eq:su_term}) &\geq \pts{\frac{\lambda_{s_1}}{|a_{s_1}|}}\frac{ \pts{1 - \frac{c_4t}{c_3n^{\frac{1}{2}-\epsilon}-1}}\pts{\sqrt{\frac{2\beta_{s_1}}{\pi}}\Tr\pts{\rmA_{-s_1:s_t}^{-1}} - 2c_1\norm{\rmA_{-s_1:s_t}^{-1}}_2 \cdot n^{\frac{1}{2}+\epsilon}}}{1 + \lambda_{s_1}\pts{\Tr\pts{\rmA_{-\pts{s_1:s_t}\cup[\sk]}^{-1}} + c_1\norm{\rmA_{-\pts{s_1:s_t}\cup[\sk]}^{-1}}_2 \cdot n^{\frac{1}{2}+\epsilon}}}\nonumber\\
        &=  \sqrt{\frac{2}{\pi\var}} \pts{\frac{ \pts{1 - \frac{c_4t}{c_3n^{\frac{1}{2}-\epsilon}-1}}\lambda_{s_1}\pts{\Tr\pts{\rmA_{-s_1:s_t}^{-1}} - c_2\norm{\rmA_{-s_1:s_t}^{-1}}_2 \cdot n^{\frac{1}{2}+\epsilon}}}{1 + \lambda_{s_1}\pts{\Tr\pts{\rmA_{-\pts{s_1:s_t}\cup[\sk]}^{-1}} + c_1\norm{\rmA_{-\pts{s_1:s_t}\cup[\sk]}^{-1}}_2 \cdot n^{\frac{1}{2}+\epsilon}}}}\label{eq:su_lower_bound}.
    \end{align}
    In the equalities, we use $\frac{1}{|a_{s_1}|}\sqrt{\frac{2\beta_{s_1}}{\pi}}=\frac{1}{|a_{s_1}|}\sqrt{\frac{2a_{s_1}^2}{\pi\sum_{j\in\gS}a_j^2}}=\sqrt{\frac{2}{\pi\var}}$. Next, we need to upper bound $\Tr\pts{\rmA_{-s_1:s_t}^{-1}}$ and $\norm{\rmA_{-s_1:s_t}^{-1}}_2$ and lower bound $\Tr\pts{\rmA_{-s_1:s_t}^{-1}}$ terms; we achieve this by relating $\rmA_{-s_1:s_t}$ to $\rmA_{-\pts{s_1:s_t}\cup[\sk]}$. Lemma~\ref{lem:trace_lower_bound} provides that
    \begin{align*}
        \Tr\pts{\rmA_{-\pts{s_1:s_t}\cup[\sk]}^{-1}}\geq \Tr\pts{\rmA_{-s_1:s_t}^{-1}} \geq \pts{1-\frac{c}{n}}^{\sk}\Tr\pts{\rmA_{-\pts{s_1:s_t}\cup[\sk]}^{-1}},
    \end{align*}
    and we also have $\rmA_{-s_1:s_t}^{-1} \preceq \rmA_{-\pts{s_1:s_t}\cup[\sk]}^{-1}$. As a result,~\Eqref{eq:su_upper_bound} and~\Eqref{eq:su_lower_bound} now become
    \begin{align}
        \SU_{s_1} &\leq \sqrt{\frac{2}{\pi\var}} \pts{\frac{ \pts{1 + \frac{c_4t}{c_3n^{\frac{1}{2}-\epsilon}-1}}\lambda_{s_1}\pts{\Tr\pts{\rmA_{-\pts{s_1:s_t}\cup[\sk]}^{-1}} + c_2\norm{\rmA_{-\pts{s_1:s_t}\cup[\sk]}^{-1}}_2 \cdot n^{\frac{1}{2}+\epsilon}}}{1 + \lambda_{s_1}\pts{\pts{1-\frac{c}{n}}^{\sk+t-1}\Tr\pts{\rmA_{-\pts{s_1:s_t}\cup[\sk]}^{-1}} - c_1\norm{\rmA_{-\pts{s_1:s_t}\cup[\sk]}^{-1}}_2 \cdot n^{\frac{1}{2}+\epsilon}}}},\label{eq:su_upper_bound2}\\
        \SU_{s_1} &\geq \sqrt{\frac{2}{\pi\var}} \pts{\frac{ \pts{1 - \frac{c_4t}{c_3n^{\frac{1}{2}-\epsilon}-1}}\lambda_{s_1}\pts{\pts{1-\frac{c}{n}}^{\sk}\Tr\pts{\rmA_{-\pts{s_1:s_t}\cup[\sk]}^{-1}} - c_2\norm{\rmA_{-\pts{s_1:s_t}\cup[\sk]}^{-1}}_2 \cdot n^{\frac{1}{2}+\epsilon}}}{1 + \lambda_{s_1}\pts{\Tr\pts{\rmA_{-\pts{s_1:s_t}\cup[\sk]}^{-1}} + c_1\norm{\rmA_{-\pts{s_1:s_t}\cup[\sk]}^{-1}}_2 \cdot n^{\frac{1}{2}+\epsilon}}}}\label{eq:su_lower_bound2}.
    \end{align}
    Next, Lemma~\ref{lem:eig_constant} shows that all eigenvalues of $\rmA_{-\pts{s_1:s_t}\cup[\sk]}^{-1}$ are identical up to a constant such that
    \begin{align*}
        \frac{1}{cn} \leq \frac{\norm{\rmA_{-[\sk]\cup\pts{s_1:s_t}}^{-1}}_2}{\Tr\pts{\rmA_{-[\sk]\cup\pts{s_1:s_t}}^{-1}}} \leq \frac{c}{n}.
    \end{align*}
    By dividing $\Tr\pts{\rmA_{-[\sk]\cup\pts{s_1:s_t}}^{-1}}$ in the numerator and denominator of~\Eqref{eq:su_upper_bound2}, we can have the upper bound of $\SU_{s_1}$ as
    \begin{align}
        \SU_{s_1} &\leq \sqrt{\frac{2}{\pi\var}} \pts{\frac{ \pts{1 + \frac{c_4t}{c_3n^{\frac{1}{2}-\epsilon}-1}}\lambda_{s_1}\pts{\Tr\pts{\rmA_{-\pts{s_1:s_t}\cup[\sk]}^{-1}} + c_2\norm{\rmA_{-\pts{s_1:s_t}\cup[\sk]}^{-1}}_2 \cdot n^{\frac{1}{2}+\epsilon}}}{1 + \lambda_{s_1}\pts{\pts{1-\frac{c}{n}}^{\sk+t-1}\Tr\pts{\rmA_{-\pts{s_1:s_t}\cup[\sk]}^{-1}} - c_1\norm{\rmA_{-\pts{s_1:s_t}\cup[\sk]}^{-1}}_2 \cdot n^{\frac{1}{2}+\epsilon}}}}\nonumber\\
        &= \sqrt{\frac{2}{\pi\var}} \pts{\frac{ \pts{1 + \frac{c_4t}{c_3n^{\frac{1}{2}-\epsilon}-1}}\lambda_{s_1}\Tr\pts{\rmA_{-\pts{s_1:s_t}\cup[\sk]}^{-1}}\pts{1 + \frac{c_2\norm{\rmA_{-\pts{s_1:s_t}\cup[\sk]}^{-1}}_2 \cdot n^{\frac{1}{2}+\epsilon}}{\Tr\pts{\rmA_{-\pts{s_1:s_t}\cup[\sk]}^{-1}}}}}{1 + \lambda_{s_1}\Tr\pts{\rmA_{-\pts{s_1:s_t}\cup[\sk]}^{-1}}\pts{\pts{1-\frac{c}{n}}^{\sk+t-1} - \frac{c_1\norm{\rmA_{-\pts{s_1:s_t}\cup[\sk]}^{-1}}_2 \cdot n^{\frac{1}{2}+\epsilon}}{\Tr\pts{\rmA_{-\pts{s_1:s_t}\cup[\sk]}^{-1}}}}}}\nonumber\\
        &\leq  \sqrt{\frac{2}{\pi\var}} \pts{\frac{ \pts{1 + \frac{c_4t}{c_3n^{\frac{1}{2}-\epsilon}-1}}\lambda_{s_1}\Tr\pts{\rmA_{-\pts{s_1:s_t}\cup[\sk]}^{-1}}\pts{1 + \frac{c_3}{n^{\frac{1}{2}-\epsilon}}}}{1 + \lambda_{s_1}\Tr\pts{\rmA_{-\pts{s_1:s_t}\cup[\sk]}^{-1}}\pts{\pts{1-\frac{c}{n}}^{\sk+t-1} - \frac{c_4}{n^{\frac{1}{2}-\epsilon}}}}}\nonumber\\
        &\leq \sqrt{\frac{2}{\pi\var}} \pts{\frac{ \pts{1 + \frac{c_5t}{n^{\frac{1}{2}-\epsilon}}}\lambda_{s_1}\Tr\pts{\rmA_{-\pts{s_1:s_t}\cup[\sk]}^{-1}}}{1 + \pts{1 - \frac{c_6}{n^{\frac{1}{2}-\epsilon}}}\lambda_{s_1}\Tr\pts{\rmA_{-\pts{s_1:s_t}\cup[\sk]}^{-1}}}}\label{eq:su_upper_bound3},
    \end{align}
    where we apply the bounds in Lemma~\ref{lem:eig_constant} in the second inequality, and in the last inequality we introduce some new constants since $t\ll n^{\frac{1}{2}-\epsilon}$ and $\epsilon\in(0, \frac{1}{4})$. We repeat the same derivation for the lower bound of $\SU_{s_1}$. From~\Eqref{eq:su_lower_bound2}, we have
    \begin{align}
        \SU_{s_1} &\geq \sqrt{\frac{2}{\pi\var}} \pts{\frac{ \pts{1 - \frac{c_4t}{c_3n^{\frac{1}{2}-\epsilon}-1}}\lambda_{s_1}\pts{\pts{1-\frac{c}{n}}^{\sk}\Tr\pts{\rmA_{-\pts{s_1:s_t}\cup[\sk]}^{-1}} - c_2\norm{\rmA_{-\pts{s_1:s_t}\cup[\sk]}^{-1}}_2 \cdot n^{\frac{1}{2}+\epsilon}}}{1 + \lambda_{s_1}\pts{\Tr\pts{\rmA_{-\pts{s_1:s_t}\cup[\sk]}^{-1}} + c_1\norm{\rmA_{-\pts{s_1:s_t}\cup[\sk]}^{-1}}_2 \cdot n^{\frac{1}{2}+\epsilon}}}}\nonumber\\
        &= \sqrt{\frac{2}{\pi\var}} \pts{\frac{ \pts{1 - \frac{c_4t}{c_3n^{\frac{1}{2}-\epsilon}-1}}\lambda_{s_1}\Tr\pts{\rmA_{-\pts{s_1:s_t}\cup[\sk]}^{-1}}\pts{\pts{1-\frac{c}{n}}^{\sk} - \frac{c_2\norm{\rmA_{-\pts{s_1:s_t}\cup[\sk]}^{-1}}_2 \cdot n^{\frac{1}{2}+\epsilon}}{\Tr\pts{\rmA_{-\pts{s_1:s_t}\cup[\sk]}^{-1}}}}}{1 + \lambda_{s_1}\Tr\pts{\rmA_{-\pts{s_1:s_t}\cup[\sk]}^{-1}}\pts{1 + \frac{c_1\norm{\rmA_{-\pts{s_1:s_t}\cup[\sk]}^{-1}}_2 \cdot n^{\frac{1}{2}+\epsilon}}{\Tr\pts{\rmA_{-\pts{s_1:s_t}\cup[\sk]}^{-1}}}}}}\nonumber\\
        &\geq \sqrt{\frac{2}{\pi\var}} \pts{\frac{ \pts{1 - \frac{c_4t}{c_3n^{\frac{1}{2}-\epsilon}-1}}\lambda_{s_1}\Tr\pts{\rmA_{-\pts{s_1:s_t}\cup[\sk]}^{-1}}\pts{\pts{1-\frac{c}{n}}^{\sk} - \frac{c_3}{n^{\frac{1}{2}-\epsilon}}}}{1 + \lambda_{s_1}\Tr\pts{\rmA_{-\pts{s_1:s_t}\cup[\sk]}^{-1}}\pts{1 + \frac{c_4}{n^{\frac{1}{2}-\epsilon}}}}}\nonumber\\
        &\geq \sqrt{\frac{2}{\pi\var}} \pts{\frac{ \pts{1 - \frac{c_5t}{n^{\frac{1}{2}-\epsilon}}}\lambda_{s_1}\Tr\pts{\rmA_{-\pts{s_1:s_t}\cup[\sk]}^{-1}}}{1 + \pts{1 + \frac{c_6}{n^{\frac{1}{2}-\epsilon}}}\lambda_{s_1}\Tr\pts{\rmA_{-\pts{s_1:s_t}\cup[\sk]}^{-1}}}}\label{eq:su_lower_bound3}.
    \end{align}
    Finally, Lemma~\ref{lem:eig_constant} also implies the bounds of $\Tr\pts{\rmA_{-\pts{s_1:s_t}\cup[\sk]}^{-1}}$ such that
    \begin{align*}
        \Tr\pts{\rmA_{-\pts{s_1:s_t}\cup[\sk]}^{-1}} &= \sum_{i=1}^n\frac{1}{\mu_i\pts{\rmA_{-\pts{s_1:s_t}\cup[\sk]}}} \leq \frac{cn}{\tlambda_{1} r_{0}\pts{\cov_{-\pts{s_1:s_t}\cup[\sk]}}},\\
        \Tr\pts{\rmA_{-\pts{s_1:s_t}\cup[\sk]}^{-1}} &\geq \frac{n}{c\tlambda_{1} r_{0}\pts{\cov_{-\pts{s_1:s_t}\cup[\sk]}}},
    \end{align*}
    where we denote $\{\tlambda_j\}_{j=1}^{d - |\pts{s_1:s_t}\cup[\sk]|}$ the diagonal entries of the leave-$t$ and $\sk$-out covariance operator $\cov_{-\pts{s_1:s_t}\cup[\sk]}$.
    By substituting the bounds of $\Tr\pts{\rmA_{-\pts{s_1:s_t}\cup[\sk]}^{-1}}$ into~\Eqref{eq:su_upper_bound3} and~\Eqref{eq:su_lower_bound3}, the survival proof is done. Next, we show the convergence of $\SU_{s_1}$. From~\Eqref{eq:su_upper_bound3}, since $t \ll n^{\frac{1}{2}-\epsilon}$ and $\epsilon\in(0, \frac{1}{4})$, for $n,d\to\infty$, we have
    \begin{align*}
        \lim_{n,d\to\infty} \SU_{s_1} &\leq \lim_{n,d\to\infty}\sqrt{\frac{2}{\pi\var}} \pts{\frac{ \pts{1 + \frac{c_5t}{n^{\frac{1}{2}-\epsilon}}}\lambda_{s_1}\Tr\pts{\rmA_{-\pts{s_1:s_t}\cup[\sk]}^{-1}}}{1 + \pts{1 - \frac{c_6}{n^{\frac{1}{2}-\epsilon}}}\lambda_{s_1}\Tr\pts{\rmA_{-\pts{s_1:s_t}\cup[\sk]}^{-1}}}}\\
        &= \sqrt{\frac{2}{\pi\var}}\frac{\lambda_{s_1}\Tr\pts{\rmA_{-\pts{s_1:s_t}\cup[\sk]}^{-1}}}{1+\lambda_{s_1}\Tr\pts{\rmA_{-\pts{s_1:s_t}\cup[\sk]}^{-1}}}.
    \end{align*}
    Similarly, for the lower bound of $\SU_{s_1}$, from~\Eqref{eq:su_lower_bound3}, we have
    \begin{align*}
        \lim_{n,d\to\infty} \SU_{s_1} &\geq \lim_{n,d\to\infty}\sqrt{\frac{2}{\pi\var}} \pts{\frac{ \pts{1 - \frac{c_5t}{n^{\frac{1}{2}-\epsilon}}}\lambda_{s_1}\Tr\pts{\rmA_{-\pts{s_1:s_t}\cup[\sk]}^{-1}}}{1 + \pts{1 + \frac{c_6}{n^{\frac{1}{2}-\epsilon}}}\lambda_{s_1}\Tr\pts{\rmA_{-\pts{s_1:s_t}\cup[\sk]}^{-1}}}}\\
        &= \sqrt{\frac{2}{\pi\var}}\frac{\lambda_{s_1}\Tr\pts{\rmA_{-\pts{s_1:s_t}\cup[\sk]}^{-1}}}{1+\lambda_{s_1}\Tr\pts{\rmA_{-\pts{s_1:s_t}\cup[\sk]}^{-1}}}.
    \end{align*}
    Therefore, we can conclude the convergence of $\SU_{s_1}$ by its matching upper and lower bounds.
    
    Next, we prove the contamination upper bound, and the proof follows the proof idea of Lemma 5 in~\cite{wang2023benign} closely. We start with the classification MNI with the indices not supported in $\svtheta$. For $j\in\gS^\mathsf{c}$, we have
    \begin{align}
        \hevtheta_j &= \sqrt{\lambda_j}\rvz_j^\top\rmA^{-1}\hrvy\nonumber\\
        &= \sqrt{\lambda_j}\rvz_j^\top\pts{\rmA_{-s_1:s_1}^{-1} - \frac{\lambda_{s_1}\rmA_{-s_1:s_1}^{-1}\rvz_{s_1}\rvz_{s_1}^\top\rmA_{-s_1:s_1}^{-1}}{1+\lambda_{s_1}\rvz_{s_1}^\top\rmA_{-s_1:s_1}^{-1}\rvz_{s_1}}}\hrvy\nonumber\\
        &= \sqrt{\lambda_j}\rvz_j^\top\rmA_{-s_1:s_1}^{-1}\underbrace{\pts{\hrvy - \frac{\lambda_{s_1}\rvz_{s_1}^\top\rmA_{-s_1:s_1}^{-1}\hrvy}{1+\lambda_{s_1}\rvz_{s_1}^\top\rmA_{-s_1:s_1}^{-1}\rvz_{s_1}}\rvz_{s_1}}}_{\eqcolon\crvy_{s_1}}\nonumber\\
        &= \sqrt{\lambda_j}\rvz_j^\top\rmA_{-s_1:s_1}^{-1}\crvy_{s_1}\nonumber\\
        &= \sqrt{\lambda_j}\rvz_j^\top\rmA_{-s_1:s_t}^{-1}\crvy_{s_t}\label{eq:htheta_nonsupport2},
    \end{align}
    where we apply the Sherman-Morrison-Woodbury identity recursively and also denote $\crvy_{s_\ell} \coloneqq \crvy_{s_{\ell-1}} - \frac{\lambda_{s_\ell}\rvz_{s_\ell}^\top\rmA_{-s_1:s_\ell}^{-1}\crvy_{s_{\ell-1}}}{1+\lambda_{s_\ell}\rvz_{s_\ell}^\top\rmA_{-s_1:{s_\ell}}^{-1}\rvz_{s_\ell}}\rvz_{s_\ell}$ and $\crvy_{s_0} = \hrvy$ for $\ell \in [t]$. Next, we take the square of $\hevtheta_j$ and according to the $\CN$ definition in Definition~\ref{def:su_cn}, we have $\hevtheta_j^2 = \lambda_j \crvy_{s_t}^\top \rmA_{-s_1:s_t}^{-1}\rvz_j\rvz_j^\top\rmA_{-s_1:s_t}^{-1}\crvy_{s_t}$ and therefore
    \begin{align}
        \CN^2 &= \sum_{j\in\gS^\mathsf{c}}\lambda_j\hevtheta_j^2\nonumber\\
        &= \sum_{j\in\gS^\mathsf{c}}\lambda_j^2\crvy_{s_t}^\top \rmA_{-s_1:s_t}^{-1}\rvz_j\rvz_j^\top\rmA_{-s_1:s_t}^{-1}\crvy_{s_t}\nonumber\\
        &= \crvy_{s_t}^\top\underbrace{\rmA_{-s_1:s_t}^{-1}\pts{\sum_{j\in\gS^\mathsf{c}}\lambda_j^2\rvz_j\rvz_j^\top}\rmA_{-s_1:s_t}^{-1}}_{\coloneqq \tilde{\rmC}}\crvy_{s_t}\nonumber.
    \end{align}
    Next, we apply the triangle inequality such that $\sqrt{\pts{\rvx-\rvy}^\top\rmM\pts{\rvx-\rvy}} \leq \sqrt{\rvx^\top\rmM\rvx} + \sqrt{\rvy^\top\rmM\rvy}$ $t$ times to decompose $\crvy_{s_t}$ and get
    \begin{align}
        \CN &= \sqrt{\crvy_{s_t}^\top \tilde{\rmC}\crvy_{s_t}} \leq \sqrt{\hrvy^\top\tilde{\rmC}\hrvy} + \sum_{\ell=1}^t \sqrt{\pts{\frac{\lambda_{s_\ell}\rvz_{s_\ell}^\top\rmA_{-s_1:s_\ell}^{-1}\crvy_{s_{\ell-1}}}{1+\lambda_{s_\ell}\rvz_{s_\ell}^\top\rmA_{-s_1:s_\ell}^{-1}\rvz_{s_\ell}}}^2\rvz_{s_\ell}^\top\tilde{\rmC}\rvz_{s_\ell}}\label{eq:cn_upper_bound}.
    \end{align}
    Since $\gS=\{s_1,\cdots,s_t\}$, $\hrvy$ and all $\rvz_{s_\ell}$ are independent to $\tilde{\rmC}$, we can apply the Hanson-Wright inequality in Lemma~\ref{lem:hanson_wright_ineq} to obtain
    \begin{align*}
        \hrvy^\top\tilde{\rmC}\hrvy &\leq \Tr\pts{\tilde{\rmC}}\pts{1+\frac{1}{c}}\ln\pts{n} \quad \tn{and} \\
        \rvz_{s_\ell}^\top\tilde{\rmC}\rvz_{s_\ell} &\leq \Tr\pts{\tilde{\rmC}}\pts{1+\frac{1}{c}}\ln\pts{n},
    \end{align*}
    with probability at least $1-\frac{1}{n}$.
    Substitute these inequalities into~\Eqref{eq:cn_upper_bound}, we get
    \begin{align*}
        \CN &\leq  \pts{1+\sum_{\ell=1}^t\left|\frac{\lambda_{s_\ell}\rvz_{s_\ell}^\top\rmA_{-s_1:s_\ell}^{-1}\crvy_{s_{\ell-1}}}{1+\lambda_{s_\ell}\rvz_{s_\ell}^\top\rmA_{-s_1:s_\ell}^{-1}\rvz_{s_\ell}}\right|}\sqrt{\Tr\pts{\tilde{\rmC}}\pts{1+\frac{1}{c}}\ln\pts{n}}\\
        &\leq \pts{1+tc_2}\sqrt{\Tr\pts{\tilde{\rmC}}\pts{1+\frac{1}{c}}\ln\pts{n}},
    \end{align*}
    where we apply Lemma~\ref{lem:cn_const} in the last inequality. It remains to upper bound $\Tr\pts{\tilde{\rmC}}$ to complete the proof of the upper bound of $\CN$. Then we can use Lemma 11 in~\cite{bartlett2020benign} to show
    \begin{align}
        \Tr\pts{\tilde{\rmC}} &= \Tr\pts{\rmA_{-s_1:s_t}^{-1}\pts{\sum_{j\in\gS^\mathsf{c}}\lambda_j^2\rvz_j\rvz_j^\top}\rmA_{-s_1:s_t}^{-1}}\nonumber\\
        &= \sum_{j\in\gS^\mathsf{c}}\lambda_j^2\rvz_j^\top \rmA_{-s_1:s_t}^{-2}\rvz_j\nonumber\\
        &= \sum_{j\leq \sk, j\in\gS^\mathsf{c}}\lambda_j^2\rvz_j^\top \rmA_{-s_1:s_t}^{-2}\rvz_j + \sum_{j> \sk, j\in\gS^\mathsf{c}}\lambda_j^2\rvz_j^\top \rmA_{-s_1:s_t}^{-2}\rvz_j\nonumber\\
        &\leq c\pts{\frac{\sk}{n} + n\frac{\sum_{j=\sk+1, j\in\gS^\mathsf{c}}^d \lambda_j^2}{\pts{\sum_{j=\sk+1,j\in\gS^\mathsf{c}}^d\lambda_j}^2}}.\nonumber%
    \end{align}
    The proof of the upper bound is completed.
\end{proof}

\section{Zero-shot task shift in the case of sparse signal}
In this section, we provide proofs from Section~\ref{sec:zeroshot_sparse} concerning our sparse signal model.
In Section~\ref{app:proof_zero_shot_general_vacuous}, we leverage Lemma~\ref{lem:k_sparse_SU_CN_bounds} to prove the convergence of the regression risk of the classification MNI for general covariances.
The results of Section~\ref{app:proof_zero_shot_general_vacuous} are left in terms of the inverse leave-$t$-out Gram matrix; in the following sections, we provide more precise derivations for specific covariance ensembles.
In Section~\ref{app:diff_cov}, we provide closed-form expressions for the limiting regression risk of the classification MNI under the spiked covariance model.
Finally, in Section~\ref{app:poly_cov}, we provide corresponding expressions for the more delicate case of the polynomial decay covariance model.

\subsection{Characterization of zero-shot task shift for general covariance}\label{app:proof_zero_shot_general_vacuous}
In this section, we provide the proof of Theorem~\ref{thm:zero_shot_general_vacuous}.

\begin{proof} (Theorem~\ref{thm:zero_shot_general_vacuous})

The excess risk expression of~\Eqref{eq:risk_su_cn} in terms of $\{\SU_j\}_{j \in \gS}$ and $\CN$ gives us
    \begin{align*}
        \risk{\hvtheta} &= \sum_{j\in\gS} a_j^2\pts{\SU_j - 1}^2 + \CN^2.
    \end{align*}
    We now characterize the limiting regression risk $\lim_{n,d\to\infty}\risk{\hvtheta}$. Since the covariance matrix is benign and satisfies $\lim_{n,d\to\infty}\frac{t^2 \cdot \sk \cdot \ln\pts{n}}{n}=\lim_{n,d\to\infty}\frac{t^2\cdot n \cdot \ln\pts{n}}{R_0\pts{\cov_{-\bgS}}}=0$, we can apply Lemma~\ref{lem:k_sparse_SU_CN_bounds}.
    We then have
    \begin{align*}
        \lim_{n,d\to\infty} \CN^2 \leq \lim_{n,d\to\infty} t^2\pts{\frac{\sk}{n} + \frac{n}{R_0\pts{\cov_{-\bgS}}}}\ln\pts{n} = 0.
    \end{align*}
    Since $\CN^2 \geq 0$ by definition, we have $\lim_{n,d \to \infty} \CN^2 = 0$.
    Hence, we have $\lim_{n,d \to \infty} \risk{\hvtheta}  = \sum_{j\in\gS} a_j^2\pts{\SU_j - 1}^2$. 
    Then, Lemma~\ref{lem:k_sparse_SU_CN_bounds} tells us that
    \begin{align*}
        \lim_{n,d\to\infty} \SU_j = \sqrt{\frac{2}{\pi\var}}\cdot \frac{\lambda_j\Tr\pts{\rmA_{-\bgS}^{-1}}}{1 + \lambda_j\Tr\pts{\rmA_{-\bgS}^{-1}}}.
    \end{align*}
    Denoting $b_j \coloneqq \lim_{n,d\to\infty}\lambda_j\Tr\pts{\rmA_{-\bgS}^{-1}}$ as shorthand and putting it all together, we have
    \begin{align*}
        \lim_{n,d\to\infty}\risk{\hvtheta} = \lim_{n,d\to\infty} \sum_{j\in\gS} a_j^2\pts{\SU_j - 1}^2 = \sum_{j\in\gS} a_j^2\pts{\lim_{n,d\to\infty}\SU_j - 1}^2 = \sum_{j\in\gS}a_j^2\pts{\sqrt{\frac{2}{\pi \var}}\frac{b_j}{1+b_j} - 1}^2.
    \end{align*}
    This completes the proof of the theorem.
\end{proof}
\subsection{Closed-form expressions for spiked covariance} \label{app:diff_cov}

In this section, we provide the proof of Corollary~\ref{cor:diff_cov}.

\begin{proof} (Corollary~\ref{cor:diff_cov})

First of all, we show that the following limit holds
\begin{align*}
    \lim_{n,d\to\infty}\frac{t^2\cdot \sk \cdot \ln\pts{n}}{n}=\lim_{n,d\to\infty}\frac{t^2\cdot n \cdot \ln\pts{n}}{R_0\pts{\cov_{-\bgS}}}=0.
\end{align*}
This ensures that the assumption in Theorem~\ref{thm:zero_shot_general_vacuous} is satisfied. By the definition of spiked covariance (Definition~\ref{def:spiked}), we observe that $k^\star = s = n^r$ for $q < 1- r$ and $k^\star = 0$ for $q > 1 - r$. This implies $\lim_{n,d\to\infty}\frac{\sk \cdot \ln\pts{n}}{n} = 0$. On the other hand, we also have $R_0\pts{\cov_{-\bgS}} = d - |\bgS| = \frac{n^p}{c}$ which leads to $\lim_{n,d\to\infty}\frac{ n \cdot \ln\pts{n}}{R_0\pts{\cov_{-\bgS}}}=0$. Additionally, the corollary assumption guarantees that $t \ll \min\{\sqrt{\frac{n}{\sk\cdot\ln\pts{n}}} \sqrt{\frac{n^{p-1}}{\ln\pts{n}}}\}$. Combining these results, we conclude that $\lim_{n,d\to\infty}\frac{t^2\cdot \sk \cdot \ln\pts{n}}{n}=\lim_{n,d\to\infty}\frac{t^2\cdot n \cdot \ln\pts{n}}{R_0\pts{\cov_{-\bgS}}}=0$, and therefore, Theorem~\ref{thm:zero_shot_general_vacuous} holds.

Next, recall that $b_j \coloneqq \lim_{n,d\to\infty}\lambda_j\Tr\pts{\rmA_{-\bgS}^{-1}}$, according to Theorem~\ref{thm:zero_shot_general_vacuous}, when $n,d\to\infty$, the limiting regression risk is equal to
    \begin{align}
        \lim_{n,d\to\infty}\risk{\hvtheta} &= \sum_{j\in\gS}a_j^2\pts{\sqrt{\frac{2}{\pi \var}}\frac{b_j}{1+b_j} - 1}^2\nonumber\\
        & = \sum_{j\leq\sk,j\in\gS}a_j^2\pts{\sqrt{\frac{2}{\pi \var}}\frac{b_j}{1+b_j} - 1}^2 +\sum_{j>\sk,j\in\gS}a_j^2\pts{\sqrt{\frac{2}{\pi \var}}\frac{b_j}{1+b_j} - 1}^2,\label{eq:risk_b_j}
    \end{align}
    where we separate the summation into two sections. Therefore, it suffices to characterize $\frac{b_j}{1+b_j}$ for two sections: $j\leq \sk$ and $j>\sk$. In both sections, we have
    \begin{align}
        \frac{b_j}{1+b_j} = \frac{1}{\frac{1}{b_j} + 1} = \lim_{n,d\to\infty}\frac{1}{\pts{\frac{1}{\lambda_j\Tr\pts{\rmA_{-\bgS}^{-1}}}}+1}.\label{eq:b_j_eq}
    \end{align}
    Next, since $r_0\pts{\cov_{-\bgS}} \geq bn$, we can bound $\Tr\pts{\rmA_{-\bgS}^{-1}}$ according to Lemma 10 in~\cite{bartlett2020benign} such that
    \begin{align*}
        \Tr\pts{\rmA_{-\bgS}^{-1}} &= \sum_{i=1}^n\frac{1}{\mu_i\pts{\rmA_{-\bgS}}} \leq \frac{n}{\mu_n\pts{\rmA_{-\bgS}}} \leq \frac{cn}{\tlambda_1 r_0\pts{\cov_{-\bgS}}} = \frac{cn}{\sum_{j=\sk+1,j\in\gS^\mathsf{c}}^d\lambda_j}\\
        \Tr\pts{\rmA_{-\bgS}^{-1}} &= \sum_{i=1}^n\frac{1}{\mu_i\pts{\rmA_{-\bgS}}} \geq \frac{n}{\mu_1\pts{\rmA_{-\bgS}}} \geq \frac{n}{c\tlambda_1 r_0\pts{\cov_{-\bgS}}} = \frac{n}{c\sum_{j=\sk+1,j\in\gS^\mathsf{c}}^d\lambda_j}.
    \end{align*}
    Recall that, in the above, we denoted $\{\tlambda_j\}_{j=1}^{d - |\bgS|}$ the eigenvalues of the leave-$t$ and $\sk$-out covariance matrix $\cov_{-\bgS}$. Therefore, we can upper and lower bound~\Eqref{eq:b_j_eq} as below:
    \begin{align}
        \lim_{n,d\to\infty}\frac{1}{\pts{\frac{c\sum_{k=\sk+1,k\in\gS^\mathsf{c}}^d\lambda_k}{\lambda_jn}}+1} \leq \frac{b_j}{1+b_j} \leq \lim_{n,d\to\infty}\frac{1}{\pts{\frac{\sum_{k=\sk+1,k\in\gS^\mathsf{c}}^d\lambda_k}{c\lambda_jn}}+1}\label{eq:b_j_bounds}.
    \end{align}
    As a result, it suffices to characterize the limit of $\frac{\sum_{k=\sk+1,k\in\gS^\mathsf{c}}^d\lambda_k}{\lambda_jn}$ as $n, d \to \infty$. We discuss our evaluation of the spiked covariance ensemble in two cases.
    \begin{itemize}
        \item \textbf{Spiked covariance with $q < 1-r$ (Definition~\ref{def:spiked}):}\\
        Recall that, for this choice of parameters, the regression MNI would generalize~\citep{muthukumar2021classification}.
        Define $\bar{k}\coloneqq |\bgS|$ as shorthand.
        From Definition~\ref{def:spiked}, we have $\sk=s$.
        Therefore, for $j \in \gS \cap [\sk]$,
        \begin{align*}
            \lim_{n,d\to\infty}\frac{\sum_{k=\sk+1,k\in\gS^\mathsf{c}}^d\lambda_k}{\lambda_jn} &= \lim_{n,d\to\infty}\frac{\pts{d-\bar{k}}\frac{\pts{1-a}d}{d-s}}{\frac{ad}{s}\cdot n} \\
            &= \lim_{n,d\to\infty}\frac{\pts{n^p-\bar{k}}\frac{\pts{1-n^{-q}}n^p}{\pts{n^p-n^r}}}{n^{\pts{p-q-r+1}}}\\
            &= \lim_{n,d\to\infty}\frac{\pts{n^p-\bar{k}}\frac{\pts{1-n^{-q}}}{\pts{1-n^{r-p}}}}{n^{\pts{p-q-r+1}}} \\
            &= 0,
        \end{align*}
        where the last equality followed because $1 - q - r > 0$. As a result, from~\Eqref{eq:b_j_bounds}, we have $\frac{b_j}{1+b_j} = 1$ for $j \in \gS \cap [\sk]$. On the other hand, for $j \in \gS \cap \{\sk + 1,\ldots,d\}$, we have
        \begin{align*}
            \lim_{n,d\to\infty}\frac{\sum_{k=\sk+1,k\in\gS^\mathsf{c}}^d\lambda_k}{\lambda_jn} = \lim_{n,d\to\infty}\frac{\pts{d-\bar{k}}\frac{\pts{1-a}d}{d-s}}{\frac{\pts{1-a}d}{d-s}\cdot n} = \lim_{n,d\to\infty}\frac{\pts{n^p-\bar{k}}}{n} = \infty.
        \end{align*}
        Therefore, from~\Eqref{eq:b_j_bounds}, we have $\frac{b_j}{1+b_j} = 0$ for this case. 
        Substituting these values back into~\Eqref{eq:risk_b_j} completes the proof of the corollary for this case.
        \item \textbf{Spiked covariance with $q > 1-r$ (Definition~\ref{def:spiked}):}\\
        Recall that, for this choice of parameters, the regression MNI would \emph{not} generalize~\citep{muthukumar2021classification}.
        In this case, we have $\sk=0$. 
        Define the number of support indices that are contained within the spike as $t_1 := |\gS \cap [s]|$, and note that $0 \leq t_1 \leq \min\{s, t\}$. Therefore, for $j \in \gS \cap [s]$, we have
        \begin{align*}
            \lim_{n,d\to\infty}\frac{\sum_{k=1,k\in\gS^\mathsf{c}}^d\lambda_k}{\lambda_jn} &= \lim_{n,d\to\infty}\frac{\pts{s-t_1}\frac{ad}{s} + \pts{d-s-\pts{t-t_1}}\frac{\pts{1-a}d}{d-s}}{\frac{ad}{s}\cdot n}\\
            &= \lim_{n,d\to\infty}\frac{\pts{n^r-t_1}n^{\pts{p-q-r}} + \pts{n^p-n^r-\pts{t-t_1}}\frac{\pts{1-n^{-q}}n^p}{\pts{n^p-n^r}}}{n^{\pts{p-q-r+1}}} \\
            &= \infty,
        \end{align*}
        where the last equality follows because $1 - q - r < 0$. As a result, from~\Eqref{eq:b_j_bounds}, $\frac{b_j}{1+b_j} = 0$ for $j \in \gS\cap[s]$. Similarly, for $j \in \gS \cap \{s+1,\ldots,d\}$, we have
        \begin{align*}
            \lim_{n,d\to\infty}\frac{\sum_{k=\sk+1,k\in\gS^\mathsf{c}}^d\lambda_k}{\lambda_jn} &\geq \lim_{n,d\to\infty}\frac{\sum_{k=\sk+1,k\in\gS^\mathsf{c}}^d\lambda_k}{\lambda_1 n} = \infty,
        \end{align*}
        where the last equality follows from the preceding calculation. Substituting these values back into~\Eqref{eq:risk_b_j} completes the proof for this case.
    \end{itemize}
\end{proof}

\subsection{Closed-form expressions for polynomial decay covariance} \label{app:poly_cov}
Next, we provide corresponding expressions for the limiting regression risk for the more delicate case of the polynomial decay covariance ensemble. It is worth noting that, unlike the spiked covariance case—where regression consistency can be achieved when the true signal is confined to the top $\sk$ components and the signal magnitude is fixed at $\var = \frac{2}{\pi}$ (Corollary~\ref{cor:diff_cov})—the polynomial decay covariance case requires stricter conditions. Specifically, the true signal must be restricted to the top $\tk_1$ components, which depend on both the polynomial decay parameters and the sample size $n$.
\begin{corollary}\label{cor:poly_cov}
    Under Assumptions~\ref{asm:rank} and~\ref{asm:k_sparse} with $t \ll \min\bigg\{\sqrt{\frac{n}{\sk\cdot\ln\pts{n}}}, \sqrt{\frac{R_0\pts{\cov_{-\bgS}}}{n\cdot\ln\pts{n}}}\bigg\}$ and polynomial decay covariance (Definition~\ref{def:polynomial}),
    \begin{itemize}
        \item For $u\in[0,1)$, $v=0$, and $p\cdot\pts{1-u} > 1$, we have
            \begin{align*}
                \lim_{n,d\to\infty}\risk{\hvtheta} = \sum_{j\in\gS}a_j^2.
            \end{align*}
            In this case, we have $\risk{\hvtheta} > 0$, and regression consistency is not possible unless we have zero signal, i.e. $\svtheta = \bm{0}$.
         \item For $u\in(0, 1)$, $v=0$, and $p\cdot\pts{1-u} < 1$, we have
            \begin{align*}
                \lim_{n,d\to\infty}\risk{\hvtheta} &\leq \sum_{j\leq \tk_1,j\in\gS}a_j^2\pts{\sqrt{\frac{2}{\pi \var}} - 1}^2 + \abar{C}\sum_{\tk_1<j<\tk_2,j\in\gS}a_j^2 + \sum_{\tk_2\leq j,j\in\gS}a_j^2,\\
                \lim_{n,d\to\infty}\risk{\hvtheta} &\geq \sum_{j\leq \tk_1,j\in\gS}a_j^2\pts{\sqrt{\frac{2}{\pi \var}} - 1}^2 + \ubar{C}\sum_{\tk_1<j<\tk_2,j\in\gS}a_j^2 + \sum_{\tk_2\leq j,j\in\gS}a_j^2,
            \end{align*}
            we define $\tk_1 \coloneqq \max \left\{k\geq 0: k =o\pts{n^{\frac{1-p\cdot\pts{1-u}}{u}}}\right\}$ and $\tk_2\coloneqq \min \left\{k\geq 0: k =\omega\pts{n^{\frac{1-p\cdot\pts{1-u}}{u}}}\right\}$, and
            \begin{align*}
                \abar{C} \coloneqq \max_{\tk_1<j<\tk_2, j\in\gS} \pts{\sqrt{\frac{2}{\pi \var}}c_j - 1}^2, \quad\quad \ubar{C} \coloneqq \min_{\tk_1<j<\tk_2, j\in\gS} \pts{\sqrt{\frac{2}{\pi \var}}c_j - 1}^2.
            \end{align*}
            Recalling $b_j \coloneqq \lim_{n,d\to\infty}\lambda_j\Tr\pts{\rmA_{-\bgS}^{-1}}$, we define $c_j \coloneqq \frac{b_j}{1+b_j}$, and $\{c_j\}_{\tk_1<j<\tk_2, j\in\gS}$ is a non-increasing sequence in $\pts{0,1}$. This implies regression consistency \emph{if and only if} the signal magnitude is fixed at $\var = \frac{2}{\pi}$ and $a_j = 0$ for all $j >\tk_1, j \in \gS$.
            Note that benign overfitting of the regression MNI is attained for this choice of parameters~\citep{bartlett2020benign}.
    \end{itemize}
    
\end{corollary}
\begin{proof}
    As in the proof of Corollary~\ref{cor:diff_cov}, we first need to show $\lim_{n,d\to\infty}\frac{t^2\cdot \sk \cdot \ln\pts{n}}{n}=\lim_{n,d\to\infty}\frac{t^2\cdot n \cdot \ln\pts{n}}{R_0\pts{\cov_{-\bgS}}}=0$, and therefore, Theorem~\ref{thm:zero_shot_general_vacuous} holds. Then, we can write the limit of the risk in~\Eqref{eq:risk_b_j} as
    \begin{align}
        \lim_{n,d\to\infty}\risk{\hvtheta} &= \sum_{j\in\gS}a_j^2\pts{\sqrt{\frac{2}{\pi \var}}\frac{b_j}{1+b_j} - 1}^2\nonumber\\
        &=\sum_{j\leq\tk_1,j\in\gS}a_j^2\pts{\sqrt{\frac{2}{\pi \var}}\frac{b_j}{1+b_j} - 1}^2 + \sum_{\tk_1<j<\tk_2,j\in\gS}a_j^2\pts{\sqrt{\frac{2}{\pi \var}}\frac{b_j}{1+b_j} - 1}^2\nonumber\\
        &\myquad[20]+\sum_{\tk_2\leq j,j\in\gS}a_j^2\pts{\sqrt{\frac{2}{\pi \var}}\frac{b_j}{1+b_j} - 1}^2,\label{eq:risk_b_j2}
    \end{align}
    where we separate the summation into three sections. Next, we can bound the value of $\frac{b_j}{1+b_j}$ for each section by~\Eqref{eq:b_j_bounds} as
    \begin{align}
        \lim_{n,d\to\infty}\frac{1}{\pts{\frac{c\sum_{k=\sk+1,k\in\gS^\mathsf{c}}^d\lambda_k}{\lambda_jn}}+1} \leq \frac{b_j}{1+b_j} \leq \lim_{n,d\to\infty}\frac{1}{\pts{\frac{\sum_{k=\sk+1,k\in\gS^\mathsf{c}}^d\lambda_k}{c\lambda_jn}}+1}\label{eq:b_j_bounds2}.
    \end{align}
    Hence, for each $j$, it suffices to characterize the limit of $\frac{\sum_{k=\sk+1,k\in\gS^\mathsf{c}}^d\lambda_k}{\lambda_jn}$ as $n, d \to \infty$ in order to characterize the value of $\frac{b_j}{1+b_j}$.
    The proof follows the idea of Theorem 31 in~\cite{bartlett2020benign} closely.
    \begin{itemize}
        \item \textbf{Polynomial decay covariance with $u\in[0, 1)$, $v=0$, $p\cdot \pts{1-u} > 1$ (Definition~\ref{def:polynomial}):}\\
            Firstly, we show that this choice of parameters satisfies the assumption in Theorem~\ref{thm:zero_shot_general_vacuous} such that $\lim_{n,d\to\infty}\frac{t^2\cdot \sk \cdot \ln\pts{n}}{n}=\lim_{n,d\to\infty}\frac{t^2\cdot n \cdot \ln\pts{n}}{R_0\pts{\cov_{-\bgS}}}=0$. From Definition~\ref{def:polynomial}, we have $\lambda_j = \frac{1}{j^u}$ and we can show $\sum_{j=1}^d \lambda_j \asymp d^{1-u} = n^{p\cdot\pts{1-u}}$ by showing
            \begin{subequations}
                \begin{align}
                    \sum_{j=1}^d \lambda_j &= 1 + \sum_{j=2}^d \frac{1}{j^u} \leq 1 + \int_{1}^d \frac{1}{x^u}dx = 1 + \frac{x^{1-u}}{1-u}\Big|_{1}^{d} = 1 + \frac{d^{1-u}-1}{1-u}\label{eq:poly_upperbound}\\
                    \sum_{j=1}^d \lambda_j &= \sum_{j=1}^d \frac{1}{j^u} \geq \sum_{j=1}^d\frac{1}{d^u} = d^{1-u}\label{eq:poly_lowerbound}.
                \end{align}
            \end{subequations}
            Therefore, we can derive $r_0(\cov) = \frac{\sum_{j=1}^d \lambda_j}{\lambda_1} = \sum_{j=1}^d \lambda_j \geq d^{1-u} = n^{p\cdot\pts{1-u}} > n$, so $\sk=0$. Next, we need to verify $\lim_{n,d\to\infty}\frac{t^2\cdot n \cdot \ln\pts{n}}{R_0\pts{\cov_{-\gS}}}=0$, and we have
            \begin{align*}
                \frac{1}{R_0(\cov_{-\gS})} = \frac{\sum_{j=1,j\in\gS^\mathsf{c}}^d\lambda_j^2}{\pts{\sum_{j=1,j\in\gS^\mathsf{c}}^d\lambda_j}^2} = \frac{\sum_{j=1,j\in\gS^\mathsf{c}}^d\frac{1}{j^{2u}}}{\pts{\sum_{j=1,j\in\gS^\mathsf{c}}^d\frac{1}{j^u}}^2} \leq \frac{ \pts{1 + \int_{1}^d \frac{1}{x^{2u}}dx}}{\pts{\sum_{j=1,j\in\gS^\mathsf{c}}^d\frac{1}{d^u}}^2} = \frac{\pts{1+\int_{1}^d \frac{1}{x^{2u}}dx}}{\pts{d-t}^2\cdot d^{-2u}}.
            \end{align*}
            The integral $\int_{1}^d \frac{1}{x^{2u}}dx$ varies for different values of $u$, and we have
            \begin{align*}
                \int_{1}^d \frac{1}{x^{2u}}dx = \condds{\frac{x^{1-2u}}{1-2u}\Big|_{1}^{d} = \frac{d^{1-2u}-1}{1-2u},}{\text{for $u\in[0, 0.5)$,}}{\ln\pts{x}\Big|_{1}^{d}=\ln(d),}{\text{for $u=0.5$,}}{\frac{x^{1-2u}}{1-2u}\Big|_{1}^{d} = \frac{1-d^{1-2u}}{2u-1},}{\text{for $u\in\pts{0.5, 1}$}.}
            \end{align*}
            Therefore, we can upper bound $\frac{n\ln(n)}{R_0(\cov_{-\gS})}$ as
            \begin{align*}
                \frac{n\ln(n)}{R_0(\cov_{-\gS})} = \condds{\frac{n\ln(n)\pts{1+\frac{d^{1-2u}-1}{1-2u}}}{\pts{d-t}^2\cdot d^{-2u}} \leq \frac{c_1n\ln(n)}{d} = \frac{c_1n\ln(n)}{n^p},}{\text{for $u\in[0, 0.5)$,}}{\frac{n\ln(n)\pts{1+\ln\pts{d}}}{\pts{d-t}^2\cdot d^{-1}}\leq \frac{c_2n\ln(n)\ln\pts{d}}{d} = \frac{c_2n\cdot p\ln^2\pts{n}}{n^p},}{\text{for $u=0.5$,}}{\frac{n\ln(n)\pts{1+\frac{1-d^{1-2u}}{2u-1}}}{\pts{d-t}^2\cdot d^{-2u}} \leq \frac{c_3n\ln(n)}{d^{2-2u}} = \frac{c_3n\ln(n)}{n^{p\pts{2-2u}}},}{\text{for $u\in\pts{0.5, 1}$}.}
            \end{align*}
            Note that since $p\cdot \pts{1-u} > 1$, the limit of all three cases goes to zero by L'H\^{o}pital's Rule, and we have  $\lim_{n,d\to\infty}\frac{n\ln(n)}{R_0\pts{\cov_{-\gS}}} = 0$ for $u\in[0, 1)$. Combining these results with the corollary assumption $t \ll \min\bigg\{\sqrt{\frac{n}{\sk\cdot\ln\pts{n}}}, \sqrt{\frac{R_0\pts{\cov_{-\bgS}}}{n\cdot\ln\pts{n}}}\bigg\}$, we have $\lim_{n,d\to\infty}\frac{t^2\cdot \sk \cdot \ln\pts{n}}{n}=\lim_{n,d\to\infty}\frac{t^2\cdot n \cdot \ln\pts{n}}{R_0\pts{\cov_{-\bgS}}}=0$ and Theorem~\ref{thm:zero_shot_general_vacuous} holds. Next, we characterize the limit  $\frac{\sum_{k=1,k\in\gS^\mathsf{c}}^d\lambda_k}{\lambda_jn}$ as $n, d \to \infty$ to evaluate $\frac{b_j}{1+b_j}$. For $j\in[d]$, since $p\cdot \pts{1-u} > 1$, we have
            \begin{align*}
                \lim_{n,d\to\infty}\frac{\sum_{k=1,k\in\gS^\mathsf{c}}^d\lambda_k}{\lambda_jn} \geq \lim_{n,d\to\infty}\frac{ d^{\pts{1-u}}}{c\lambda_1n} =  \lim_{n,d\to\infty}\frac{n^{p\pts{1-u}}}{cn} = \infty.
            \end{align*}
            By substituting the value of $\lim_{n,d\to\infty}\frac{\sum_{k=1,k\in\gS^\mathsf{c}}^d\lambda_k}{\lambda_jn}$ into~\Eqref{eq:b_j_bounds2}, we get $\frac{b_j}{1+b_j} = 0$ for $j\in[d]$. Finally, we substitute $\frac{b_j}{1+b_j} = 0$ into~\Eqref{eq:risk_b_j2} and complete the proof for this case.
        \item \textbf{Polynomial decay covariance with $u \in (0, 1)$, $v=0$, $p\cdot \pts{1-u} < 1$ (Definition~\ref{def:polynomial}):}\\
            Recall that, for this choice of parameters, we attain benign overfitting in regression~\citep[Theorem 31]{bartlett2020benign}. This result implies that $\lim_{n,d\to\infty}\frac{\sk}{n}=\lim_{n,d\to\infty}\frac{n}{R_0\pts{\cov_{-\bgS}}}=0$. However, to ensure that Theorem~\ref{thm:zero_shot_general_vacuous} holds, we require a stronger condition $\lim_{n,d\to\infty}\frac{t^2\cdot \sk \cdot \ln\pts{n}}{n}=\lim_{n,d\to\infty}\frac{t^2\cdot n \cdot \ln\pts{n}}{R_0\pts{\cov_{-\bgS}}}=0$. We first check the rate of $\sk$. For any $k\in[d-2]$ and $\lambda_j = j^{-u}$, we have
            \begin{align*}
                F(d) - F(k+1) = \int_{k+1}^d x^{-u} dx \leq \sum_{j=k+1}^d \lambda_j \leq \int_{k}^d x^{-u} dx = F(d) - F(k),
            \end{align*}
            where $F(x) = \frac{x^{1-u}}{1-u}$. Therefore, we have $\sum_{j=k+1}^d \lambda_j = \mathcal{O}(d^{1-u})$. Next, we can calculate the rate of the effective rank $r_k=\frac{\sum_{j=k+1}^d \lambda_j}{\lambda_{k+1}} = \mathcal{O}(k^u \cdot d^{1-u})$. According to the definition of $\sk$ such that $\sk \coloneqq \min \{k\geq 0: r_k(\cov) \geq bn\}$, we have $\sk = \mathcal{O}(n^{\frac{1}{u}}\cdot d^{\frac{u-1}{u}}) = \mathcal{O}(n^{\frac{1-p\cdot\pts{1-u}}{u}})$, where $\frac{1-p\cdot\pts{1-u}}{u} < 1$ since $p > 1$ and $p\cdot \pts{1-u} < 1$. This result implies $\lim_{n,d\to\infty}\frac{\sk \cdot \ln\pts{n}}{n} = 0$. Next, we show that $\lim_{n,d\to\infty}\frac{n \cdot \ln\pts{n}}{R_0\pts{\cov_{-\bgS}}}=0$. Similar to the rate of $\sum_{j=k+1}^d \lambda_j$, we have the rate of $\sum_{j=k+1}^d \lambda_j^2$ as 
            \begin{align*}
                \sum_{j=k+1}^d \lambda_j^2=\condds{\mathcal{O}(d^{1-2u}),}{\text{for $u\in(0, 0.5)$,}}{\mathcal{O}\left(\ln\left(\frac{d}{k}\right)\right),}{\text{for $u=0.5$,}}{\mathcal{O}(k^{1-2u}),}{\text{for $u \in(0.1, 1)$.}}
            \end{align*}
            Therefore, recalling that $\sum_{j=k+1}^d \lambda_j = \mathcal{O}(d^{1-u})$, we can calculate $R_0\pts{\cov_{-\bgS}}$ as
            \begin{align*}
                R_0\pts{\cov_{-\bgS}}=\frac{\pts{\sum_{j=\sk+1,j\in\gS^\mathsf{c}}^d \lambda_j}^2}{\sum_{j=\sk+1,j\in\gS^\mathsf{c}}^d \lambda_j^2} =\condds{\mathcal{O}(d)=\mathcal{O}(n^p),}{\text{for $u\in(0, 0.5)$,}}{\mathcal{O}\left(\frac{d}{\ln\left(\frac{d}{\sk}\right)}\right)=\mathcal{O}\left(\frac{n^p}{\ln\pts{n}}\right),}{\text{for $u=0.5$,}}{\mathcal{O}\left(d^{2-2u}\cdot k^{{\star}^{2u-1}}\right)=\mathcal{O}\left(n^{2-\frac{1-p\pts{1-u}}{u}}\right),}{\text{for $u\in\pts{0.5, 1}$.}}
            \end{align*}
            As a result, we can conclude that
            \begin{align*}
                \frac{n\ln({n)}}{R_0(\cov_{-\gS})} = \condds{\mathcal{O}\left(\frac{\ln(n)}{n^{p-1}}\right),}{\text{for $u\in[0, 0.5)$,}}{\mathcal{O}\left(\frac{\ln^2(n)}{n^{p-1}}\right),}{\text{for $u=0.5$,}}{\mathcal{O}\left(\frac{\ln(n)}{n^{1-\frac{1-p\pts{1-u}}{u}}}\right),}{\text{for $u\in\pts{0.5, 1}$}.}
            \end{align*}
            Since $\frac{1-p\pts{1-u}}{u} < 1$, the limit of all three cases goes to zero by L'H\^{o}pital's Rule, and we have  $\lim_{n,d\to\infty}\frac{n\ln(n)}{R_0\pts{\cov_{-\gS}}} = 0$ for $u\in(0, 1)$. Combining these results with the corollary assumption $t \ll \min\bigg\{\sqrt{\frac{n}{\sk\cdot\ln\pts{n}}}, \sqrt{\frac{R_0\pts{\cov_{-\bgS}}}{n\cdot\ln\pts{n}}}\bigg\}$, we have $\lim_{n,d\to\infty}\frac{t^2\cdot \sk \cdot \ln\pts{n}}{n}=\lim_{n,d\to\infty}\frac{t^2\cdot n \cdot \ln\pts{n}}{R_0\pts{\cov_{-\bgS}}}=0$ and Theorem~\ref{thm:zero_shot_general_vacuous} holds. Next, from~\Eqref{eq:risk_b_j2}, we will present our calculation in three sections: $j\leq \tk_1$, $\tk_1 < j < \tk_2$ and $\tk_2 \leq j$, where we recall $\tk_1 \coloneqq \max \left\{k\geq 0: k =o\pts{n^{\frac{1-p\cdot\pts{1-u}}{u}}}\right\}$ and $\tk_2\coloneqq \min \left\{k\geq 0: k =\omega\pts{n^{\frac{1-p\cdot\pts{1-u}}{u}}}\right\}$. We first consider the case $j \leq \tk_1$, and we have
            \begin{align*} 
                \lim_{n,d\to\infty}\frac{\sum_{k=\sk+1,k\in\gS^\mathsf{c}}^d\lambda_k}{\lambda_jn} \leq  \lim_{n,d\to\infty}\frac{\sum_{j=1}^d\lambda_j}{\lambda_{\tk_1}n} \leq \lim_{n,d\to\infty}\frac{cd^{1-u}}{\lambda_{\tk_1}n} = \lim_{n,d\to\infty}\frac{cn^{p\cdot\pts{1-u}}}{\lambda_{\tk_1}n} = \lim_{n,d\to\infty}\frac{c\cdot\tk_1^u}{n^{1-p\cdot\pts{1-u}}} = 0,
            \end{align*}
            where we apply~\Eqref{eq:poly_upperbound} in the second inequality. Therefore, by~\Eqref{eq:b_j_bounds2}, we have $\frac{b_{j}}{1 + b_{j}} = 1$ for $j \leq \tk_1$. Next, for $\tk_1 < j < \tk_2$, we have
            \begin{align*}
                \lim_{n,d\to\infty}\frac{\sum_{k=\sk+1,k\in\gS^\mathsf{c}}^d\lambda_k}{\lambda_jn} &\leq  \lim_{n,d\to\infty}\frac{\sum_{j=1}^d\lambda_j}{\lambda_{\tk_2 - 1}n} \leq \lim_{n,d\to\infty}\frac{cd^{1-u}}{\lambda_{\tk_2 - 1}n} = \lim_{n,d\to\infty}\frac{c\cdot\pts{\tk_2-1}^u}{n^{1-p\cdot\pts{1-u}}} < \infty,\\
                \lim_{n,d\to\infty}\frac{\sum_{k=\sk+1,k\in\gS^\mathsf{c}}^d\lambda_k}{\lambda_jn} &\geq \lim_{n,d\to\infty}\frac{d^{1-u}}{c\lambda_{\tk_1 + 1}n} = \lim_{n,d\to\infty}\frac{\pts{\tk_1+1}^u}{cn^{1-p\cdot\pts{1-u}}} > 0,
            \end{align*}
            by applying $\sum_{k=\sk+1,k\in\gS^\mathsf{c}}^d\lambda_k \asymp d^{1-u}$ and the definition of $\tk_1$ and $\tk_2$. As a result, according to~\Eqref{eq:b_j_bounds2}, we can derive $0 < \frac{b_j}{1 + b_j} < 1$ for $\tk_1 < j < \tk_2$. Finally, for $\tk_2 \leq j$, we have
            \begin{align*}
                \lim_{n,d\to\infty}\frac{\sum_{k=\sk+1,k\in\gS^\mathsf{c}}^d\lambda_k}{\lambda_jn} \geq  \lim_{n,d\to\infty}\frac{d^{1-u}}{c\lambda_{\tk_2}n} = \lim_{n,d\to\infty}\frac{\tk_2^u}{cn^{1-p\cdot\pts{1-u}}} = \infty.
            \end{align*}
            Therefore, according to~\Eqref{eq:b_j_bounds2}, we have $\frac{b_{j}}{1 + b_{j}} = 0$ for $\tk_2 \leq j$. Substituting the value of $\frac{b_{j}}{1 + b_{j}}$ into~\Eqref{eq:risk_b_j2} for all three cases completes the proof.
    \end{itemize}
\end{proof}

\section{Zero-shot task shift in the case of random signal} \label{app:zeroshot_dense}

In this section, we provide proofs from Section \ref{sec:zeroshot_dense} concerning our random signal model.
In Section \ref{app:decomposition} we detail the decomposition of the regression risk of the classification MNI into regression bias and task shift error.
In Section \ref{app:benign}, we provide the main proofs of the upper and lower bounds on task shift error in our random signal model.
In Section \ref{app:benign_labels}, we bound the deviation of the classification and regression labels.
Then, in Section \ref{app:benign_xcx}, we adapt the benign overfitting analysis of~\cite{bartlett2020benign} to our dependent noise setting.
Finally, in Section \ref{app:benign_lower}, we prove a fundamental tradeoff between the statistical consistency of regression bias and task shift error.

\subsection{Task shift error decomposition} \label{app:decomposition}

In this section, we provide the proof of Lemma~\ref{lem:decomposition}.

\begin{proof} (Lemma~\ref{lem:decomposition})

Recall that the MNI on regression labels and the MNI on classification labels are defined as
\begin{align*}
    \tvtheta &\coloneqq \argmin \bts{\norm{\vtheta}_2:\rmX\vtheta=\trvy}, \\
    \hvtheta &\coloneqq \argmin \bts{\norm{\vtheta}_2:\sgn{\rmX\vtheta}=\hrvy},
\end{align*}
respectively, and they have closed forms
\begin{align*}
    \tvtheta &= \XXXty, \\
    \hvtheta &= \XXXhy.
\end{align*}

Now, we have
\begin{align*}
    \risk{\hvtheta} &= \E_\rvx \pts{\rvx^\top\hvtheta-\rvx^\top\svtheta}^2 \\
    &= \E_\rvx \pts{\pts{\rvx^\top\tvtheta-\rvx^\top\svtheta}+\pts{\rvx^\top\hvtheta-\rvx^\top\tvtheta}}^2 \\
    &= \risk{\tvtheta} + \E_\rvx \pts{\rvx^\top\hvtheta-\rvx^\top\tvtheta}^2 + 2 \E_\rvx \mts{\pts{\rvx^\top\tvtheta-\rvx^\top\svtheta} \pts{\rvx^\top\hvtheta-\rvx^\top\tvtheta}}.
\end{align*}
We will show that the third term is precisely zero. Write
\begin{equation*}
    \rmE \coloneqq \pts{\rmX^\top \XXinv \rmX - \mI} \cov \rmX^\top \XXinv. 
\end{equation*}
Substituting $\trvy\coloneqq\rmX\svtheta$ and the closed-form expressions for the minimum-norm interpolators,
\begin{align*}
    &\E_\rvx \mts{\pts{\rvx^\top\tvtheta-\rvx^\top\svtheta} \pts{\rvx^\top\hvtheta-\rvx^\top\tvtheta}} \\
    &\qquad = \E_\rvx \mts{\pts{\rvx^\top\rmX^\top \XXinv \rmX \svtheta - \rvx^\top \svtheta} \pts{\rvx^\top \XXXhy - \rvx^\top \XXXty}} \\
    &\qquad = \E_\rvx \mts{\pts{\rvx^\top \pts{\rmX^\top \XXinv \rmX - \mI} \svtheta} \pts{\rvx^\top \rmX^\top \XXinv (\hrvy - \trvy)}} \\
    &\qquad = \vtheta^{\star\top} \rmE (\hrvy - \trvy).
\end{align*}

Recall that $\hrvy-\trvy\coloneqq \rmD\rmX \svtheta$ where $\rmD\coloneqq\diag{\tn{d}_1,\dots,\tn{d}_n}$ and $\tn{d}_i\coloneqq\frac{\sgn{\rvx^\top\svtheta}-\rvx^\top\svtheta}{\rvx^\top\svtheta}$.
Note $\rmE \rmD \rmX$ is nilpotent: we have $\rmX \rmE = \vzero$, so $(\rmE \rmD \rmX)^2 = \vzero$, and hence all eigenvalues of $\rmE \rmD \rmX$ are zero.
Therefore,
\begin{equation*}
    \vtheta^{\star\top} \rmE (\hrvy - \trvy) = \vtheta^{\star\top} \rmE \rmD \rmX \svtheta = 0,
\end{equation*}
which completes the proof.
\end{proof}

\subsection{Upper and lower bounds on task shift error} \label{app:benign}
In this section, we provide the proof of Theorem \ref{thm:benign}.
First, let us define
\begin{equation*}
    \rmC \coloneqq \XXXcovXXX.
\end{equation*}
Recall also that $\hrvy-\trvy\coloneqq \rmD\rmX \svtheta$ where $\rmD\coloneqq\diag{\tn{d}_1,\dots,\tn{d}_n}$ and $\tn{d}_i\coloneqq\frac{\sgn{\rvx^\top\svtheta}-\rvx^\top\svtheta}{\rvx^\top\svtheta}$.

We now introduce several lemmas.
The first lemma upper bounds the deviation of the classification and regression labels.
The proof is in Appendix \ref{app:benign_labels}.
\begin{lemma}\label{lem:benign_labels}
    For any $\cov$ and $\svtheta$, there exists a constant $c>1$ such that
    \begin{equation*}
        \norm{\hrvy-\trvy}_2^2 \leq cn \var
    \end{equation*}
    with probability at least $1-e^{-\frac{n}{c}}$.
\end{lemma}

The remaining lemmas bound traces involving $\rmC$ and $\rmX$ in high probability.
The first of these is a restatement of the main technical result of~\cite{bartlett2020benign}.
\begin{lemma}\label{lem:benign_c}
    For any $\cov$ and $\svtheta$, there exist constants $c,c_1\geq 1$ such that the following hold. If $\sk<\frac{n}{c_1}$, we have
    \begin{equation*}
        \frac{1}{c} \pts{\frac{\sk}{n} + \frac{n}{R_{\sk}(\cov)}} \leq \Tr(\rmC) \leq c \pts{\frac{\sk}{n} + \frac{n}{R_{\sk}(\cov)}}
    \end{equation*}
    with probability at least $1-17e^{-\frac{n}{c}}$.
    On the other hand, if $\sk\geq \frac{n}{c_1}$, we have
    \begin{equation*}
        \Tr(\rmC) \geq \frac{1}{c}
    \end{equation*}
    with probability at least $1-10e^{-\frac{n}{c}}$.
\end{lemma}

The final lemma characterizes the key additional term arising in our dependent noise model.
The proof is in Appendix \ref{app:benign_xcx}.
\begin{lemma}\label{lem:benign_xcx}
    For any $\cov$ and $\svtheta$, there exist constants $c,c_1\geq 1$ such that the following hold. If $\sk<\frac{n}{c_1}$, we have
    \begin{equation*}
        \frac{1}{c} \pts{\sum_{j=1}^{\sk} \lambda_j + \frac{n}{R_{\sk}(\cov)}\sum_{j=\sk+1}^d\lambda_j} \leq \Tr(\rmC \rmX \rmX^\top) \leq c\pts{\sum_{j=1}^{\sk} \lambda_j + \frac{n}{R_{\sk}(\cov)}\sum_{j=\sk+1}^d\lambda_j}
    \end{equation*}
    with probability at least $1-14e^{-\frac{n}{c}}$.
    On the other hand, if $\sk\geq \frac{n}{c_1}$, we have
    \begin{equation*}
        \Tr(\rmC \rmX \rmX^\top) \geq \frac{1}{c}
    \end{equation*}
    with probability at least $1-10e^{-\frac{n}{c}}$.
\end{lemma}

We are now ready to prove Theorem \ref{thm:benign}.

\begin{proof} (Theorem \ref{thm:benign})

We begin with the upper bound.
Substituting the closed-form expressions for the minimum-norm interpolators,
\begin{align}
    \tserror &= \E_\rvx \pts{\rvx^\top \XXXhy-\rvx^\top \XXXty}^2 \nonumber \\
    &= \E_\rvx \pts{\rvx^\top \rmX^\top \XXinv (\hrvy - \trvy)}^2 \nonumber \\
    &= (\hrvy - \trvy)^\top \rmC (\hrvy - \trvy). \label{eq:benign_yCy}
\end{align}
By definition of operator norm and trace,
\begin{equation*}
    (\hrvy - \trvy)^\top \rmC (\hrvy - \trvy) \leq \norm{\rmC} \yynorm \leq \Tr(\rmC) \yynorm.
\end{equation*}

Since $\sk < \frac{n}{c_1}$, by Lemma \ref{lem:benign_c}, there exists a constant $c_2\geq 1$ such that
\begin{equation} \label{eq:benign_trace}
    \Tr(\rmC) \leq c_2 \pts{\frac{\sk}{n} + \frac{n}{R_\sk(\cov)}}
\end{equation}
with probability at least $1-17e^{-\frac{n}{c_2}}$.
Moreover, by Lemma \ref{lem:benign_labels}, there exists a constant $c_3\geq 1$ such that
\begin{equation} \label{eq:benign_yynorm}
    \yynorm \leq c_3n\var
\end{equation}
with probability at least $1-e^{-\frac{n}{c_3}}$.
Combining Equations \ref{eq:benign_trace} and \ref{eq:benign_yynorm} with a union bound completes the upper bound.

Now we will prove the lower bound.
Starting from \Eqref{eq:benign_yCy} and using the assumption that $\rmD=\alpha\mI$,
\begin{align*}
    \rtserror &= \E_{\svtheta} (\hrvy - \trvy)^\top \rmC (\hrvy - \trvy) \\
    &= \E_{\svtheta} \svthetat \rmX^\top \rmD \rmC \rmD \rmX \svtheta \\
    &= \alpha \E_{\svtheta} \svthetat \rmX^\top \rmC \rmX \svtheta.
\end{align*}
By assumption, $\E_{\svtheta} \svtheta\svthetat \geq \sigma^2 \mI$. Using the cyclic and linear properties of trace,
\begin{equation}
   \alpha \E_{\svtheta} \svthetat \rmX^\top \rmC \rmX \svtheta = \alpha \Tr (\rmX^\top \rmC \rmX \E_{\svtheta} \svtheta\svthetat) \geq \alpha \sigma^2 \Tr(\rmX^\top \rmC \rmX). \label{eq:benign_tr}
\end{equation}

If $\sk < \frac{n}{c_1}$, then by Lemma \ref{lem:benign_xcx} there exists a constant $c_4\geq 1$ such that
\begin{equation}
    \Tr(\rmX^\top\rmC \rmX) \geq \frac{1}{c_4}\pts{\sum_{j=1}^{\sk} \lambda_j + \frac{n}{R_{\sk}(\cov)}\sum_{j=\sk+1}^d\lambda_j} \label{eq:benign_lower_1}
\end{equation}
with probability at least $1-14e^{-\frac{n}{c_4}}$.
On the other hand, if $\sk>\frac{n}{c_1}$, then by Lemma \ref{lem:benign_xcx},
\begin{equation}
    \Tr(\rmX^\top\rmC \rmX) \geq \frac{1}{c_4} \label{eq:benign_lower_2}
\end{equation}
with probability at least $1-10e^{-\frac{n}{c_4}}$.
Substituting Equations \ref{eq:benign_lower_1} and \ref{eq:benign_lower_2} into \Eqref{eq:benign_tr} and choosing $c=\max(c_2c_3,c_4)$ completes the proof.
\end{proof}

\subsection{Deviation of classification and regression labels} \label{app:benign_labels}
In this section, we provide the proof of Lemma \ref{lem:benign_labels}.

\begin{proof} (Lemma \ref{lem:benign_labels})

First, note that
\begin{align*}
    \yynorm &= \norm{\hrvy}_2^2 + \norm{\trvy}_2^2 - 2\sum_{j=1}^n \try_i\hry_i \\
    &= n + \norm{\trvy}_2^2 - 2\norm{\trvy}_1.
\end{align*}
Write $\psi^2\coloneqq \var$ so that $\try_i\sim \gN(0, \psi^2)$ for all $1\leq i \leq n$. In particular,
\begin{equation} \label{eq:labels_expectation}
    \E_\rmX \yynorm = n + n\psi^2 - \sqrt{\frac{8}{\pi}}n\psi. 
\end{equation}

For $\gamma>1$, define the sub-Gaussian and sub-exponential norms of a random variable $X$ by
\begin{align*}
    \norm{X}_{\psi_2} &\coloneqq \inf \bts{t>0:\E_X \exp(X^2/t^2)<\gamma} \quad \tn{and} \\
    \norm{X}_{\psi_1} &\coloneqq \inf \big\{t>0:\E_X \exp(|X|/t)<\gamma\big\},
\end{align*}
respectively.

By definition, $\try_i$ is sub-Gaussian for any $i$ with $\norm{\try_i}_{\psi_2}=\xi \psi$ where $\xi > 0$.
Moreover, because $\try_i$ is symmetric, $\hry_i$ is Bernoulli and therefore sub-Gaussian with $\norm{\hry_i}_{\psi_2}=\xi$.
Combining terms, $\hry_i-\try_i$ is sub-Gaussian with $\norm{\hry_i-\try_i}_{\psi_2}=\xi+\xi\psi$, and $(\hry_i-\try_i)^2$ is sub-exponential with $\norm{(\hry_i-\try_i)^2}_{\psi_1}\leq(\xi+\xi\psi)^2$.

By Bernstein's inequality, there exists a constant $c_1>0$ such that for any $t\geq 0$,
\begin{equation*}
    \sP_\rmX \pts{\yynorm - \E_\rmX\yynorm \geq t} \leq \exp\pts{-c_1\min\pts{\frac{t^2}{n(\xi+\xi\psi)^4}, \frac{t}{(\xi+\xi\psi)^2}}}.
\end{equation*}
Hence,
\begin{equation} \label{eq:labels_bernstein}
    \yynorm \leq \E_\rmX\yynorm + \sqrt{\frac{n}{c_1}}\max\pts{\sqrt{\frac{n}{c_1}}, \sqrt{n}}(\xi+\xi\psi)^2
\end{equation}
with probability at least $1-e^{-n}$.
Defining $c_2\coloneqq\max(c_1^{-1}, c_1^{-1/2})$ and using $(a+b)^2\leq 2(a^2+b^2)$ for any $a,b\in \R$, we have
\begin{align} 
    \sqrt{\frac{n}{c_1}}\max\pts{\sqrt{\frac{n}{c_1}}, \sqrt{n}}(\xi+\xi\psi^2) &\leq c_2 n(\xi+\xi\psi)^2 \nonumber \\
    &\leq 2c_2 n\xi^2(1+\psi^2). \label{eq:labels_residual}
\end{align}

Substituting Equations \ref{eq:labels_expectation} and \ref{eq:labels_residual} into \Eqref{eq:labels_bernstein} and dropping the negative term, we have
\begin{align*}
    \yynorm &\leq n + n\psi^2 + 2c_2 n\xi^2(1+\psi^2) \\
    &\leq c_2 (1+2\xi^2)(1+\psi^2)n.
\end{align*}
Treating $\psi^2$ as a constant, there exists a constant $c_3\geq 1$ such that $1+\psi^2 \leq c_3\psi^2$. Hence,
\begin{equation*}
    \yynorm \leq c_2c_3(1+2\xi^2)n\psi^2.
\end{equation*}
Choosing $c=c_2 c_3$ completes the proof.
\end{proof}

\subsection{Benign overfitting analysis in our dependent noise model} \label{app:benign_xcx}
In this section, we provide the proof of Lemma \ref{lem:benign_xcx}.
First, we write a convenient representation of the data matrix $\rmX$.
Let $\rmZ\in\R^{n\times d}$ have \iid standard Gaussian elements, then $\rmX \deq \rmZ \cov^{\frac{1}{2}}$.
The matrix $\rmZ$ is known as the \emph{whitened} data matrix.
We write $\rvz_j\in \R^n$ to denote the $j^{th}$ column of $\rmZ$.

Using this notation, write
\begin{equation*}
    \rmA \coloneqq \ZcovZ = \sum_{j=1}^d \lambda_j \rvz_j \rvz_j^\top,
\end{equation*}
and similarly,
\begin{equation*}
    \ZcovtZ = \sum_{j=1}^d \lambda_j^2 \rvz_j\rvz_j^\top.
\end{equation*}
Note that $\XX\deq \rmA$.
Finally, let $\rmA_{-k}\coloneqq \sum_{j\neq k} \lambda_j \rvz_j\rvz_j^\top$ denote the \emph{leave-one-out} Gram matrix for some $1\leq k \leq d$.

We will use the following lemma in the proof.
It is a short consequence of the Sherman-Woodbury-Morrison identity applied to a matrix-vector product.
\begin{lemma}\label{lem:benign_woodbury}
    Suppose $\rmA$ and $\rvz_j$ are defined as above for some $1\leq j \leq d$. Then,
    \begin{equation*}
        \rmA^{-1}\rvz_j = \frac{\rmA_{-j}^{-1}\rvz_j}{1+\lambda_j\rvz_j^\top\rmA_{-j}^{-1}\rvz_j}.
    \end{equation*}
\end{lemma}
\begin{proof} By the Sherman-Woodbury-Morrison identity,
\begin{equation*}
    \rmA^{-1} = \rmA_{-j}^{-1} - \frac{\lambda_j\rmA_{-j}^{-1}\rvz_j\rvz_j^\top\rmA_{-j}^{-1}}{1+\lambda_j\rvz_j^\top \rmA_{-j}^{-1}\rvz_j}.
\end{equation*}
Multiplying by $\rvz_j$ on the right we obtain
\begin{align*}
    \rmA^{-1}\rvz_j &= \rmA_{-j}^{-1}\rvz_j - \frac{\lambda_j\rmA_{-j}^{-1}\rvz_j\rvz_j^\top\rmA_{-j}^{-1}\rvz_j}{1+\lambda_j\rvz_j^\top \rmA_{-j}^{-1}\rvz_j} \\
    &= \frac{\rmA_{-j}^{-1}\rvz_j(1+\lambda_j\rvz_j^\top \rmA_{-j}^{-1}\rvz_j)-\lambda_j\rmA_{-j}^{-1}\rvz_j\rvz_j^\top\rmA_{-j}^{-1}\rvz_j}{1+\lambda_j\rvz_j^\top \rmA_{-j}^{-1}\rvz_j} \\
    &= \frac{\rmA_{-j}^{-1}\rvz_j}{1+\lambda_j\rvz_j^\top \rmA_{-j}^{-1}\rvz_j},
\end{align*}
as desired.
\end{proof}

Now, we begin the main proof of this section.
\begin{proof} (Lemma \ref{lem:benign_xcx})

Substituting the whitened data matrix,
\begin{align}
    \Tr \pts{\rmX^\top \rmC \rmX} &= \Tr \pts{\rmX \cov \rmX^\top \XXinv} \nonumber \\
    &\deq \Tr \pts{\rmA^{-1} \rmZ \cov^2 \rmZ^\top} \nonumber \\
    &= \sum_{j=1}^d \lambda_j^2 \rvz_j^\top \rmA^{-1} \rvz_j. \label{eq:benign_st}
\end{align}
It is instructive to compare this term with Lemma 8 of~\cite{bartlett2020benign}.
In contrast to their formulation, we have one fewer $\rmA^{-1}$ in the center of the expression due to our dependent noise model.
The remainder of the proof utilizes a leave-one-out technique similarly to~\cite{bartlett2020benign}, except the $j^{th}$ summand carries an additional $n\lambda_j$.
Applying Lemma \ref{lem:benign_woodbury},
\begin{equation*}
    \sum_{j=1}^d \lambda_j^2 \rvz_j^\top \rmA^{-1} \rvz_j = \sum_{j=1}^{\sk} \frac{\lambda_j^2 \rvz_j^\top \rmA^{-1}_{-j} \rvz_j}{1+\lambda_j\rvz_j^\top \rmA^{-1}_{-j} \rvz_j} + \sum_{j=\sk+1}^d \lambda_j^2 \rvz_j^\top \rmA^{-1} \rvz_j. 
\end{equation*}
Clearly, the first term is at most $\sum_{j=1}^{\sk} \lambda_j$.
For the second term, we have
\begin{equation*}
    \sum_{j=\sk+1}^d \lambda_j^2 \rvz_j^\top \rmA^{-1} \rvz_j \leq \frac{\sum_{j=\sk+1}^d \lambda_j^2 \norm{\rvz_j}^2_2}{\mu_n(\rmA)}.
\end{equation*}
By Lemma 10 of~\cite{bartlett2020benign}, since $\sk<\frac{n}{c_1}$, there exists a constant $c_2\geq 1$ such that $\mu_n(\rmA)\geq \frac{\lambda_{\sk+1}r_{\sk}(\cov)}{c_2}$ with probability at least $1-2e^{-\frac{n}{c_2}}$.
Moreover, by Lemma 12 of~\cite{bartlett2020benign}, there exists a constant $c_3\geq 1$ such that
\begin{equation*}
    \sum_{j=\sk+1}^d \lambda_j^2 \norm{\rvz_j}^2_2 \leq c_3 n \sum_{j=\sk+1}^d \lambda_j^2
\end{equation*}
with probability at least $1-2e^{-\frac{n}{c_3}}$.
By a union bound,
\begin{align*}
    \Tr \pts{\rmX^\top \rmC \rmX} &\leq \sum_{j=1}^{\sk} \lambda_j + \frac{c_2 c_3 n\sum_{j=\sk+1}^d \lambda_j^2}{\lambda_{\sk+1}r_{\sk}(\cov)} \\
    &\leq c_2c_3 \pts{\sum_{j=1}^{\sk} \lambda_j + \frac{n}{R_{\sk}(\cov)}\sum_{j=\sk+1}^{d} \lambda_j}
\end{align*}
with probability at least $1-4e^{-\frac{n}{c_2c_3}}$, which completes the upper bound.

We now prove the lower bound.
Beginning from \Eqref{eq:benign_st} and applying Lemma \ref{lem:benign_woodbury},
\begin{equation*}
    \sum_{j=1}^d \lambda_j^2 \rvz_j^\top \rmA^{-1} \rvz_j = \sum_{j=1}^d \frac{\lambda_j^2 \rvz_j^\top \rmA^{-1}_{-j} \rvz_j}{1+\lambda_j\rvz_j^\top \rmA^{-1}_{-j} \rvz_j}. 
\end{equation*}
By Lemma 14 of~\cite{bartlett2020benign}, for any $k<\frac{n}{c_1}$, there exists a constant $c_4\geq 1$ such that for any $1\leq j \leq d$,
\begin{equation*}
    1+\lambda_j\rvz_j^\top \rmA^{-1}_{-j} \rvz_j \leq c_4 \lambda_j \rvz_j^\top \rmA^{-1}_{-j} \rvz_j \pts{\frac{n\lambda_j + \lambda_{k+1}r_k(\cov)+n\lambda_{k+1}}{n\lambda_j}}
\end{equation*}
with probability at least $1-5e^{-\frac{n}{c_4}}$.
By Lemma 15 of~\cite{bartlett2020benign},
\begin{align*}
    \sum_{j=1}^d \frac{\lambda_j^2 \rvz_j^\top \rmA^{-1}_{-j} \rvz_j}{1+\lambda_j\rvz_j^\top \rmA^{-1}_{-j} \rvz_j} &\geq \sum_{j=1}^d \frac{\lambda_j^2 \rvz_j^\top \rmA^{-1}_{-j} \rvz_j}{c_4\lambda_j\rvz_j^\top \rmA^{-1}_{-j} \rvz_j} \pts{\frac{n\lambda_j}{n\lambda_j +\lambda_{k+1}r_k(\cov)+n\lambda_{k+1}}} \\
    &= \frac{1}{c_4} \sum_{j=1}^d \frac{n\lambda_j^2}{n\lambda_j+\lambda_{k+1}r_k(\cov)+n\lambda_{k+1}}
\end{align*}
with probability at least $1-10e^{-\frac{n}{c_4}}$.
By the mediant inequality, there exist constants $b,c_5\geq 1$ such that
\begin{equation*}
    \frac{1}{c_4} \sum_{j=1}^d \frac{n\lambda_j^2}{\lambda_{k+1}r_k(\cov)+n\lambda_{k+1} + n\lambda_j} \geq \frac{1}{bc_5} \sum_{j=1}^d \min \pts{\lambda_j, \frac{bn\lambda_j^2}{\lambda_{k+1}r_k(\cov)}, \frac{\lambda_j^2}{\lambda_{k+1}}}.
\end{equation*}

If $r_k(\cov)<bn$, the second term in the minimum is larger than the third term.
In this case,
\begin{align}
    \frac{1}{bc_5} \sum_{j=1}^d \min \pts{\lambda_j \frac{bn\lambda_j^2}{\lambda_{k+1}r_k(\cov)}, \frac{\lambda_j^2}{\lambda_{k+1}}} &\geq \frac{1}{bc_5} \sum_{j=1}^d \min \pts{\lambda_j, \frac{\lambda_j^2}{\lambda_{k+1}}} \nonumber \\
    &= \frac{1}{bc_5} \pts{\sum_{j=1}^k \lambda_j + \sum_{j=k+1}^d \frac{\lambda_j^2}{\lambda_{k+1}}}. \label{eq:benign_xcx}
\end{align}
On the other hand, if $r_k(\cov)\geq bn$, the second term in the minimum is smaller than the third term.
In this case,
\begin{align*}
    \frac{1}{bc_5} \sum_{j=1}^d \min \pts{\lambda_j, \frac{bn\lambda_j^2}{\lambda_{k+1}r_k(\cov)}, \frac{\lambda_j^2}{\lambda_{k+1}}} &\geq \frac{1}{bc_5} \sum_{j=1}^d \min \pts{\lambda_j, \frac{bn\lambda_j^2}{\lambda_{k+1}r_k(\cov)}} \\
    &= \frac{1}{bc_5} \min_{1\leq \ell \leq k} \pts{\sum_{j=1}^\ell \lambda_j + \frac{bn\sum_{j=\ell+1}^d\lambda_j^2}{\lambda_{k+1}r_k(\cov)}}.
\end{align*}

Recall that $\sk \coloneqq \min \{ k\geq 0: r_k(\cov) \geq bn\}$.
By Lemma 17 of~\cite{bartlett2020benign}, there exists a constant $c_6\geq 1$ such that
\begin{equation*}
    \frac{1}{bc_5} \min_{1\leq \ell \leq k} \pts{\sum_{j=1}^\ell \lambda_j + \frac{bn\sum_{j=\ell+1}^d\lambda_j^2}{\lambda_{k+1}r_k(\cov)}} = \frac{1}{c_4} \pts{\sum_{j=1}^{\sk} \lambda_j + \frac{n\sum_{j=\sk+1}^d\lambda_j^2}{\lambda_{k+1}r_k(\cov)}}.
\end{equation*}
Therefore, if $\sk < \frac{n}{c_1}$,
\begin{equation*}
    \Tr(\rmX^\top \rmC \rmX) \geq \frac{1}{c_6} \pts{\sum_{j=1}^{\sk} \lambda_j + \frac{n}{R_{\sk}(\cov)} \sum_{j=\sk+1}^d\lambda_j}.
\end{equation*}
On the other hand, if $\sk > \frac{n}{c_1}$, then $r_k(\cov)<bn$ for all $k\leq \frac{n}{c_1}$.
\Eqref{eq:benign_xcx} then implies
\begin{equation*}
    \Tr(\rmX^\top \rmC \rmX) \geq \frac{1}{bc_5} \pts{\sum_{j=1}^{\frac{n}{c_1}} \lambda_j + \sum_{j=\frac{n}{c_1}+1}^d \frac{\lambda_j^2}{\lambda_{\frac{n}{c_1}+1}}} \geq \frac{1}{bc_5}.
\end{equation*}
Choosing $c=\max(c_2c_3,c_4,bc_5)$ and taking a union bound over the upper and lower bounds completes the proof.

\end{proof}

\subsection{Tradeoff between regression bias and task shift error} \label{app:benign_lower}
In this section, we provide the proof of Theorem \ref{thm:benign_lower}.

\begin{proof} (Theorem \ref{thm:benign_lower})

Note that for the given signal model we have $\E_{\svtheta} \theta_j^{\star 2}=\bar{\theta}_j^2$ and $\bar{\theta}_j^2\geq 1$ for all $1\leq j \leq d$.
Hence, the assumption of the Theorem \ref{thm:benign} lower bound is satisfied with $\sigma^2=1$.

We first prove the case where $\kappa = 0$.
In this case, there is some finite $n$ after which $r_0(\cov) \geq bn$ for a constant $b>1$.
By the lower bound of Lemma \ref{lem:bias}, there exists a constant $c_1\geq 1$ such that
\begin{equation*}
    \E_{\svtheta} \risk{\tvtheta} \geq \frac{1}{c_1} \sum_{j=1}^d \frac{\lambda_j \bar{\theta}_j^2}{\pts{1+\frac{n\lambda_j}{\sum_{j=1}^d \lambda_j}}^2}
\end{equation*}
with probability at least $1-c_1 e^{-\frac{n}{c_1}}$.
Since $r_0(\cov)\geq bn$, for any $1\leq j \leq d$ we have
\begin{equation*}
    \frac{n\lambda_j}{\sum_{j=1}^d \lambda_j} \leq \frac{n\lambda_1}{\sum_{j=1}^d \lambda_j} \leq \frac{1}{b}.
\end{equation*}
Therefore,
\begin{equation*}
    \E_{\svtheta} \risk{\tvtheta} \geq \frac{1}{c_1} \sum_{j=1}^d \frac{\lambda_j \bar{\theta}_j^2}{\pts{1+\frac{1}{b}}^2} = \frac{\varbar}{c_1\pts{1+\frac{1}{b}}^2}.
\end{equation*}
In the limit as $n\to\infty$, the term $c_1 e^{-\frac{n}{c_1}}$ is zero.
Hence,
\begin{equation*}
    \lim_{n,d\to\infty} \E_{\svtheta} \risk{\tvtheta} \geq \frac{\varbar}{c_1\pts{1+\frac{1}{b}}^2}
\end{equation*}
almost surely.

We now prove the case where $0<\kappa<n$.
By assumption, $\varbar$ is constant for all $n$ and $\bar{\theta}_j^2\geq 1$ for all $1\leq j \leq d$.
Hence, there exists a constant $c_2\geq 1$ such that either (i) $\lim_{n,d\to\infty} \sum_{j=1}^\kappa \lambda_j\geq \frac{1}{c_2}$ or (ii) $\lim_{n,d\to\infty} \sum_{j=\kappa+1}^d \lambda_j\geq \frac{1}{c_2}$.
(In other words, we cannot have both terms go to zero).

In case (i), by the lower bound of Theorem \ref{thm:benign}, there exist constants $c_3,c_4\geq 1$ such that if $k^\star<\frac{n}{c_3}$, then
\begin{equation*}
    \rtserror \geq \frac{\alpha^2}{c_4} \pts{\sum_{j=1}^{\sk} \lambda_j + \frac{n}{R_{\sk}(\cov)}\sum_{j=\sk+1}^d\lambda_j} \geq \frac{\alpha^2}{c_4}\sum_{j=1}^{\sk} \lambda_j
\end{equation*}
with probability at least $1-14e^{-\frac{n}{c_4}}$,
In the limit as $n\to\infty$, the term $14e^{-\frac{n}{c_4}}$ is zero, and any $\kappa<n$ satisfies the condition $\kappa<\frac{n}{c_3}$.
Using the assumption of case (i),
\begin{align*}
    \lim_{n,d\to\infty} \rtserror \geq \lim_{n,d\to\infty}\frac{\alpha^2}{c_4}\sum_{j=1}^{\kappa} \lambda_j \geq \frac{\alpha^2}{c_2c_4}
\end{align*}
almost surely.

In case (ii), since case (i) is not satisfied, we have $\lambda_1\bar{\theta}_1^2,\dots,\lambda_\kappa\bar{\theta}_\kappa^2\to 0$.
But by assumption, $\bar{\theta}_j^2\geq 1$ for all $1\leq j \leq \kappa$, so then $\lambda_1,\dots,\lambda_{\kappa}\to 0$.
This implies $\cov\to\bm{0}$, a contradiction with positive-definiteness.

Finally, we prove the case where $\kappa>n$.
By the lower bound of Theorem \ref{thm:benign}, we have
\begin{equation*}
    \rtserror \geq \frac{\alpha^2}{c_4}
\end{equation*}
with probability at least $1-10e^{-\frac{n}{c_4}}$.
In the limit as $n\to\infty$, the term $10e^{-\frac{n}{c_4}}$ is zero.
Hence,
\begin{equation*}
    \lim_{n,d\to\infty} \rtserror \geq \frac{\alpha^2}{c_4}
\end{equation*}
almost surely.
Choosing $c=\max\pts{c_1\pts{1+\frac{1}{b}}^2,c_2c_4}$ completes the proof.
\end{proof}

\section{Zero-shot task shift in the case of dense signal}\label{app:dense}

We conclude the random signal section with a model which does not require the ansatz $\rmD\approx \alpha\mI$ introduced in Section~\ref{sec:zeroshot_dense}. In this section, we first introduce the settings and characterize the task shift of the dense signal model in Section~\ref{app:dense_model}.
In Section \ref{app:benign_dense}, we prove that dense signal implies poor bias, and we reduce the task shift error to a benign overfitting term via high-dimensional probability arguments.
In Section \ref{app:benign_arcsin}, we show concentration of task shift error terms via an adaptation of standard sub-Gaussian random matrix analysis.

\subsection{Dense random signal model without simplifying ansatz}\label{app:dense_model}
We study a ``dense'' signal, \ie one which has similar magnitude in all dimensions.
Specifically, we let $\theta^\star_j \sim \mathcal{N}\pts{0, \frac{1}{d\lambda_j}}$ for all $j\in [d]$ so that $\var=1$.
Writing $\rmX\deq\rmZ \cov^{1/2}$ where $\rmZ\in \R^{n\times d}$ has independent standard Gaussian entries, we can see that
\begin{equation*}
    \trvy = \rmX\svtheta \deq \pts{\rmZ \cov^{1/2}}\pts{\frac{1}{\sqrt{d}}\cov^{-1/2}\bm{z}}=\frac{1}{\sqrt{d}}\rmZ\bm{z},
\end{equation*}
where $\bm{z}\in\R^d$ is a standard Gaussian vector with independent entries.
Therefore, this setting of $\svtheta$ is equivalent in distribution to scaled Gaussian random signal under isotropic covariance, clearly a ``dense'' problem instance.

In this regime, we show that while the limiting bias is nonzero, the limiting task shift error is zero as long as the covariance matrix has large effective rank compared to $n$.
Note that this condition is necessary, but not sufficient, for $\cov$ to exhibit benign overfitting.
\begin{theorem} \label{thm:benign_dense}
    For any $\cov$ there exists a constant $c\geq 1$ such that the following holds.
    Suppose $\svtheta$ is such that $\theta^\star_j \sim \mathcal{N}\pts{0, \frac{1}{d\lambda_j}}$ for all $j\in [d]$.
    Then we have $\lim_{n,d\to\infty} \E_{\svtheta} L(\tvtheta) \geq \frac{1}{c}$ but $\lim_{n,d\to\infty} \E_{\svtheta,\rvx}\pts{\rvx^\top \hvtheta - \rvx^\top \tvtheta}^2 \leq c \lim_{n,d\to\infty} \pts{\frac{\sk}{n} + \frac{n}{R_{\sk}(\cov)}}$ almost surely.
    In particular, if $\lim_{n,d\to\infty} \frac{\sk}{n}=\lim_{n,d\to\infty}\frac{n}{R_{\sk}(\cov)}= 0$, then \\ $\lim_{n,d\to\infty} \E_{\svtheta,\rvx}\pts{\rvx^\top \hvtheta - \rvx^\top \tvtheta}^2=0$ almost surely.
\end{theorem}
The proof of Theorem \ref{thm:benign_dense} is in the following Appendix \ref{app:benign_dense}.

\subsection{Analysis of bias and task shift error via benign overfitting} \label{app:benign_dense}
In this section, we provide the proof of Theorem \ref{thm:benign_dense}, which lower bounds the bias and upper bounds the task shift error of a ``dense'' random signal.

\begin{proof} (Theorem \ref{thm:benign_dense})

We begin by characterizing the bias term $\E_{\svtheta} L(\tvtheta)$.
Since the Gaussian distribution is symmetric, we may write $\theta^\star_j=\rr_j\bar{\theta}_j$ where each $\rr_j$ is an independent Rademacher random variable and $\bar{\theta}_j$ is drawn according to a Gaussian distribution.
Applying the lower bound of Lemma \ref{lem:bias}, for a constant $c_1\geq 1$ we have
\begin{equation*}
    \E_{\svtheta} L(\tvtheta) \geq \frac{1}{c_1}\sum_{j=1}^d \frac{\lambda_j \E \bar{\theta}^2_j}{\pts{1+\frac{n\lambda_j}{\sum_{k=1}^d \lambda_k}}^2}
\end{equation*}
with probability at least $1-c_1 e^{-\frac{n}{c_1}}$\footnote{The statement of Lemma \ref{lem:bias} is for deterministic $\bar{\vtheta}$, but the same result holds for $\bar{\vtheta}$ with random coordinates independent of each other and $\bm{r}$ by taking the expectation over $\bar{\vtheta}$.}.
As previously mentioned, setting $\theta^\star_j \sim \mathcal{N}\pts{0, \frac{1}{d\lambda_j}}$ for all $1\leq j \leq d$ under general covariance is equivalent in distribution to setting $\theta^\star_j \sim \mathcal{N}\pts{0, \frac{1}{d}}$ for all $1\leq j \leq d$ under isotropic covariance.
In this case, $\E \bar{\theta}_j^2=\frac{1}{d}$ and $\lambda_1=\cdots = \lambda_d = 1$.
Therefore, we have
\begin{align*}
    \E_{\svtheta} L(\tvtheta) &\geq \frac{1}{dc_1} \sum_{j=1}^d \frac{1}{\pts{1+\frac{n\lambda_j}{\sum_{k=1}^d \lambda_k}}^2} \\
    &= \frac{1}{dc_1} \sum_{j=1}^d \frac{1}{\pts{1+\frac{n}{d}}^2} \\
    &= \frac{1}{c_1}\frac{1}{\pts{1+\frac{n}{d}}^2}.
\end{align*}
In the limit as $n\to\infty$, the term $c_1e^{-\frac{n}{c_1}}$ is zero.
Hence,
\begin{equation*}
    \lim_{n,d\to\infty} \E_{\svtheta} L(\tvtheta) \geq \frac{1}{c_1}
\end{equation*}
almost surely.

For the task shift error $\E_{\svtheta, \rvx}\pts{\rvx^\top \hvtheta - \rvx^\top \tvtheta}^2$, we begin with a lemma.
\begin{lemma} \label{lem:benign_arcsin1}
Let $\rvu_1,\dots,\rvu_n,\rveps\in \R^d$ be independent standard Gaussian random vectors.
We have for any $1\leq i,k\leq n$,
\begin{equation*}
    \E_{\rveps}\mts{\sgn{\rvu_i^\top\rveps}\rvu_k^\top\rveps}=\sqrt{\frac{2}{\pi}}\frac{\rvu_i^\top\rvu_k}{\norm{\rvu_i}_2}.
\end{equation*}
\end{lemma}
The proof is in Appendix \ref{app:benign_arcsin}.
Now, have by \Eqref{eq:benign_yCy} that
\begin{align*}
    \E_{\svtheta, \rvx}\pts{\rvx^\top \hvtheta - \rvx^\top \tvtheta}^2 &= \E_{\svtheta} \mts{(\hrvy-\trvy)^\top \rmC (\hrvy-\trvy)} \\
    &= \E_{\svtheta} \mts{\trvy^\top \rmC \trvy} - \E_{\svtheta} \mts{\trvy^\top \rmC \hrvy} - \E_{\svtheta} \mts{\hrvy^\top \rmC \trvy} + \E_{\svtheta} \mts{\hrvy^\top \rmC \hrvy}.
\end{align*}
By the cyclic property of trace,
\begin{align*}
    \E_{\svtheta} \mts{\trvy^\top \rmC \trvy} &= \Tr\pts{\rmC \E_{\svtheta}\mts{ \trvy\trvy^\top}} \\
    &= \Tr\pts{\rmC \E_{\svtheta} \mts{\rmX \svtheta\svthetat \rmX^\top}} \\
    &= \Tr\pts{\rmC \pts{\frac{1}{d} \rvu\rvu^\top}},
\end{align*}
where $\rvu\sim\gN(\vzero, \mI)$.
Next,
\begin{equation*}
    \E_{\svtheta} \mts{\trvy^\top \rmC \hrvy} = \Tr\pts{\rmC \E_{\svtheta}\mts{ \hrvy\trvy^\top}} \eqqcolon \Tr\pts{\rmC \rmF},
\end{equation*}
where we define
\begin{align*}
    F_{ik} &\coloneqq \E_{\svtheta}\mts{\sgn{\rvx_i^\top\svtheta}\rvx_k^\top\svtheta} \\
    &=
    \begin{cases}
        \frac{1}{\sqrt{d}} \E_{\rveps}\mts{\left|\rvu_i^\top\rveps\right|} & i = k \\
        \frac{1}{\sqrt{d}} \E_{\rveps}\mts{\sgn{\rvu_i^\top\rveps}\rvu_k^\top\rveps} & i \neq k
    \end{cases} \\
    &=
    \begin{cases}
        \sqrt{\frac{2}{\pi d}}\norm{\rvu_i}_2 & i = k \\
        \sqrt{\frac{2}{\pi d}}\frac{\rvu_i^\top\rvu_k}{\norm{\rvu_i}_2} & i \neq k
    \end{cases}
\end{align*}
by Lemma \ref{lem:benign_arcsin1}.
Similarly,
\begin{equation*}
    \E_{\svtheta} \mts{\hrvy^\top \rmC \hrvy} = \Tr\pts{\rmC \E_{\svtheta}\mts{ \hrvy\hrvy^\top}} \eqqcolon \Tr\pts{\rmC \rmG},
\end{equation*}
where we define
\begin{align*}
    G_{ik} &\coloneqq \E_{\svtheta}\mts{\sgn{\rvx_i^\top\svtheta}\sgn{\rvx_k^\top\svtheta}} \\
    &=
    \begin{cases}
        1 & i = k \\
        \E_{\rveps}\mts{\sgn{\rvu_i^\top\rveps}\sgn{\rvu_k^\top\rveps}} & i \neq k
    \end{cases} \\
    &=
    \begin{cases}
        1 & i = k \\
        \frac{2}{\pi}\sin^{-1}\pts{\frac{\rvu_i^\top\rvu_k}{\norm{\rvu_i}_2\norm{\rvu_k}_2}} & i \neq k
    \end{cases}
\end{align*}
by Grothendieck's identity (Lemma 3.6.6 in~\cite{vershynin2018high}).
Putting everything together, we have
\begin{equation*}
    \E_{\svtheta, \rvx}\pts{\rvx^\top \hvtheta - \rvx^\top \tvtheta}^2 = \Tr\pts{\rmC \rmS}
\end{equation*}
where
\begin{equation*}
    S_{ik} =
    \begin{cases}
        \frac{1}{d} \norm{\rvu_i}_2^2 - 2\sqrt{\frac{2}{\pi d}}\norm{\rvu_i}_2 + 1 & i=k \\
        \frac{1}{d} \rvu_i^\top \rvu_k - \sqrt{\frac{2}{\pi d}}\frac{\rvu_i^\top\rvu_k}{\norm{\rvu_k}_2} - \sqrt{\frac{2}{\pi d}}\frac{\rvu_i^\top\rvu_k}{\norm{\rvu_k}_2} + \frac{2}{\pi}\sin^{-1}\pts{\frac{\rvu_i^\top\rvu_k}{\norm{\rvu_i}_2\norm{\rvu_k}_2}} & i \neq k.
    \end{cases}
\end{equation*}
By concentration of a standard Gaussian random vector and a union bound, for all $1\leq i \leq n$ we have
\begin{align*}
    d - \sqrt{d} &\leq \norm{\rvu_i}_2^2 \leq d + \sqrt{d}, \\
    \sqrt{d - \sqrt{d}} &\leq \norm{\rvu_i}_2 \leq \sqrt{d+\sqrt{d}}
\end{align*}
with probability at least $1-2ne^{-d}$.
By Bernstein's inequality and a union bound, the above holds simultaneously with
\begin{equation*}
    |\rvu_i^\top \rvu_k| \leq \sqrt{d + \sqrt{d}}
\end{equation*}
for all $1\leq i \neq k \leq n$, with probability at least $1-2n^3e^{-d}$.
Therefore, since  $\frac{2}{\pi}|\sin^{-1}(x)|\leq |x|$ for $-1\leq x \leq 1$, there exists a constant $c_2>1$ such that
\begin{equation*}
    \left|\sin^{-1}\pts{\frac{\rvu_i^\top\rvu_k}{\norm{\rvu_i}_2\norm{\rvu_k}_2}}\right| \leq \left|\frac{\rvu_i^\top\rvu_k}{\norm{\rvu_i}_2\norm{\rvu_k}_2}\right| \leq \frac{c_2}{\sqrt{d}}
\end{equation*}
over the same randomness as above.
Thus, there exists a constant $c_3>1$ such that for all $1\leq i,k\leq n$,
\begin{align}
    \frac{1}{c_3} \leq S_{ii} &\leq c_3 \label{eq:s_matrix1} \\
    | S_{ik} | &\leq \frac{c_3}{\sqrt{d}} \quad i \neq k \label{eq:s_matrix2}
\end{align}
with probability at least $1-c_3n^3e^{-d}$.
These high-probability bounds will be used to prove the following lemma, detailed in Appendix \ref{app:benign_arcsin}.
\begin{lemma} \label{lem:covering}
    There exists a constant $c>0$ such that $\mu_1(\rmS) \leq c$ with probability at least $1-cn^3e^{n-d}$.
\end{lemma}
By Lemma \ref{lem:covering}, since $\rmC$ is positive semi-definite, there exists a constant $c_4>0$ such that
\begin{equation*}
    \Tr(\rmC \rmS) \leq \Tr(\rmC)\mu_1(\rmS) \leq c_4\Tr(\rmC)
\end{equation*}
with probability at least $1-c_4 n^3e^{n-d}$.
By Lemma \ref{lem:benign_c}, there exists a constant $c_5\geq 1$ such that
\begin{equation*}
    \Tr(\rmC) \leq c_5 \pts{\frac{\sk}{n}+\frac{n}{R_{\sk}(\cov)}}
\end{equation*}
with probability at least $1-17e^{-\frac{n}{c_5}}$.
In the limit as $n,d\to\infty$, the terms $c_4 n^3e^{n-d}$ and $17e^{-\frac{n}{c_5}}$ are zero.
Therefore, by a union bound,
\begin{equation*}
    \lim_{n,d\to\infty} \E_{\svtheta, \rvx}\pts{\rvx^\top \hvtheta - \rvx^\top \tvtheta}^2 \leq c_4c_5 \lim_{n,d\to\infty} \pts{\frac{\sk}{n}+\frac{n}{R_{\sk}(\cov)}}
\end{equation*}
almost surely.
Choosing $c=\max(c_1,c_4c_5)$ completes the proof.
\end{proof}

\subsection{Concentration of task shift error terms} \label{app:benign_arcsin}
In this section, we provide the proofs of the technical Lemmas \ref{lem:benign_arcsin1}  and \ref{lem:covering} used in the proof of Theorem \ref{thm:benign_dense}.

\begin{proof} (Lemma \ref{lem:benign_arcsin1})

Note that $\rvu_i^\top \rveps \sim \gN(0, \norm{\rvu_i}_2^2)$ and $\rvu_k^\top \rveps \sim \gN(0, \norm{\rvu_k}_2^2)$.
Their correlation coefficient is
\begin{equation} \label{eq:rho}
    \rho_{ik} = \frac{\rvu_i^\top \rvu_k}{\norm{\rvu_i}_2\norm{\rvu_k}_2}.
\end{equation}
Let $Z_1\sim\gN(0,1)$ and $Z_2\sim\gN(0,1)$ be Gaussian variables independent of each other and $\rvu_i,\rvu_k$.
We may write
\begin{align*}
    \rvu_i^\top \rveps &\deq \norm{\rvu_i}_2 Z_1 \\
    \rvu_k^\top \rveps &\deq \rho_{ik} \norm{\rvu_k}_2 Z_1 + \sqrt{1-\rho_{ik}^2}\norm{\rvu_k}_2 Z_2.
\end{align*}
Since $Z_1$ and $Z_2$ are independent and centered,
\begin{align*}
    \E_{\rveps}\mts{\sgn{\rvu_i^\top \rveps}\rvu_k^\top\rveps} &= \E_{Z_1}\mts{\sgn{Z_1}\rho_{ik} \norm{\rvu_k}_2 Z_1} + \E_{Z_1,Z_2} \mts{\sgn{Z_1}\sqrt{1-\rho_{ik}^2}\norm{\rvu_k}_2 Z_2} \\
    &= \E_{Z_1}\mts{\sgn{Z_1}\rho_{ik} \norm{\rvu_k}_2 Z_1} \\
    &= \rho_{ik} \norm{\rvu_k}_2 \E_{Z_1}\mts{\sgn{Z_1} Z_1}.
\end{align*}
Since $Z_1$ is standard Gaussian, we have
\begin{equation*}
    \E_{Z_1}\mts{\sgn{Z_1} Z_1} = \E_{Z_1} |Z_1| = \sqrt{\frac{2}{\pi}}.
\end{equation*}
Using the value of $\rho_{ik}$ derived in \Eqref{eq:rho},
\begin{equation*}
    \E_{\rveps}\mts{\sgn{\rvu_i^\top \rveps}\rvu_k^\top\rveps} = \sqrt{\frac{2}{\pi}}\frac{\rvu_i^\top \rvu_k}{\norm{\rvu_i}_2}.
\end{equation*}
\end{proof}

\begin{proof} (Lemma \ref{lem:covering})

We will adapt the proof of Theorem 5.39 of~\cite{vershynin2012introduction}.
By the approximate isometry lemma (Lemma 5.36 in~\cite{vershynin2012introduction}), if for some $c>0$ we have
\begin{equation*}
    \norm{\frac{1}{n}\rmS^\top\rmS-\mI}\leq \max(c,c^2)
\end{equation*}
then $\mu_1(\rmS)\leq c$ as desired.
Let $\gN$ be a $\frac{1}{4}$-net of the unit sphere $S^{n-1}$ with respect to the Euclidean metric, and by Lemma 5.2 of~\cite{vershynin2012introduction} put $|\gN|\leq 9^n$.
Then, by Lemma 5.4 of~\cite{vershynin2012introduction},
\begin{equation*}
    \norm{\frac{1}{n}\rmS^\top\rmS-\mI} \leq 2\max_{\vx\in \gN}\left| \left\langle \pts{\frac{1}{n}\rmS^\top\rmS -\mI}\vx, \vx \right\rangle\right| = 2\max_{\vx\in\gN} \left| \frac{1}{n}\norm{\rmS \vx}_2^2 - 1 \right|.
\end{equation*}

Now consider a fixed vector $\vx\in S^{n-1}$.
We have
\begin{align*}
    \norm{\rmS \vx}_2^2 &= \sum_{i=1}^n \langle \rmS_i, \vx\rangle^2 \\
    &= \sum_{i=1}^n \pts{\sum_{k=1}^n S_{ik} x_k}^2 \\
    &\leq \sum_{i=1}^n \pts{\sum_{k=1}^n |S_{ik}| |x_k|}^2.
\end{align*}
By the concentration results in Equations \ref{eq:s_matrix1} and \ref{eq:s_matrix2}, there exists a constant $c_1>0$ such that
\begin{align*}
    | S_{ik} | &\leq c_1 \\
    | S_{ik} | &\leq \frac{c_1}{\sqrt{d}} \quad i \neq k
\end{align*}
with probability at least $1-c_1n^3e^{-d}$.
In particular,
\begin{equation*}
    \norm{\rmS \vx}_2^2 \leq \sum_{i=1}^n \pts{c_1 |x_i| + \sum_{i\neq k} \frac{c_1}{\sqrt{d}} |x_k|}^2
\end{equation*}
with probability at least $1-c_1n^3e^{-d}$.
Using $(a+b)^2\leq 2(a^2+b^2)$ for $a,b\in\R$ and the fact that $\vx\in S^{n-1}$,
\begin{align*}
    \norm{\rmS \vx}_2^2 &\leq 2c_1^2 \pts{\sum_{i=1}^n x_i^2 + \frac{1}{d}\pts{\sum_{i=1}^n |x_i|}^2} \\
    &\leq 2c_1^2 \pts{1 + \frac{n}{d}}
\end{align*}
with probability at least $1-c_1n^3e^{-d}$.
Therefore
\begin{equation*}
    \frac{1}{n}\norm{\rmS \vx}_2^2 \leq 2c_1^2\pts{\frac{1}{n}+\frac{1}{d}},
\end{equation*}
so for $c_3>\max(1,4c_1^2-1)$, we have
\begin{equation*}
    \left| \frac{1}{n}\norm{\rmS \vx}_2^2 - 1 \right| \leq c_3
\end{equation*}
with probability at least $1-c_1n^3e^{-d}$.

Taking a union bound over $\gN$, there exists a constant $c_4\geq 1$ such that
\begin{equation*}
    \mathbb{P}\pts{\max_{\vx\in \gN}\left| \frac{1}{n}\norm{\rmS \vx}_2^2 - 1 \right| \geq c } \leq 9^n \cdot c_1n^3e^{-d} \leq c_4 n^3e^{n-d}.
\end{equation*}
Therefore, for $c=\max(c_3^2,c_4^2)$, we have $\mu_1(\rmS)\leq c$ with probability at least $1-cn^3e^{n-d}$.

\end{proof}

\section{Support recovery: analysis beyond survival and contamination}\label{app:recover_support_proof}
In this section, we provide all of the proofs of our support recovery results from Section~\ref{sec:fewshot} under the diagonal covariance assumption.
In Section \ref{app:support_id1}, we introduce our key lemma of this section, then use it to prove Theorem \ref{thm:support_identification}, which characterizes the magnitudes of individual parameters arising from minimum-norm interpolation.
In Section \ref{app:support_id2}, we provide proofs for the support recovery of spiked and polynomial decay $\cov$ when $\svtheta$ is supported entirely in the top $\sk$ indices of the spectrum of $\cov$.
Finally, in Section \ref{app:support_id3}, we show that support recovery works even when $\svtheta$ is supported outside these top $\sk$ indices --- thereby handling cases wherein $\sk=0$, such as isotropic covariance --- under some additional conditions.

We remark that repeated application of the Sherman-Morrison-Woodbury identity induces a linear dependence on the sparsity parameter $t$ in Lemma \ref{lem:recover_support}.
This necessitates $t\ll n^{\frac{1}{2}}$ when the support lies entirely within the top $\sk$ indices of the covariance spectrum (Section \ref{app:support_id2}) and $t\ll n^{\frac{1}{4}}$ otherwise (Section \ref{app:support_id3}).
It is conceivable that these bounds could be improved with a finer analysis.
A relevant work is \cite{wu2023precise}, who study a multiclass classification setting where $\hrvy$ is a one-hot encoded vector and develop an improved Hanson-Wright inequality utilizing the sparsity in $\hrvy$ in the multiclass settings.
Unfortunately, we cannot directly apply their bound: even though our $\svtheta$ is sparse, our $\hry_i$ have Rademacher distribution and are generally not sparse.

\subsection{Characterization of classification MNI parameters}\label{app:support_id1}
Before we prove Theorem~\ref{thm:support_identification}, we prove the following lemma which lower bounds support indices of $\hvtheta$ and upper bounds non-support indices of $\hvtheta$. 
\begin{lemma}\label{lem:recover_support} 
    Define $\bgS \coloneqq \gS \cup [\sk]$ and denote by $\{\tlambda_j\}_{j=1}^{d - |\bgS|}$ the diagonal entries of the matrix $\cov_{-\bgS}$. Under Assumptions~\ref{asm:rank} and~\ref{asm:k_sparse}, for large enough $n$, we have
    \begin{align*}
        \left|\hevtheta_j\right| &\geq \sqrt{\frac{2\beta_j}{\lambda_j\pi}} \frac{\lambda_j\pts{\frac{n}{c_1\tlambda_{1}r_{0}\pts{\cov_{-\bgS}}}}}{1+\lambda_j\pts{\frac{c_2n}{\tlambda_{1}r_{0}\pts{\cov_{-\bgS}}}}}\quad\quad \text{ for } j \in \gS,\\
        \left|\hevtheta_j\right| &\leq t\sqrt{\frac{c_3}{\lambda_jn^{1-2\epsilon}}}\quad\quad\quad\quad\quad\quad\quad\quad \text{ for } j \in [\sk]\cap\gS^\mathsf{c},\\
        \left|\hevtheta_j\right| &\leq t\sqrt{\frac{c_4n^{1+2\epsilon}\lambda_j}{\pts{\tlambda_{1}r_{0}\pts{\cov_{-\bgS}}}^2}}\quad\quad\quad\quad \text{ for } \sk < j \leq d, j \in \gS^\mathsf{c},
    \end{align*}
    where $\beta_{j} \coloneqq \frac{\lambda_{j}\sevthetas_{j}}{\sum_{k\in\gS} \lambda_k\sevthetas_k} = \frac{a_{j}^2}{\sum_{k\in\gS} a_k^2}$ for $j\in\gS$, with probability at least $1 - ctde^{-n^{2\epsilon}}$.
\end{lemma}
\begin{proof} We start with the lower bound for the support $\left|\hevtheta_j\right|$ for $j \in \gS$; this term is related to the survival term $\SU_j$. According to Lemma~\ref{lem:k_sparse_SU_CN_bounds}, for any $j \in \gS$, we have
    \begin{align}\label{eq:htheta_lowerbound}
        \left|\hevtheta_j\right| = \SU_j \cdot \left|\sevtheta_j\right| = \SU_j \cdot \frac{\left|a_j\right|}{\sqrt{\lambda_j}} \geq \sqrt{\frac{2}{\pi\var}}\cdot\frac{\lambda_j\pts{\frac{n}{c_1\tlambda_{1}r_{0}\pts{\cov_{-\bgS}}}}}{1+\lambda_j\pts{\frac{c_2n}{\tlambda_{1}r_{0}\pts{\cov_{-\bgS}}}}} \cdot \frac{\left|a_j\right|}{\sqrt{\lambda_j}} = \sqrt{\frac{2\beta_j}{\lambda_j\pi}} \frac{\lambda_j\pts{\frac{n}{c_1\tlambda_{1}r_{0}\pts{\cov_{-\bgS}}}}}{1+\lambda_j\pts{\frac{c_2n}{\tlambda_{1}r_{0}\pts{\cov_{-\bgS}}}}}.
    \end{align}
    Next, we derive the upper bound of $\left|\hevtheta_j\right|$ for $j\in \gS^\mathsf{c}$. This term is related to $\CN$, but note that $\CN$ is the summation of all non-support dimensions. Here, we only need bounds for individual $\hevtheta_j$ for $j\in\gS^\mathsf{c}$. Since the covariance operator satisfies $r_{\sk}\pts{\cov}\geq bn$, we have two cases such that for $j \in [\sk]\cap\gS^\mathsf{c}$ and $\sk < j \leq d, j \in \gS^\mathsf{c}$ where they may have different covariance eigenvalue magnitude range. Based on the $\hevtheta_j$ definition in~\Eqref{eq:htheta_nonsupport2}, we follow the proof steps in Lemma~\ref{lem:k_sparse_SU_CN_bounds} and get
    \begin{align}
        \left|\hevtheta_j\right| &= \left|\sqrt{\lambda_j}\rvz_j^\top\rmA_{-s_1:s_t}^{-1}\crvy_{s_t}\right|\nonumber\\
        &= \frac{1}{\sqrt{\lambda_j}}\left|\lambda_j\rvz_j^\top\rmA_{-s_1:s_t}^{-1}\crvy_{s_t}\right|\nonumber\\
        &= \frac{1}{\sqrt{\lambda_j}}\sqrt{\crvy_{s_t}^\top \underbrace{\rmA_{-s_1:s_t}^{-1}\pts{\lambda_j^2\rvz_j\rvz_j^\top}\rmA_{-s_1:s_t}^{-1}}_{\coloneqq\tilde{\rmC}_{j}}\crvy_{s_t}}\nonumber\\
        &\leq \frac{1}{\sqrt{\lambda_j}} \pts{\sqrt{\hrvy^\top\tilde{\rmC}_{j}\hrvy} + \sum_{\ell=1}^t\sqrt{\pts{\frac{\lambda_{s_\ell}\rvz_{s_\ell}^\top\rmA_{-s_1:s_\ell}^{-1}\crvy_{s_{\ell-1}}}{1+\lambda_{s_\ell}\rvz_{s_\ell}^\top\rmA_{-s_1:s_\ell}^{-1}\rvz_{s_\ell}}}^2\rvz_{s_\ell}^\top\tilde{\rmC}_{j}\rvz_{s_\ell}}}\label{eq:ns_upper0},
    \end{align}
    where we apply the definition of $\crvy_{s_t}$ such that $\crvy_{s_\ell} \coloneqq \crvy_{s_{\ell-1}} - \frac{\lambda_{s_\ell}\rvz_{s_\ell}^\top\rmA_{-s_1:s_\ell}^{-1}\crvy_{s_{\ell-1}}}{1+\lambda_{s_\ell}\rvz_{s_\ell}^\top\rmA_{-s_1:{s_\ell}}^{-1}\rvz_{s_\ell}}\rvz_{s_\ell}$ and $\crvy_{s_0} = \hrvy$ for $\ell\in [t]$. For the inequality, we apply the triangle inequality such that $\sqrt{\pts{\rvx-\rvy}^\top\rmM\pts{\rvx-\rvy}} \leq \sqrt{\rvx^\top\rmM\rvx} + \sqrt{\rvy^\top\rmM\rvy}$ a total of $t$ times. Next, we discuss the upper bounds for $ j \in [\sk]\cap\gS^\mathsf{c}$ and $\sk < j \leq d,j\in\gS^\mathsf{c}$ respectively. For $ j \in [\sk]\cap\gS^\mathsf{c}$, by the Sherman-Morrison-Woodbury identity, we have $\hrvy^\top\tilde{\rmC}_{j}\hrvy$ as
    \begin{align*}
        \hrvy^\top\tilde{\rmC}_{j}\hrvy = \hrvy^\top \rmA_{-s_1:s_t}^{-1}\pts{\lambda_j^2\rvz_j\rvz_j^\top}\rmA_{-s_1:s_t}^{-1}\hrvy = \frac{\lambda_j^2\hrvy^\top \rmA_{-\pts{s_1:s_t}\cup\pts{j}}^{-1}\rvz_j\rvz_j^\top\rmA_{-\pts{s_1:s_t}\cup\pts{j}}^{-1}\hrvy}{\pts{1+\lambda_j\rvz_j^\top\rmA_{-\pts{s_1:s_t}\cup\pts{j}}^{-1}\rvz_j}^2} \leq \frac{\pts{\rvz_j^\top\rmA_{-\pts{s_1:s_t}\cup\pts{j}}^{-1}\hrvy}^2}{\pts{\rvz_j^\top\rmA_{-\pts{s_1:s_t}\cup\pts{j}}^{-1}\rvz_j}^2}.
    \end{align*}
    Next, for $ j \in [\sk]\cap\gS^\mathsf{c}$, by the Hanson-Wright inequality (Lemma~\ref{lem:hanson_wright_ineq}), we have the upper bound for $\sqrt{\hrvy^\top\tilde{\rmC}_{j}\hrvy}$ as
    \begin{align}
        \sqrt{\hrvy^\top\tilde{\rmC}_{j}\hrvy} &\leq \sqrt{\frac{c_1^2\norm{\rmA_{-\pts{s_1:s_t}\cup\pts{j}}^{-1}}_2^2\cdot n^{1+2\epsilon}}{\pts{\Tr\pts{\rmA_{-\pts{s_1:s_t}\cup\pts{j}}^{-1}}-c_1\norm{\rmA_{-\pts{s_1:s_t}\cup\pts{j}}^{-1}}_2\cdot n^{\frac{1}{2}+\epsilon}}^2}}\nonumber\\
        &\leq \sqrt{\frac{c_1^2\norm{\rmA_{-\pts{s_1:s_t}\cup[\sk]}^{-1}}_2^2\cdot n^{1+2\epsilon}}{\pts{\Tr\pts{\rmA_{-\pts{s_1:s_t}\cup\pts{j}}^{-1}}-c_1\norm{\rmA_{-\pts{s_1:s_t}\cup[\sk]}^{-1}}_2\cdot n^{\frac{1}{2}+\epsilon}}^2}}\label{eq:ns_upper1},
    \end{align}
    where in the second inequality, we use the fact that $\rmA_{-\pts{s_1:s_t}\cup[\sk]}^{-1} \succeq \rmA_{-\pts{s_1:s_t}\cup\pts{j}}^{-1}$ for all $j\in[\sk]$. Next, we use Lemma~\ref{lem:trace_lower_bound} to get the lower bound for the trace term in the denominator, and get
    \begin{align}
        \pts{\ref{eq:ns_upper1}}&\leq  \sqrt{\frac{c_1^2\norm{\rmA_{-\pts{s_1:s_t}\cup[\sk]}^{-1}}_2^2\cdot n^{1+2\epsilon}}{\pts{\pts{1-\frac{c}{n}}^{\sk-1}\Tr\pts{\rmA_{-\pts{s_1:s_t}\cup[\sk]}^{-1}}-c_1\norm{\rmA_{-\pts{s_1:s_t}\cup[\sk]}^{-1}}_2\cdot n^{\frac{1}{2}+\epsilon}}^2}}\nonumber\\
        &= \sqrt{\frac{1}{\pts{\pts{1-\frac{c}{n}}^{\sk-1}\frac{\Tr\pts{\rmA_{-\pts{s_1:s_t}\cup[\sk]}^{-1}}}{c_1\norm{\rmA_{-\pts{s_1:s_t}\cup[\sk]}^{-1}}_2\cdot n^{\frac{1}{2}+\epsilon}}-1}^2}}\label{eq:ns_upper2}
    \end{align}
    Finally, we apply Lemma~\ref{lem:eig_constant} to show that eigenvalues in $\rmA_{-\pts{s_1:s_t}\cup[\sk]}^{-1}$ are identical up to a constant.
    \begin{align*}
        \pts{\ref{eq:ns_upper2}}&\leq \sqrt{\frac{1}{\pts{\pts{1-\frac{c}{n}}^{\sk-1}\frac{n^{\frac{1}{2}-\epsilon}}{c_2}-1}^2}} \leq \sqrt{\frac{c_3}{n^{1-2\epsilon}}},
    \end{align*}
    Since $\rvz_{s_\ell}$ is independent to $\tilde{\rmC}_{j}$ for $j\in [\sk]\cap\gS^\mathsf{c}$ and $\ell\in[\sk]$, following the same procedure, we can show $\sqrt{\rvz_{s_\ell}^\top\tilde{\rmC}_{j}\rvz_{s_\ell}}$ is upper bounded in the same rate. As a result, from~\Eqref{eq:ns_upper0} we have
    \begin{align*}
        \left|\hevtheta_j\right| &\leq  \frac{1}{\sqrt{\lambda_j}}\pts{1+\sum_{\ell=1}^t\left|\frac{\lambda_{s_\ell}\rvz_{s_\ell}^\top\rmA_{-s_1:s_\ell}^{-1}\crvy_{s_{\ell-1}}}{1+\lambda_{s_\ell}\rvz_{s_\ell}^\top\rmA_{-s_1:s_\ell}^{-1}\rvz_{s_\ell}}\right|}\sqrt{\frac{c_3}{n^{1-2\epsilon}}}\\
        &\leq \frac{1}{\sqrt{\lambda_j}}\pts{1+tc}\sqrt{\frac{c_3}{n^{1-2\epsilon}}},
    \end{align*}
    where we apply Lemma~\ref{lem:cn_const} in the last inequality. Next, starting from~\Eqref{eq:ns_upper0} again, we show the upper bound of $\left|\hevtheta_j\right|$ for $\sk < j \leq d,j\in\gS^\mathsf{c}$. By the Sherman-Morrison-Woodbury identity, we have $\hrvy^\top\tilde{\rmC}_{j}\hrvy$ as
    \begin{align}
        \sqrt{\hrvy^\top\tilde{\rmC}_{j}\hrvy} &= \sqrt{\hrvy^\top \rmA_{-s_1:s_t}^{-1}\pts{\lambda_j^2\rvz_j\rvz_j^\top}\rmA_{-s_1:s_t}^{-1}\hrvy}\nonumber\\
        &= \sqrt{\frac{\lambda_j^2\hrvy^\top \rmA_{-\pts{s_1:s_t}\cup\pts{j}}^{-1}\rvz_j\rvz_j^\top\rmA_{-\pts{s_1:s_t}\cup\pts{j}}^{-1}\hrvy}{\pts{1+\lambda_j\rvz_j^\top\rmA_{-\pts{s_1:s_t}\cup\pts{j}}^{-1}\rvz_j}^2}}\nonumber \\
        &\leq \sqrt{\lambda_j^2\hrvy^\top \rmA_{-\pts{s_1:s_t}\cup\pts{j}}^{-1}\rvz_j\rvz_j^\top\rmA_{-\pts{s_1:s_t}\cup\pts{j}}^{-1}\hrvy}\label{eq:balance_support}.
    \end{align}
    Next, we apply the Hanson-Wright inequality (Lemma~\ref{lem:hanson_wright_ineq}) to obtain
    \begin{align}
        \pts{\ref{eq:balance_support}} &\leq \sqrt{\lambda_j^2\pts{2c_1\norm{\rmA_{-\pts{s_1:s_t}\cup\pts{j}}^{-1}}_2\cdot n^{\frac{1}{2}+\epsilon}}^2}\nonumber\\
        &\leq \sqrt{\lambda_j^2\pts{2c_1\norm{\rmA_{-\pts{s_1:s_t}\cup\pts{j}\cup[\sk]}^{-1}}_2\cdot n^{\frac{1}{2}+\epsilon}}^2}\nonumber\\
        &=\sqrt{\lambda_j^2\pts{\frac{2c_1\cdot n^{\frac{1}{2}+\epsilon}}{\mu_n\pts{\rmA_{-\pts{s_1:s_t}\cup\pts{j}\cup[\sk]}}}}^2}\label{eq:ns_upper3},
    \end{align}
    where in the second inequality, we use the property that $\rmA_{-\pts{s_1:s_t}\cup\pts{j}}^{-1} \preceq \rmA_{-\pts{s_1:s_t}\cup\pts{j}\cup[\sk]}^{-1}$ for $\sk < j \leq d,j\in\gS^\mathsf{c}$. Finally, we apply Lemma~\ref{lem:balance_extension} to show that $\cov_{-\pts{s_1:s_t}\cup\pts{j}\cup[\sk]}$ satisfies $r_0\pts{\cov_{-\pts{s_1:s_t}\cup\pts{j}\cup[\sk]}} \geq bn$. Therefore, by Lemma 10 in~\cite{bartlett2020benign}, we have 
    \begin{align*}
        \mu_n\pts{\rmA_{-\pts{s_1:s_t}\cup\pts{j}\cup[\sk]}} \geq \frac{1}{c}\pts{\sum_{k=\sk+1, k\in\gS^\mathsf{c}, k\neq j}\lambda_k}.
    \end{align*}
    Applying the inequality to~\Eqref{eq:ns_upper3}, we get
    \begin{align}
        \pts{\ref{eq:ns_upper3}}&\leq \sqrt{\lambda_j^2\pts{\frac{c_2\cdot n^{1+2\epsilon}}{\pts{\sum_{k=\sk+1, k\in\gS^\mathsf{c}, k\neq j}\lambda_k}^2}}}.\label{eq:non-support_eq}
    \end{align}
    Next, we want to upper bound $\frac{1}{\sum_{k=\sk+1, k\in\gS^\mathsf{c}, k\neq j}\lambda_k}$ by $\frac{c}{\sum_{k=\sk+1, k\in\gS^\mathsf{c}}\lambda_k}$, and we have
    \begin{align}
        \frac{1}{\sum_{k=\sk+1, k\in\gS^\mathsf{c}, k\neq j}\lambda_k} &= \pts{\frac{1}{\sum_{k=\sk+1, k\in\gS^\mathsf{c}}\lambda_k}}\pts{\frac{\sum_{k=\sk+1, k\in\gS^\mathsf{c}}\lambda_k}{\sum_{k=\sk+1, k\in\gS^\mathsf{c}, k\neq j}\lambda_k}}\nonumber\\
        & =\pts{\frac{1}{\sum_{k=\sk+1, k\in\gS^\mathsf{c}}\lambda_k}}\pts{1 + \frac{\lambda_j}{\sum_{k=\sk+1, k\in\gS^\mathsf{c}, k\neq j}\lambda_k}}\label{eq:ns_upper4}
    \end{align}
    According to the effective rank of $\cov$, for $\sk < j \leq d,j\in\gS^\mathsf{c}$, we have $\sum_{k=\sk+1, k\in\gS^\mathsf{c}}\lambda_k \geq b\lambda_{\sk+1}n \geq b\lambda_{j}n$. By deducting $\lambda_j$ on both side, we get
    \begin{align*}
        \sum_{k=\sk+1, k\in\gS^\mathsf{c}, k\neq j}\lambda_k \geq b\lambda_{j}n - \lambda_j.
    \end{align*}
    Therefore, we can write
    \begin{align*}
        \pts{\ref{eq:ns_upper4}} & \leq \pts{\frac{1}{\sum_{k=\sk+1, k\in\gS^\mathsf{c}}\lambda_k}}\pts{1 + \frac{\lambda_j}{b\lambda_jn - \lambda_j}} = \pts{\frac{1}{\sum_{k=\sk+1, k\in\gS^\mathsf{c}}\lambda_k}}\pts{1 + \frac{1}{bn-1}}.
    \end{align*}
    Substituting this inequality into~\Eqref{eq:non-support_eq}, we finish the upper bound for $\sqrt{\hrvy^\top\tilde{\rmC}_{j}\hrvy}$.
    Follow this procedure, we can show $\sqrt{\rvz_{s_\ell}^\top\tilde{\rmC}_{j}\rvz_{s_\ell}}$ is upper bounded by the same term. As a result, we have for $\sk < j \leq d, j\in\gS^\mathsf{c}$,
    \begin{align*}
        \left|\hevtheta_j\right| &\leq  \frac{1}{\sqrt{\lambda_j}}\pts{1+\sum_{\ell=1}^t\left|\frac{\lambda_{s_\ell}\rvz_{s_\ell}^\top\rmA_{-s_1:s_\ell}^{-1}\crvy_{s_{\ell-1}}}{1+\lambda_{s_\ell}\rvz_{s_\ell}^\top\rmA_{-s_1:s_\ell}^{-1}\rvz_{s_\ell}}\right|}\sqrt{\lambda_j^2\frac{c_2n^{1+2\epsilon}}{\pts{\sum_{k=\sk+1, k\in\gS^\mathsf{c}}\lambda_k}^2}}\\
        &\leq \pts{1+tc}\sqrt{\frac{c_2n^{1+2\epsilon}}{\pts{\sum_{k=\sk+1, k\in\gS^\mathsf{c}}\lambda_k}^2}},
    \end{align*}
    where we apply Lemma~\ref{lem:cn_const} in the last inequality. This completes the proof of the lemma.
\end{proof}
Equipped with Lemma~\ref{lem:recover_support}, we can now prove Theorem~\ref{thm:support_identification}. While it restricts to the case where $\svtheta$ is supported only in the top $\sk$ indices of the covariance spectrum, we handle support recovery outside the top $\sk$ indices in Section \ref{app:support_id3}.
\begin{proof}(Theorem~\ref{thm:support_identification})\\
    We show that as long as one of the following conditions hold, the support recovery is guaranteed: (1) $t$ is known and 
    \begin{align*}
        \lambda_j \ll \frac{\lambda_q \cdot n^{1-2\epsilon}}{t^2} \text{ and } \lambda_j \ll \frac{\pts{\tlambda_1 r_0\pts{\cov_{-[\sk]}}}^2}{t^2\cdot\pts{n^{1+2\epsilon}}\cdot\lambda_\ell} 
    \end{align*}
    hold for all $j\in\gS$, $q \in \gS^\mathsf{c} \cap [\sk]$, $\ell>\sk$ or (2) $\cov$ is known. To start with, we first show the support recovery under condition (1). By assumption that $\svtheta$ is supported only in the top $\sk$ indices of the covariance spectrum, we have $j \in [\sk]$ for $j\in\gS$. According to Lemma~\ref{lem:recover_support}, we have the lower bound of $\hevtheta$ as
    \begin{align}
        \left|\hevtheta_j\right| &\geq \sqrt{\frac{2\beta_j}{\lambda_j\pi}} \frac{\lambda_j\pts{\frac{n}{c_1\tlambda_{1}r_{0}\pts{\cov_{-\bgS}}}}}{1+\lambda_j\pts{\frac{c_2n}{\tlambda_{1}r_{0}\pts{\cov_{-\bgS}}}}} = \sqrt{\frac{2\beta_j}{\lambda_j\pi}} \frac{1}{\frac{c_1\tlambda_1r_0\pts{\cov_{-\bgS}}}{\lambda_jn}+c_1c_2} \geq \sqrt{\frac{2\beta_j}{\lambda_j\pi}} \frac{1}{\frac{c_1\tlambda_1r_0\pts{\cov_{-\bgS}}}{\tlambda_1r_0\pts{\cov_{-\bgS}}}+c_1c_2} = \frac{c_3}{\sqrt{\lambda_j}}\label{eq:t_s_lower},
    \end{align}
    where we apply $\tlambda_1 r_0\pts{\cov_{-\bgS}} = \sum_{k=\sk+1,k\in\gS^\mathsf{c}}^d \lambda_k \leq \lambda_{\sk}n \leq \lambda_{j}n$ for $j \in [\sk]$ by the definition of $\sk$. As a result, support is lower bounded in the rate of $\frac{1}{\sqrt{\lambda_j}}$. Next, for non-support indices $j \in \gS^\mathsf{c} \cap [\sk]$, by Lemma~\ref{lem:recover_support}, we have the upper bound by
    \begin{align}
        \left|\hevtheta_j\right| &\leq t\sqrt{\frac{c}{\lambda_jn^{1-2\epsilon}}}.\label{eq:t_ns_upper1}
    \end{align}
    Therefore, for $j \in \gS^\mathsf{c} \cap [\sk]$, non-support is upper bounded in the rate of $\frac{1}{\sqrt{\lambda_j}}\cdot\pts{\frac{t}{\sqrt{n^{1-2\epsilon}}}}$. Lastly, for non-support indices $j > \sk$, by Lemma~\ref{lem:recover_support}, we have the upper bound by
    \begin{align}
        \left|\hevtheta_j\right| &\leq  t\sqrt{\frac{cn^{1+2\epsilon}\lambda_j}{\pts{\tlambda_{1}r_{0}\pts{\cov_{-\bgS}}}^2}}.\label{eq:t_ns_upper2}
    \end{align}
    Therefore, for $j > \sk$, non-support is upper bounded in the rate of $\frac{1}{\sqrt{\lambda_j}}\cdot\sqrt{\frac{t^2n^{1+2\epsilon}\lambda_j^2}{\pts{\tlambda_{1}r_{0}\pts{\cov_{-\bgS}}}^2}}$.
    As a result, we need the following conditions to guarantee the lower bound of the support indices is larger than the upper bound of the non-support indices:
    \begin{align*}
        \frac{1}{\sqrt{\lambda_j}} &\gg \frac{t}{\sqrt{\lambda_k\cdot n^{1-2\epsilon}}} \Leftrightarrow \frac{\lambda_k\cdot n^{1-2\epsilon}}{t^2} \gg \lambda_j\\
        \frac{1}{\sqrt{\lambda_j}} &\gg t\sqrt{\frac{n^{1+2\epsilon}\lambda_\ell}{\pts{\tlambda_{1}r_{0}\pts{\cov_{-\bgS}}}^2}} \Leftrightarrow \frac{\pts{\tlambda_1 r_0\pts{\cov_{-\bgS}}}^2}{t^2\cdot \pts{n^{1+2\epsilon}}\cdot\lambda_\ell} \gg \lambda_j
    \end{align*}
    for all $j\in\gS$ and $k\in\gS^{C}\cap[\sk]$ and $\ell>\sk$. In this way, since we also know $t$, we can achieve support recovery by choosing the largest $t$ indices of $\hvtheta$.
    
    On the other hand, for condition (2), if we have access to all the eigenvalues $\{\lambda_j\}_{j=1}^d$ of $\cov$, we can determine whether $\hevtheta_j$ is a support by checking if $\left|\hevtheta_j\right| = \frac{c}{\sqrt{\lambda_j}}$ for some constant $c>0$ according to the support lower bound in~\Eqref{eq:t_s_lower}. For non-support rate, we can further upper bound~\Eqref{eq:t_ns_upper2} by $t\sqrt{\frac{c}{\lambda_jn^{1-2\epsilon}}}$ since
    \begin{align}
        t\sqrt{\frac{cn^{1+2\epsilon}\lambda_j}{\pts{\tlambda_{1}r_{0}\pts{\cov_{-\bgS}}}^2}} \leq t\sqrt{\frac{cn^{1+2\epsilon}\lambda_j}{\pts{b\lambda_jn}^2}}=\frac{c_1t}{\sqrt{\lambda_jn^{1-2\epsilon}}}\label{eq:t_ns_upper3},
    \end{align}
    where we apply $\tlambda_1 r_0\pts{\cov_{-\bgS}} = \sum_{k=\sk+1,k\in\gS^\mathsf{c}}^d \lambda_k \geq b\lambda_{j}n$ for $j > \sk$.
    As a result, according to the non-support upper bound in~\Eqref{eq:t_ns_upper1} and ~\Eqref{eq:t_ns_upper3}, $\hevtheta_j$ can be classified as non-support if $\left|\hevtheta_j\right| = o\pts{\frac{t}{\sqrt{\lambda_j\cdot n^{1-2\epsilon}}}}$, which decays at a rate proportional to $\frac{t}{\sqrt{n^{1-2\epsilon}}}$. This completes the proof of the theorem.
\end{proof}
\subsection{Spiked and polynomial decay covariance: support inside the top $\sk$ indices} \label{app:support_id2}
In this section, we characterize different covariance matrices wherein $\svtheta$ is supported only in the top $\sk$ indices of the covariance spectrum and we demonstrate the support identification guarantee (Theorem~\ref{thm:support_identification}) is satisfied. Corollary~\ref{cor:spike_support_identification} demonstrates the characterization in spiked covariance defined in Definition~\ref{def:spiked}, and Corollary~\ref{cor:poly_support_identification} and Corollary~\ref{cor:poly_support_identification_v0} show the characterization in polynomial decay covariance defined in Definition~\ref{def:polynomial}. Note that we do not allow the case $\sk=0$ (\eg spiked covariance with $q>1-r$) in this section. The analysis closely follows Lemmas 32 and 34 of~\cite{muthukumar2021classification}.
\begin{corollary}\label{cor:spike_support_identification}
    Under Assumptions~\ref{asm:rank} and~\ref{asm:k_sparse} with $t\ll n^{\frac{1}{2} -\epsilon}$ and spiked covariance matrix (Definition~\ref{def:spiked}), we assume support are all in the top $\sk$ indices of the covariance spectrum. By substituting spiked parameters in Definition~\ref{def:spiked},  for $q < \pts{1-r}$, we have 
    \begin{align*}
        \left|\hevtheta_j\right| \condd{\gg n^{\frac{-p+q+r}{2}}}{\text{for } j \in \gS}{\ll n^{\frac{-p+q+r}{2}}}{\text{for } j \in [\sk]\cap \gS^\mathsf{c}}{\ll n^{\frac{-2p+2}{2}}}{\text{for } \sk < j \leq d}.
    \end{align*}
    Therefore, if $p > 2 -q -r$, we can pick a threshold between bounds to distinguish support and non-support such as $T=n^{\frac{-p+q+r}{2}}$.
\end{corollary}
\begin{proof} Recall the spiked covariance in Definition~\ref{def:spiked} such that
    \begin{equation*}
        \lambda_j \coloneqq
        \begin{cases}
            \frac{ad}{s} & j\in[s] \\
            \frac{(1-a)d}{d-s} & \tn{otherwise.}
        \end{cases}
    \end{equation*}
    First, we need to make sure the conditions in Theorem~\ref{thm:support_identification} scenario (1) are satisfied such that
    \begin{align*}
        \lambda_j &\ll \frac{\lambda_q \cdot n^{1-2\epsilon}}{t^2},\\
        \lambda_j &\ll \frac{\pts{\tlambda_1 r_0\pts{\cov_{-[\sk]}}}^2}{t^2\cdot\pts{n^{1+2\epsilon}}\cdot\lambda_\ell},
    \end{align*}
    for all $j\in\gS$, $q \in \gS^\mathsf{c} \cap [\sk]$, and $\ell>\sk$.
    For the first condition, we show it holds because
    \begin{align*}
         \frac{\lambda_q \cdot n^{1-2\epsilon}}{t^2} = \frac{\frac{ad}{s} \cdot n^{1-2\epsilon}}{t^2} \gg \frac{ad}{s} = \lambda_j,
    \end{align*}
    since $\lambda_j = \lambda_q$ for $j\in\gS$, $q \in \gS^\mathsf{c} \cap [\sk]$, and $t\ll n^{\frac{1}{2} -\epsilon}$. For the second condition, we can show the right-hand side as
    \begin{align*}
         \frac{\pts{\tlambda_1 r_0\pts{\cov_{-[\sk]}}}^2}{t^2\cdot\pts{n^{1+2\epsilon}}\cdot\lambda_\ell} = \frac{\pts{\pts{d-s}\cdot\frac{\pts{1-a}d}{\pts{d-s}}}^2}{t^2\cdot\pts{n^{1+2\epsilon}}\cdot\frac{\pts{1-a}d}{\pts{d-s}}} = \frac{\pts{1-a}d\cdot\pts{d-s}\cdot n^{1-2\epsilon}}{t^2\cdot n^2} > \frac{1}{c}n^{2p-2}
    \end{align*}
    for all $j\in\gS$, $q \in \gS^\mathsf{c} \cap [\sk]$, and $\ell>\sk$. On the left-hand side, we have
    \begin{align*}
        \lambda_j = \frac{ad}{s} = n^{p-q-r}.
    \end{align*}
    Since the corollary assumes $p > 2 -q -r$, the second condition holds. Next, we show the precise threshold value by using Lemma~\ref{lem:recover_support}. We have the lower bound for support indices for $j \in\gS$ as
    \begin{align*}
        \left|\hevtheta_j\right| &\geq \sqrt{\frac{2\beta_j}{\lambda_j\pi}} \frac{\lambda_j\pts{\frac{n}{c_1\tlambda_{1}r_{0}\pts{\cov_{-\bgS}}}}}{1+\lambda_j\pts{\frac{c_2n}{\tlambda_{1}r_{0}\pts{\cov_{-\bgS}}}}}\\
        &= \sqrt{\frac{2\beta_j}{\pi}}\sqrt{\frac{s}{ad}}\pts{\frac{\frac{ad}{s}\pts{\frac{n}{c_1\pts{d-|\bgS|}\frac{\pts{1-a}d}{\pts{d-s}}}}}{1+\frac{ad}{s}\pts{\frac{c_2n}{\pts{d-|\bgS|}\frac{\pts{1-a}d}{\pts{d-s}}}}}}\\
        &= \sqrt{\frac{2\beta_j}{\pi}}n^{\frac{-p+q+r}{2}}\frac{n^{p-q-r}\pts{\frac{n}{c_1\pts{n^p-|\bgS|}\frac{\pts{1-n^{-q}}n^p}{\pts{n^p-n^r}}}}}{1+n^{p-q-r}\pts{\frac{c_2n}{\pts{n^p-|\bgS|}\frac{\pts{1-n^{-q}}n^p}{\pts{n^p-n^r}}}}}\\
        &\geq \sqrt{\frac{2\beta_j}{\pi}}n^{\frac{-p+q+r}{2}}\cdot\frac{\frac{1}{c_3}\cdot n ^{\pts{1-r}-q}}{1+c_4n^{\pts{1-r}-q}},
    \end{align*}
    where we substitute $\lambda_j = \frac{ad}{s} = n^{p-q-r}$ and $\tlambda_1 r_0\pts{\cov_{-\bgS}}= \pts{d-|\bgS|}\frac{\pts{1-a}d}{d-s}$. When $q < (1-r)$, the $n^{\pts{1-r}-q}$ term dominates in the fraction part. Hence, we have
    \begin{align*}
        \left|\hevtheta_j\right| \geq \sqrt{\frac{2\beta_j}{\pi}}n^{\frac{-p+q+r}{2}}\cdot c_{10}\pts{1 + c n^{q-\pts{1-r}}}^{-1}.
    \end{align*}
    For non-support upper bound, we need to consider indices in $j \in[\sk]\cap\gS^\mathsf{c}$ and indices in $j > \sk$. According to Lemma~\ref{lem:recover_support}, for $ j \in[\sk]\cap\gS^\mathsf{c}$, we have $\lambda_j = \frac{ad}{s} = n^{p-q-r}$ and we get
    \begin{align*}
        \left|\hevtheta_j\right| &\leq t\sqrt{\frac{c}{\lambda_jn^{1-2\epsilon}}} \leq c_1 tn^{\frac{-p+q+r-1+2\epsilon}{2}} < c_1 n^{\frac{-p+q+r}{2}},
    \end{align*}
    since $t \ll n^{\frac{1}{2} -\epsilon}$. On the other hand, for $j > \sk$, according to Lemma~\ref{lem:recover_support}, we substitute $\lambda_j = \frac{\pts{1-a}d}{d-s} = \frac{\pts{1-n^{-q}}\cdot n^p}{n^p - n^r}$ and $\tlambda_1 r_0\pts{\cov_{-\bgS}}= \pts{d-|\bgS|}\frac{\pts{1-a}d}{d-s}$ and we have
    \begin{align*}
        \left|\hevtheta_j\right| &\leq t\sqrt{\frac{c}{\lambda_j}\frac{n^{1+2\epsilon}\lambda_j^2}{\pts{\tlambda_1 r_0\pts{\cov_{-\bgS}}}^2}}\\
        &= t\sqrt{\frac{c\pts{d-s}}{\pts{1-a}d}\pts{\frac{n^{1+2\epsilon}\pts{\frac{\pts{1-a}d}{d-s}}^2}{\pts{\pts{d-|\bgS|}\frac{\pts{1-a}d}{d-s}}^2}}}\\
        &= t\sqrt{\frac{c\pts{n^p - n^r}}{\pts{1-n^{-q}}\cdot n^p}\pts{\frac{n^{1+2\epsilon}}{\pts{n^p -|\bgS|}^2}}}\\
        &\leq c_1 t n^{\frac{-2p+1+2\epsilon}{2}}\\
        &< c_1 n^{\frac{-2p+2}{2}},
    \end{align*}
    since $t \ll n^{\frac{1}{2} -\epsilon}$. As a result, since we assume $p > 2 -q -r$, we have support lower bound larger than the non-support upper bound, and the support identification is guaranteed. 
\end{proof}

\begin{corollary}\label{cor:poly_support_identification}
    Under Assumptions~\ref{asm:rank} and~\ref{asm:k_sparse} with $t\ll n^{\frac{1}{2} -\epsilon}$ and polynomial decay covariance in Definition~\ref{def:polynomial} with $u=1, v=2$, we assume support are all in the top $\sk$ indices of the covariance spectrum. By substituting $\lambda_j = \frac{1}{j\cdot \ln^2\pts{j+1}}$, we have
    \begin{align*}
        \left|\hevtheta_j\right| \condd{\gg \frac{1}{\sqrt{\lambda_j}}}{\text{for } j \in \gS}{\ll \frac{t}{\sqrt{\lambda_j\cdot n^{1-2\epsilon}}}}{\text{for } j \in [\sk]\cap\gS^\mathsf{c}}{\ll \frac{t}{\sqrt{\lambda_j\cdot n^{1-2\epsilon}} }}{\text{for } \sk < j \leq d}.
    \end{align*}
    If we get access to all $\{\lambda_j\}_{j=1}^d$, we can therefore distinguish support and non-support by examining each $\hevtheta_j$ has order $\omega\pts{\frac{1}{\sqrt{\lambda_j}}}$ or decay in a rate of $o\pts{\frac{t}{\sqrt{\lambda_j\cdot n^{1-2\epsilon}}}}$.
\end{corollary}
\begin{proof} Since the conditions in the first scenario in Theorem~\ref{thm:support_identification} are not satisfied, we need to assume $\cov$ is known. Next, we show the rate of supports and non-supports based on Lemma~\ref{lem:recover_support}. We have
    \begin{align*}
        \left|\hevtheta_j\right| &\geq \sqrt{\frac{2\beta_j}{\lambda_j\pi}} \frac{\lambda_j\pts{\frac{n}{c_1\tlambda_{1}r_{0}\pts{\cov_{-\bgS}}}}}{1+\lambda_j\pts{\frac{c_2n}{\tlambda_{1}r_{0}\pts{\cov_{-\bgS}}}}},\quad\quad \text{ for } j \in \gS,\\
        \left|\hevtheta_j\right| &\leq t\sqrt{\frac{c}{\lambda_jn^{1-2\epsilon}}},\quad\quad\quad\quad\quad\quad\quad\quad\quad \text{ for } j \in [\sk]\cap\gS^\mathsf{c},\\
        \left|\hevtheta_j\right| &\leq t\sqrt{\frac{n^{1+2\epsilon}\lambda_j}{\pts{\tlambda_{1}r_{0}\pts{\cov_{-\bgS}}}^2}},\quad\quad\quad\quad \text{ for } \sk < j \leq d, j \in \gS^\mathsf{c}.
    \end{align*}
    For polynomial decay covariance, we have $\lambda_j = \frac{1}{j\cdot \ln^2\pts{j+1}}$ for $j \in [d]$, and also $\sum_{j=1}^{\infty} \lambda_j = \sum_{j=1}^{\infty} \frac{1}{j\cdot \ln^2\pts{j+1}} = \mathcal{O}\pts{1}$. Therefore, for $j\in\gS$, we have
    \begin{align*}
        \left|\hevtheta_j\right| \geq \sqrt{\frac{2\beta_j}{\lambda_j\pi}} \frac{\lambda_j\pts{\frac{n}{c_1\tlambda_{1}r_{0}\pts{\cov_{-\bgS}}}}}{1+\lambda_j\pts{\frac{c_2n}{\tlambda_{1}r_{0}\pts{\cov_{-\bgS}}}}} = \sqrt{\frac{2\beta_j}{\lambda_j\pi}} \frac{1}{\frac{c_1\tlambda_1r_0\pts{\cov_{-\bgS}}}{\lambda_jn}+c_1c_2} \geq \sqrt{\frac{2\beta_j}{\lambda_j\pi}} \frac{1}{\frac{c_1\tlambda_1r_0\pts{\cov_{-\bgS}}}{\tlambda_1r_0\pts{\cov_{-\bgS}}}+c_1c_2} = \frac{c_3}{\sqrt{\lambda_j}}
    \end{align*}
    where we apply $\tlambda_1 r_0\pts{\cov_{-\bgS}} = \sum_{j=\sk+1}^d \lambda_j \leq \lambda_{\sk}n \leq \lambda_{j}n$ for $j \in [\sk]$ by the definition of $\sk$.
    For non-support, $j\in [\sk]\cap\gS^\mathsf{c}$, we have
    \begin{align*}
        \left|\hevtheta_j\right| \leq t\sqrt{\frac{c}{\lambda_jn^{1-2\epsilon}}}.
    \end{align*}
    For $j > \sk, j\in \gS^\mathsf{c}$, we have
    \begin{align*}
        \left|\hevtheta_j\right| &\leq t\sqrt{\frac{c}{\lambda_j}\frac{n^{1+2\epsilon}\lambda_j^2}{\pts{\tlambda_{1}r_{0}(\cov_{-\bgS})}^2}} \leq t\sqrt{\frac{c}{\lambda_j}\frac{n^{1+2\epsilon}\lambda_j^2}{\pts{b\lambda_j n}^2}} =  t\sqrt{\frac{c_1}{\lambda_jn^{1-2\epsilon}}},
    \end{align*}
    where we apply $\tlambda_1 r_0\pts{\cov_{-\bgS}} = \sum_{j=\sk+1}^d \lambda_j \geq b\lambda_j n$ for $j > \sk$.
\end{proof}
\begin{corollary}\label{cor:poly_support_identification_v0}
    Under Assumptions~\ref{asm:rank} and~\ref{asm:k_sparse} with $t\ll n^{\frac{1}{2} -\epsilon}$ and polynomial decay covariance matrix (Definition~\ref{def:polynomial}) with $u\in(0, 1), v = 0$, we assume support are all in the top $\sk$ indices of the covariance spectrum. By substituting $d=n^p$ and $\lambda_j = \frac{1}{j^u}$, for $p\cdot(1-u) < 1$, we have 
    \begin{align*}
        \left|\hevtheta_j\right| \condd{\gg \frac{1}{\sqrt{\lambda_j}}}{\text{for } j \in \gS}{\ll \frac{t}{\sqrt{\lambda_j\cdot n^{1-2\epsilon}}}}{\text{for } j \in [\sk]\cap\gS^\mathsf{c}}{\ll \frac{t}{\sqrt{\lambda_j\cdot n^{1-2\epsilon}} }}{\text{for } \sk < j \leq d}.
    \end{align*}
    If we get access to all $\{\lambda_j\}_{j=1}^d$, we can therefore distinguish support and non-support by examining whether each $\hevtheta_j$ has order $\omega\pts{\frac{1}{\sqrt{\lambda_j}}}$ or decays in a rate of $o\pts{\frac{t}{\sqrt{\lambda_j\cdot n^{1-2\epsilon}}}}$.
\end{corollary}

\begin{proof} Since the conditions in the first scenario in Theorem~\ref{thm:support_identification} are not satisfied, we need to assume $\cov$ is known. Next, we show the rate of supports and non-supports based on Lemma~\ref{lem:recover_support}. We have
    \begin{align*}
        \left|\hevtheta_j\right| &\geq \sqrt{\frac{2\beta_j}{\lambda_j\pi}} \frac{\lambda_j\pts{\frac{n}{c_1\tlambda_{1}r_{0}\pts{\cov_{-\bgS}}}}}{1+\lambda_j\pts{\frac{c_2n}{\tlambda_{1}r_{0}\pts{\cov_{-\bgS}}}}},\quad\quad \text{ for } j \in \gS,\\
        \left|\hevtheta_j\right| &\leq t\sqrt{\frac{c}{\lambda_jn^{1-2\epsilon}}},\quad\quad\quad\quad\quad\quad\quad\quad\quad \text{ for } j \in [\sk]\cap\gS^\mathsf{c},\\
        \left|\hevtheta_j\right| &\leq t\sqrt{\frac{c}{\lambda_j}\frac{n^{1+2\epsilon}\lambda_j^2}{\pts{\tlambda_{1}r_{0}\pts{\cov_{-\bgS}}}^2}},\quad\quad\quad\quad \text{ for } \sk < j \leq d, j \in \gS^\mathsf{c}.
    \end{align*}
    For polynomial decay covariance, we have $\lambda_j = \frac{1}{j^u}$ for $j \in [d]$. We have the lower bound for support indices for $j \in\gS$ as
    \begin{align*}
        \left|\hevtheta_j\right| &\geq \sqrt{\frac{2\beta_j}{\lambda_j\pi}} \frac{\lambda_j\pts{\frac{n}{c_1\tlambda_{1}r_{0}\pts{\cov_{-\bgS}}}}}{1+\lambda_j\pts{\frac{c_2n}{\tlambda_{1}r_{0}\pts{\cov_{-\bgS}}}}}\\
        &= \frac{1}{\sqrt{\lambda_j}}\sqrt{\frac{2\beta_j}{\pi}}\pts{\frac{\frac{1}{j^u}\pts{\frac{n}{c_1\sum_{k > \sk, k\in \gS^\mathsf{c}}\frac{1}{k^u}}}}{1+\frac{1}{j^u}\pts{\frac{c_2n}{\sum_{k > \sk, k\in \gS^\mathsf{c}}\frac{1}{k^u}}}}}\\
        &= \frac{1}{\sqrt{\lambda_j}}\sqrt{\frac{2\beta_j}{\pi}}\pts{\frac{1}{\frac{c_1\sum_{k > \sk, k\in \gS^\mathsf{c}}\frac{1}{k^u}}{\frac{1}{j^u}n}+c_1c_2}}\\
        &\geq \frac{1}{\sqrt{\lambda_j}}\sqrt{\frac{2\beta_j}{\pi}}\pts{\frac{1}{\frac{c_1\sum_{k > \sk, k\in \gS^\mathsf{c}}\frac{1}{k^u}}{\sum_{k > \sk, k\in \gS^\mathsf{c}}\frac{1}{k^u}}+c_1c_2}}\\
        &= \frac{c_{3}}{\sqrt{\lambda_j}},
    \end{align*}
    where we substitute $\lambda_j = \frac{1}{j^u}$ and $\tlambda_1 r_0\pts{\cov_{-\bgS}}= \sum_{j > \sk, j\in \gS^\mathsf{c}}\frac{1}{j^u}$, and by the definition of $\sk$, we apply $\frac{1}{j^u}n \geq \sum_{k > \sk, k\in \gS^\mathsf{c}}\frac{1}{k^u}$ for $j \in [\sk]$ in the last inequality.
    For non-support, $j\in [\sk]\cap\gS^\mathsf{c}$, we have
    \begin{align*}
        \left|\hevtheta_j\right| \leq t\sqrt{\frac{c}{\lambda_jn^{1-2\epsilon}}}.
    \end{align*}
    For $j > \sk, j\in \gS^\mathsf{c}$, we have
    \begin{align*}
        \left|\hevtheta_j\right| &\leq t\sqrt{\frac{c}{\lambda_j}\frac{n^{1+2\epsilon}\lambda_j^2}{\pts{\tlambda_{1}r_{0}(\cov_{-\bgS})}^2}} \leq t\sqrt{\frac{c}{\lambda_j}\frac{n^{1+2\epsilon}\lambda_j^2}{\pts{b\lambda_j n}^2}} =  t\sqrt{\frac{c_1}{\lambda_jn^{1-2\epsilon}}},
    \end{align*}
    where we apply $\tlambda_1 r_0\pts{\cov_{-\bgS}} = \sum_{k=\sk+1, k\in\gS^\mathsf{c}}^d \lambda_k \geq b\lambda_j n$ for $j > \sk$.
\end{proof}
\subsection{Spiked and polynomial decay covariance: support outside the top $\sk$ indices}\label{app:support_id3}
In this section, we show corresponding results to Corollary~\ref{cor:spike_support_identification} and Corollary~\ref{cor:poly_support_identification_v0} in the case where $\svtheta$ is supported outside the top $\sk$ indices of the covariance spectrum. Note that this is necessary to handle scenarios where $\sk=0$, including spiked covariance with $q>1-r$ (Corollary~\ref{cor:spike_support_identification2}) and polynomial decay covariance with $u\in[0, 1)$, $v=0$ and $p \cdot \pts{1-u} > 1$ (Corollary~\ref{cor:poly_support_identification2}). We will see that the results in this section require $t\ll n^{\frac{1}{4}-\epsilon}$ and stronger conditions in $p$ than those in Section \ref{app:support_id2}.

\begin{corollary}\label{cor:spike_support_identification2}
    Under Assumptions~\ref{asm:rank} and~\ref{asm:k_sparse} with $t \ll n^{\frac{1}{4}-\epsilon}$ and spiked covariance matrix (Definition~\ref{def:spiked}), for $q < \pts{1-r}$, if we have $\max\{1.5 -q - r, 1\} < p < 2.5 -q-r$, we can recover support outside of the top $\sk$ indices of the covariance spectrum such that
    \begin{align*}
        \left|\hevtheta_j\right| \conddd{\gg n^{\frac{-p+q+r}{2}}}{\text{for } j \in [\sk]\cap\gS}{\gg n^{\frac{-2p+2}{2}}}{\text{for } j >\sk, j\in\gS}{\ll n^{\frac{-p+q+r-0.5}{2}}}{\text{for } j \in [\sk]\cap \gS^\mathsf{c}}{\ll n^{\frac{-2p+1.5}{2}}}{\text{for } \sk < j \leq d, j\in \gS^\mathsf{c}}.
    \end{align*}
    Therefore, we can pick a threshold between bounds to distinguish support and non-support such as $T=\min\{n^{\frac{-p+q+r}{2}},  n^{\frac{-2p+2}{2}}\}$. On the other hand, for $q > \pts{1-r}$ and $\max\{-0.5 + q + r, 1\}< p < 0.5+q+r$, we have
    \begin{align*}
        \left|\hevtheta_j\right| \conddd{\gg n^{\frac{-p+\pts{1-r}-q+1}{2}}}{\text{for } j \in [s]\cap\gS}{\gg n^{\frac{-2p+2}{2}}}{\text{for } j > s, j \in \gS}{\ll n^{\frac{-p+\pts{1-r}-q + 0.5}{2}}}{\text{for } j \in [s]\cap\gS^\mathsf{c}}{\ll n^{\frac{-2p+1.5}{2}}}{\text{for } s < j \leq d, j\in \gS^\mathsf{c}}.
    \end{align*}
    Therefore, we can pick a threshold between bounds to distinguish support and non-support such as $T=\min\{n^{\frac{-p+\pts{1-r}-q+1}{2}},  n^{\frac{-2p+2}{2}}\}$. 
\end{corollary}
\begin{proof} In the first part, for $q < (1-r)$, we already showed the support lower bound for $j \in [\sk]$ and non-support upper bound for $j \in [\sk]\cap\gS^\mathsf{c}$ and $\sk < j \leq d, j\in \gS^\mathsf{c}$ in Corollary~\ref{cor:spike_support_identification}. Note that we apply a different upper bound of $t$ for non-support upper bound and get different rates. We still need to show the support lower bound for $j > \sk, j \in \gS$. Hence, for $j > \sk, j\in\gS$, according to Lemma~\ref{lem:recover_support}, by substituting $\lambda_j = \frac{\pts{1-a}d}{d-s} = \frac{\pts{1-n^{-q}}n^p}{n^p-n^r}$ and $\tlambda_1 r_0\pts{\cov_{-\bgS}}= \pts{d-|\bgS|}\frac{\pts{1-a}d}{d-s}$, we have 
    \begin{align*}
        \left|\hevtheta_j\right| &\geq \sqrt{\frac{2\beta_j}{\lambda_j\pi}} \frac{\lambda_j\pts{\frac{n}{c_1\tlambda_{1}r_{0}\pts{\cov_{-\bgS}}}}}{1+\lambda_j\pts{\frac{c_2n}{\tlambda_{1}r_{0}\pts{\cov_{-\bgS}}}}}\\
        &= \sqrt{\frac{2\beta_j}{\pi}}\sqrt{\frac{d-s}{\pts{1-a}d}}\pts{\frac{\frac{\pts{1-a}d}{d-s}\pts{\frac{n}{c_1\pts{d-|\bgS|}\frac{\pts{1-a}d}{\pts{d-s}}}}}{1+\frac{\pts{1-a}d}{d-s}\pts{\frac{c_2n}{\pts{d-|\bgS|}\frac{\pts{1-a}d}{\pts{d-s}}}}}}\\
        &= \sqrt{\frac{2\beta_j}{\pi}}\sqrt{\frac{n^p-n^r}{\pts{1-n^{-q}}n^p}}\pts{\frac{\frac{n}{c_1\pts{n^p-|\bgS|}}}{1+\pts{\frac{c_2n}{\pts{n^p-|\bgS|}}}}}\\
        &\geq \sqrt{\frac{2\beta_j}{\pi}}c_3n^{\frac{2-2p}{2}}.
    \end{align*}
    As a result, to recover support outside of top $\sk$ indices of the covariance spectrum, we need that the lower bound of the support is larger than the upper bound of non-support. Hence, we need the following conditions
    \begin{align*}
        n^{\frac{-2p+2}{2}} \gg n^{\frac{-p+q+r-0.5}{2}} &\Leftrightarrow 2.5 - q - r > p,\\
        n^{\frac{-p+q+r}{2}} \gg n^{\frac{-2p+1.5}{2}} &\Leftrightarrow p > 1.5 -q -r.
    \end{align*}
    The proof of the first part is done.
    
    For the second part, we have the support lower bound from Corollary~\ref{cor:spike_support_identification} as
    \begin{equation*}
        \left|\hevtheta_j\right| \geq \sqrt{\frac{2\beta_j}{\pi}}n^{\frac{-p+q+r}{2}}\cdot\frac{\frac{1}{c_8}\cdot n ^{\pts{1-r}-q}}{1+c_9n^{\pts{1-r}-q}}.
    \end{equation*}
    Since $q>(1-r)$, the numerator part decays to zero as
    \begin{align*}
        \left|\hevtheta_j\right| \geq \sqrt{\frac{2\beta_j}{\pi}}cn^{\frac{-p + \pts{1-r}-q + 1}{2}}.
    \end{align*}
    Moreover, we have $r_0\pts{\cov} > bn$ and $\sk = 0$, therefore; for all $j\in\gS^\mathsf{c}$, according to Lemma~\ref{lem:recover_support}, we have the same non-support upper bound as
    \begin{align*}
        \left|\hevtheta_j\right| &\leq t\sqrt{\frac{c}{\lambda_j}\frac{n^{1+2\epsilon}\lambda_j^2}{\pts{\tlambda_1 r_0\pts{\cov_{-\bgS}}}^2}}.
    \end{align*}
    Next, we substitute different values of $\lambda_j$ for $j\in [s]$ and $j > s$. We also assume there are $t_1$ support in the top $s$ indices of the covariance spectrum, where $t_1\leq \min\{s, t\}$. For $j \in [s]\cap\gS^\mathsf{c}$, we have
    \begin{align*}
        \left|\hevtheta_j\right| &\leq t\sqrt{\frac{c}{\lambda_j}\frac{n^{1+2\epsilon}\lambda_j^2}{\pts{\tlambda_1 r_0\pts{\cov_{-\bgS}}}^2}}\\
        &= t\sqrt{c\frac{s}{ad}\pts{\frac{n^{1+2\epsilon}\pts{\frac{ad}{s}}^2}{\pts{\pts{s-t_1}\frac{ad}{s} + \pts{d - s - \pts{t-t_1}}\frac{\pts{1-a}d}{d-s}}^2}}}\\
        &= t\sqrt{cn^{-p+q+r}\pts{\frac{n^{1+2\epsilon}\pts{n^{p-q-r}}^2}{\pts{\pts{n^r-t_1}n^{p-q-r} + \pts{n^p-n^r-\pts{t-t_1}}\frac{\pts{1-n^{-q}}n^p}{n^p-n^r}}^2}}}\\
        &\leq t\sqrt{c_1n^{-p+q+r}\pts{\frac{n^{1+2\epsilon}\pts{n^{p-q-r}}^2}{n^{2p}}}}\\
        &= c_2 tn^{\frac{-p+\pts{1-r}-q+2\epsilon}{2}}\\
        &< c_2 tn^{\frac{-p+\pts{1-r}-q+0.5}{2}},
    \end{align*}
    where we substitute $t\ll n^{\frac{1}{4}-\epsilon}$.
    For $j > s, j\in\gS^\mathsf{c}$, we have the same upper bound as $q < (1-r)$ case in Corollary~\ref{cor:spike_support_identification} such that 
    \begin{align*}
        \left|\hevtheta_j\right| &\leq c_1 tn^{\frac{-2p+1+2\epsilon}{2}} < c_1 n^{\frac{-2p+1.5}{2}},
    \end{align*}
    where we apply $t\ll n^{\frac{1}{4}-\epsilon}$.
    To recover support outside of top $\sk$ indices of the covariance spectrum, we need that the lower bound of the support be larger than the upper bound of non-support. Hence, we need the following conditions
    \begin{align*}
        n^{\frac{-2p+2}{2}} \gg n^{\frac{-p+(1-r)-q+0.5}{2}} &\Leftrightarrow 0.5 + q + r > p,\\
        n^{\frac{-p+(1-r)-q+1}{2}} \gg n^{\frac{-2p+1.5}{2}} &\Leftrightarrow p > -0.5 +q +r.
    \end{align*}
    The proof is done.
\end{proof}

\begin{corollary}\label{cor:poly_support_identification2}
    Under Assumptions~\ref{asm:rank} and~\ref{asm:k_sparse} with $t \ll n^{\frac{1}{4}-\epsilon}$ and polynomial decay covariance in Definition~\ref{def:polynomial} with $u \in [0, 1)$ and $v=0$, for $p\cdot \pts{1-u} > 1$, by substituting $\lambda_j = j^{-u}$, we have
    \begin{align*}
        \left|\hevtheta_j\right| \cond{\gg n\cdot d^{\frac{u-2}{2}}}{\text{for } j \in \gS}{\ll n^{\frac{1.5}{2}}\cdot d^{\frac{2u-2}{2}}}{\text{for } j \in \gS^\mathsf{c}}.
    \end{align*}
    Therefore, if $p\cdot \pts{1-u} > \max\{1, p-\frac{1}{2}\}$, we can pick $T=n\cdot d^{\frac{u-2}{2}}$ as threshold to distinguish support and non-support. Note that when $u=0$, $\cov$ degenerates to isotropic covariance.
\end{corollary}
\begin{proof} First, we already showed that polynomial decay covariance (Definition~\ref{def:polynomial}) with $u \in [0, 1)$, $v = 0$ and $p\cdot\pts{1-u} > 1$ implies $\sk=0$ in Corollary~\ref{cor:poly_cov}. Next, we show the lower bound for support indices $j\in\gS$ and upper bound for non-support indices $j\in\gS^\mathsf{c}$. Based on Lemma~\ref{lem:recover_support}, we have
\begin{align*}
    \left|\hevtheta_j\right| &\geq \sqrt{\frac{2\beta_j}{\lambda_j\pi}} \frac{\lambda_j\pts{\frac{n}{c_1\tlambda_{1}r_{0}\pts{\cov_{-\gS}}}}}{1+\lambda_j\pts{\frac{c_2n}{\tlambda_{1}r_{0}\pts{\cov_{-\gS}}}}},\quad\quad \text{ for } j \in \gS\\
    \left|\hevtheta_j\right| &\leq t\sqrt{\frac{c}{\lambda_j}\frac{n^{1+2\epsilon}\lambda_j^2}{\pts{\tlambda_{1}r_{0}\pts{\cov_{-\gS}}}^2}},\quad\quad\quad\quad \text{ for } j \in \gS^\mathsf{c},
\end{align*}
where $\cov_{-\gS}$ has $\sk=0$ still by Lemma~\ref{lem:balance_extension}. Therefore, for $j\in\gS$, we have
\begin{align*}
    \left|\hevtheta_j\right| \geq \sqrt{\frac{2\beta_j}{\lambda_j\pi}} \frac{\lambda_j\pts{\frac{n}{c_1\tlambda_{1}r_{0}\pts{\cov_{-\gS}}}}}{1+\lambda_j\pts{\frac{c_2n}{\tlambda_{1}r_{0}\pts{\cov_{-\gS}}}}} = \sqrt{\frac{2\beta_j}{\lambda_j\pi}} \frac{\lambda_jn}{c_1\tlambda_1r_0\pts{\cov_{-\gS}}+c_1c_2\lambda_jn} = \sqrt{\frac{2\beta_j}{\pi}} \frac{\sqrt{\lambda_j}n}{c_1\sum_{k=1, k\in\gS^\mathsf{c}}^d\lambda_k+c_1c_2\lambda_jn}.
\end{align*}
Since $\sum_{k=1, k\in\gS^\mathsf{c}}^d\lambda_k \leq \sum_{k=1}^d\lambda_k$ and $\lambda_j n \leq \sum_{k=1}^d\lambda_k$, we can upper bound the denominator by $\sum_{k=1}^d\lambda_k$ and get
\begin{align*}
    \sqrt{\frac{2\beta_j}{\pi}} \frac{\sqrt{\lambda_j}n}{c_1\sum_{k=1, k\in\gS^\mathsf{c}}^d\lambda_k+c_1c_2\lambda_jn} \geq \sqrt{\frac{2\beta_j}{\pi}} \frac{\sqrt{\lambda_j}n}{c_3 \sum_{k=1}^d\lambda_k} \geq \sqrt{\frac{2\beta_j}{\pi}} \frac{\sqrt{\lambda_d}n}{c_3 d^{1-u}} = c_4n\cdot d^{\frac{u-2}{2}}, 
\end{align*}
where we apply~\eqref{eq:poly_upperbound} in the last inequality. For non-support, $j\in \gS^\mathsf{c}$, we have
\begin{align*}
    \left|\hevtheta_j\right| \leq t\sqrt{\frac{c}{\lambda_j}\frac{n^{1+2\epsilon}\lambda_j^2}{\pts{\tlambda_{1}r_{0}(\cov_{-\gS})}^2}} = t\sqrt{c\frac{n^{1+2\epsilon}\lambda_j}{\pts{\sum_{k=1,k\in\gS^\mathsf{c}}^d\lambda_k}^2}} \leq t\sqrt{c\frac{n^{1+2\epsilon}\lambda_1}{\pts{\pts{d-t}\lambda_d}^2}} = t\sqrt{\frac{c\cdot n^{1+2\epsilon} d^{2u}}{\pts{d-t}^2}}
    <c_1n^{\frac{1.5}{2}}d^{u-1},
\end{align*}
where we apply $\lambda_1 \geq \lambda_j$ for all $j\in [d]$ and $\sum_{k=1,k\in\gS^\mathsf{c}}^d\lambda_k \geq \pts{d-t}\lambda_d$. In the last inequality, we apply $t\ll n^{\frac{1}{4}-\epsilon}$. As a result, we need 
\begin{align*}
    n\cdot d^{\frac{u-2}{2}} \gg n^{\frac{1.5}{2}}d^{u-1} \Leftrightarrow n^{\frac{1}{2}} \gg d^u \Leftrightarrow  n^{\frac{1}{2}} \gg n^{p\cdot u} \Leftrightarrow \frac{1}{2} > p\cdot u
\end{align*}
to ensure the lower bound for support indices $j\in\gS$ is larger than the upper bound for non-support indices $j\in\gS^\mathsf{c}$. Ultimately, we can combine the condition $p\cdot \pts{1-u} > 1$ and $\frac{1}{2} > p\cdot u$ into $p\cdot \pts{1-u} > \max\{1, p-\frac{1}{2}\}$. The proof is done.
\end{proof}

\section{Survival and contamination auxiliary lemmas}\label{app:su_cn_auxiliary}
In this section, we provide proofs of Lemmas~\ref{lem:Q_k_bounds}, \ref{lem:Q_1_bounds}, and~\ref{lem:cn_const} from Appendix~\ref{app:su_cn_bound_proof}.
While these are essentially extensions of results of~\cite{wang2023benign}, they are not entirely straightforward due to our analysis of $t$-sparsity.
For example, in Lemma~\ref{lem:Q_k_bounds},~\cite{wang2023benign} analyze $\E\mts{\sgn{\rvz}\rvz^\top}=\sqrt{\frac{2}{\pi}}\mI$ where $\rvz$ is an independent standard Gaussian vector representing the regression labels.
On the other hand, given a support set $\gS$, we must analyze the more complicated term $\E\mts{\sgn{\sum_{j\in \gS} a_j\rvz_j}\rvz_{s_1}^\top}$; this leads to an interesting quantification of the relative contribution of each index of $\svtheta$, which we denote by $\beta$.

\begin{proof}(Lemma~\ref{lem:Q_k_bounds})\\
    According to the definition of $Q_{s_t}$ in~\Eqref{eq:Q_def}, we have
    \begin{align*}
        Q_{s_t} = \rvz_{s_1}^\top\rmA_{-s_1:s_t}^{-1}\hrvy = \rvz_{s_1}^\top\rmA_{-s_1:s_t}^{-1}\sgn{\sum_{j\in\gS} a_j\rvz_j}
    \end{align*}
    Since $\rvz_{s_1}$ and $\sgn{\sum_{j\in\gS} a_j\rvz_j}$ are independent to $\rmA_{-s_1:s_t}^{-1}$, by applying the parallelogram rule and the Hanson-Wright inequality in Lemma~\ref{lem:hanson_wright_ineq}, we get
    \begin{align*}
        Q_{s_t} &\leq \E\mts{\rvz_{s_1}^\top\rmA_{-s_1:s_t}^{-1}\hrvy} + 2c_1\norm{\rmA_{-s_1:s_t}^{-1}}_2 \cdot n^{\frac{1}{2}+\epsilon},\\
        Q_{s_t} &\geq \E\mts{\rvz_{s_1}^\top\rmA_{-s_1:s_t}^{-1}\hrvy} - 2c_1\norm{\rmA_{-s_1:s_t}^{-1}}_2 \cdot n^{\frac{1}{2}+\epsilon},
    \end{align*}
    with probability at least $1-2e^{-n^{2\epsilon}}$.
    Next, we calculate the value of $\E\mts{\rvz_{s_1}^\top\rmA_{-s_1:s_t}^{-1}\hrvy}$.
    \begin{align*}
    \E\mts{\rvz_{s_1}^\top\rmA_{-s_1:s_t}^{-1}\hrvy} &= \Tr\pts{\rmA_{-s_1:s_t}^{-1}\E\mts{\hrvy\rvz_{s_1}^\top}}\\
    &= \Tr\pts{\rmA_{-s_1:s_t}^{-1}\E\mts{\sgn{\sum_{j\in\gS} a_j z_{j,1}}z_{s_1,1}}\cdot \mI}\\
    &=\E\mts{\sgn{\sum_{j\in\gS} a_j z_{j,1}}z_{s_1,1}} \Tr\pts{\rmA_{-s_1:s_t}^{-1}}.
    \end{align*}
    We now derive the value of $\E\mts{\sgn{\sum_{j\in\gS} a_j z_{j,1}}z_{s_1,1}}$. We denote $\rp \coloneqq z_{s_1,1} \sim \gN(0, 1)$ and $\rq \coloneqq \sum_{j\in\gS, j \neq s_1} a_j z_{j,1} \sim \gN(0, \sigma_{\rq}^2)$ where $\sigma_{\rq}^2 \coloneqq \sum_{j\in\gS, j \neq s_1} a_j^2$. We then have
    \begin{align*}
        \E\mts{\sgn{\sum_{j\in\gS} a_j z_{j,1}}z_{s_1,1}} &= \E\mts{\sgn{a_{s_1} z_{s_1, 1} + \sum_{j\in\gS, j\neq s_1} a_j z_{j,1}}z_{s_1,1}}\\
        &= \E_\rp\mts{\E_\rq\mts{\sgn{a_{s_1}\rp + \rq}} \cdot \rp}\\
        &= \E_\rp\mts{\pts{1-\Phi_\rq\pts{-a_{s_1}\rp} - \Phi_\rq\pts{-a_{s_1}\rp}}\cdot \rp} \\
        &= 2 \E_\rp\mts{\Phi_\rq\pts{a_{s_1}\rp} \cdot \rp},
    \end{align*}
    where $\Phi_{\rq}$ is the cdf of $\rq$ and we use the property that $1 - \Phi_\rq\pts{-a_{s_1}\rp} = \Phi_\rq\pts{a_{s_1}\rp}$. Then, we have
    \begin{align*}
        \E_\rp\mts{\Phi_\rq\pts{a_{s_1}\rp}\cdot \rp} &= \E_\rp\mts{  \E_\rq\mts{\1_{q \leq a_{s_1}\rp }| \rp}\cdot \rp}\\
        &= \E_{\rp, \rq}\mts{\1_{q \leq a_{s_1}\rp } \cdot \rp}\\
        &= \sgn{a_{s_1}}\int_{-\infty}^{+\infty}\frac{1}{\sqrt{2\pi\sigma_{\rq}^2}}e^{\frac{-\rq^2}{2\sigma_{\rq}^2}}\pts{\int_{\frac{\rq}{|a_{s_1}|}}^{+\infty}\frac{1}{\sqrt{2\pi}}\rp e^{\frac{-\rp^2}{2}}d\rp}d\rq\\
        &= \sgn{a_{s_1}}\int_{-\infty}^{+\infty}\frac{1}{\sqrt{2\pi\sigma_{\rq}^2}}e^{\frac{-\rq^2}{2\sigma_{\rq}^2}}\pts{\frac{1}{\sqrt{2\pi}} e^{\frac{-\rq^2}{2a_{s_1}^2}}}d\rq\\
        &= \frac{\sgn{a_{s_1}}}{\sqrt{2\pi}}\int_{-\infty}^{+\infty}\frac{1}{\sqrt{2\pi\sigma_{\rq}^2}}e^{\frac{-\pts{\sum_{j\in\gS}a_j^2}\rq^2}{2\sigma_{\rq}^2a_{s_1}^2}}d\rq\\
    &= \frac{a_{s_1}}{\sqrt{2\pi\pts{\sum_{j\in\gS}a_j^2}}}\\
    &=\sgn{a_{s_1}}\sqrt{\frac{\beta_{s_1}}{2\pi}},
    \end{align*}
    where the second equality uses the law of total expectation, and we substitute $\beta_{s_\ell} \coloneqq \frac{\lambda_{s_\ell}\sevthetas_{s_\ell}}{\sum_{j\in\gS} \lambda_j\sevthetas_j} = \frac{a_{s_\ell}^2}{\sum_{j\in\gS} a_j^2}$ in the last equality. The proof is done by substituting the expectation values in the Hanson-Wright inequalities.
\end{proof}

\begin{proof}(Lemma~\ref{lem:Q_1_bounds})\\
    Recall the definition of $Q_{s_\ell}$ in~\Eqref{eq:Q_def} such that $Q_{s_\ell} \coloneqq \rvz_{s_1}^\top \rmA_{-s_1:s_\ell}^{-1}\hrvy$ for $1 \leq \ell \leq t$. Therefore, we can write $Q_{s_1}$ as
    \begin{align*}
        Q_{s_1} &= \rvz_{s_1}^\top\rmA_{-s_1:s_1}^{-1}\hrvy\\
            &= \rvz_{s_1}^\top\pts{\rmA_{-s_1:s_2} + \lambda_{s_2}\rvz_{s_2}\rvz_{s_2}^\top}^{-1}\hrvy\\
            &= \rvz_{s_1}^\top\pts{\rmA_{-s_1:s_2}^{-1} - \frac{\lambda_{s_2}\rmA_{-s_1:s_2}^{-1}\rvz_{s_2}\rvz_{s_2}^\top\rmA_{-s_1:s_2}^{-1}}{1+ \lambda_{s_2}\rvz_{s_2}^\top\rmA_{-s_1:s_2}^{-1}\rvz_{s_2}}}\hrvy\\
            &= \rvz_{s_1}^\top\pts{\rmA_{-s_1:s_t}^{-1} -\sum_{\ell=2}^t \frac{\lambda_{s_\ell}\rmA_{-s_1:s_\ell}^{-1}\rvz_{s_\ell}\rvz_{s_\ell}^\top\rmA_{-s_1:s_\ell}^{-1}}{1+ \lambda_{s_\ell}\rvz_{s_\ell}^\top\rmA_{-s_1:s_\ell}^{-1}\rvz_{s_\ell}}}\hrvy\\
            &= Q_{s_t} - \sum_{\ell=2}^t \frac{\lambda_{s_\ell}\rvz_{s_1}^\top\rmA_{-s_1:s_\ell}^{-1}\rvz_{s_\ell}\rvz_{s_\ell}^\top\rmA_{-s_1:s_\ell}^{-1}\hrvy}{1+ \lambda_{s_\ell}\rvz_{s_\ell}^\top\rmA_{-s_1:s_\ell}^{-1}\rvz_{s_\ell}}\\
            &= Q_{s_t} \pts{1 - \sum_{\ell=2}^t \frac{\lambda_{s_\ell}\rvz_{s_1}^\top\rmA_{-s_1:s_\ell}^{-1}\rvz_{s_\ell}\rvz_{s_\ell}^\top\rmA_{-s_1:s_\ell}^{-1}\hrvy}{Q_{s_t}\pts{1+ \lambda_{s_\ell}\rvz_{s_\ell}^\top\rmA_{-s_1:s_\ell}^{-1}\rvz_{s_\ell}}}}\\
            &=Q_{s_t} \pts{1 - \sgn{a_{s_1}}\sum_{\ell=2}^t \frac{\lambda_{s_\ell}\rvz_{s_1}^\top\rmA_{-s_1:s_\ell}^{-1}\rvz_{s_\ell}\rvz_{s_\ell}^\top\rmA_{-s_1:s_\ell}^{-1}\hrvy}{|Q_{s_t}|\pts{1+ \lambda_{s_\ell}\rvz_{s_\ell}^\top\rmA_{-s_1:s_\ell}^{-1}\rvz_{s_\ell}}}}
    \end{align*}
    where we apply the Sherman-Morrison-Woodbury identity recursively over $\ell=2,\dots,t$. In the last equality, we use that fact that the sign of $Q_{s_t}$ is controlled by $a_{s_1}$ for large enough $n$ from Lemma~\ref{lem:Q_k_bounds}. According to the sign of $Q_{s_t}$, we have the upper and lower bound of $Q_{s_1}$ by
    \begin{align*}
         Q_{s_1} &\leq Q_{s_t} \pts{1 +\sgn{a_{s_1}}\left|\sum_{\ell=2}^t \frac{\lambda_{s_\ell}\rvz_{s_1}^\top\rmA_{-s_1:s_\ell}^{-1}\rvz_{s_\ell}\rvz_{s_\ell}^\top\rmA_{-s_1:s_\ell}^{-1}\hrvy}{|Q_{s_t}|\pts{1+ \lambda_{s_\ell}\rvz_{s_\ell}^\top\rmA_{-s_1:s_\ell}^{-1}\rvz_{s_\ell}}}\right|}\\
         Q_{s_1} &\geq Q_{s_t} \pts{1 - \sgn{a_{s_1}}\left|\sum_{\ell=2}^t \frac{\lambda_{s_\ell}\rvz_{s_1}^\top\rmA_{-s_1:s_\ell}^{-1}\rvz_{s_\ell}\rvz_{s_\ell}^\top\rmA_{-s_1:s_\ell}^{-1}\hrvy}{|Q_{s_t}|\pts{1+ \lambda_{s_\ell}\rvz_{s_\ell}^\top\rmA_{-s_1:s_\ell}^{-1}\rvz_{s_\ell}}}\right|}.
    \end{align*}
    Hence, it remains the upper bound the absolute value term above. We then have 
    \begin{align*}
        \left|\sum_{\ell=2}^t \frac{\lambda_{s_\ell}\rvz_{s_1}^\top\rmA_{-s_1:s_\ell}^{-1}\rvz_{s_\ell}\rvz_{s_\ell}^\top\rmA_{-s_1:s_\ell}^{-1}\hrvy}{|Q_{s_t}|\pts{1+ \lambda_{s_\ell}\rvz_{s_\ell}^\top\rmA_{-s_1:s_\ell}^{-1}\rvz_{s_\ell}}}\right| &\leq \sum_{\ell=2}^t \frac{\lambda_{s_\ell}|\rvz_{s_1}^\top\rmA_{-s_1:s_\ell}^{-1}\rvz_{s_\ell}||\rvz_{s_\ell}^\top\rmA_{-s_1:s_\ell}^{-1}\hrvy|}{|Q_{s_t}|\pts{1+ \lambda_{s_\ell}\rvz_{s_\ell}^\top\rmA_{-s_1:s_\ell}^{-1}\rvz_{s_\ell}}}\\
        &\leq \sum_{\ell=2}^t \frac{|\rvz_{s_1}^\top\rmA_{-s_1:s_\ell}^{-1}\rvz_{s_\ell}||\rvz_{s_\ell}^\top\rmA_{-s_1:s_\ell}^{-1}\hrvy|}{|Q_{s_t}|\pts{\rvz_{s_\ell}^\top\rmA_{-s_1:s_\ell}^{-1}\rvz_{s_\ell}}}\\
        &= \sum_{\ell=2}^t \underbrace{\frac{|\rvz_{s_1}^\top\rmA_{-s_1:s_\ell}^{-1}\rvz_{s_\ell}|}{|Q_{s_t}|}}_{\coloneqq T_1}\underbrace{\frac{|\rvz_{s_\ell}^\top\rmA_{-s_1:s_\ell}^{-1}\hrvy|}{\rvz_{s_\ell}^\top\rmA_{-s_1:s_\ell}^{-1}\rvz_{s_\ell}}}_{\coloneqq T_2},
    \end{align*}
    where in the second inequality we deduct 1 in the denominator. Next, we aim to upper bound $T_1$ and $T_2$ respectively. According to the bounds of $Q_{s_t}$ in Lemma~\ref{lem:Q_k_bounds}, we can have the bounds of $|Q_{s_t}|$ as
    \begin{align}
       \sqrt{\frac{2\beta_{s_1}}{\pi}}\Tr\pts{\rmA_{-s_1:s_t}^{-1}} - 2c_1\norm{\rmA_{-s_1:s_t}^{-1}}_2 \cdot n^{\frac{1}{2}+\epsilon} \leq |Q_{s_t}| \leq \sqrt{\frac{2\beta_{s_1}}{\pi}}\Tr\pts{\rmA_{-s_1:s_t}^{-1}} + 2c_1\norm{\rmA_{-s_1:s_t}^{-1}}_2 \cdot n^{\frac{1}{2}+\epsilon}\label{eq:Q_t_abs}.
    \end{align}
    For $T_1$, we can apply Hanson-Wright inequality (Lemma~\ref{lem:hanson_wright_ineq}) in the numerator and apply the lower bound of $|Q_{s_t}|$ in~\Eqref{eq:Q_t_abs} in the denominator, and we get
    \begin{align}
        T_1 &\leq \frac{2c_1\norm{\rmA_{-s_1:s_\ell}^{-1}}_2\cdot n^{\frac{1}{2}+\epsilon}}{\sqrt{\frac{2\beta_{s_1}}{\pi}}\Tr\pts{\rmA_{-s_1:s_t}^{-1}} - 2c_1\norm{\rmA_{-s_1:s_t}^{-1}}_2 \cdot n^{\frac{1}{2}+\epsilon}}\nonumber\\
        &\leq \frac{2c_1\norm{\rmA_{-\pts{s_1:s_t}\cup[\sk]}^{-1}}_2\cdot n^{\frac{1}{2}+\epsilon}}{\sqrt{\frac{2\beta_{s_1}}{\pi}}\Tr\pts{\rmA_{-s_1:s_t}^{-1}} - 2c_1\norm{\rmA_{-\pts{s_1:s_t}\cup[\sk]}^{-1}}_2 \cdot n^{\frac{1}{2}+\epsilon}}\nonumber\\
        &\leq \frac{2c_1\norm{\rmA_{-\pts{s_1:s_t}\cup[\sk]}^{-1}}_2\cdot n^{\frac{1}{2}+\epsilon}}{\sqrt{\frac{2\beta_{s_1}}{\pi}}\pts{1-\frac{c}{n}}^{\sk}\Tr\pts{\rmA_{-\pts{s_1:s_t}\cup[\sk]}^{-1}} - 2c_1\norm{\rmA_{-\pts{s_1:s_t}\cup[\sk]}^{-1}}_2 \cdot n^{\frac{1}{2}+\epsilon}}\nonumber\\
        &= \frac{1}{\sqrt{\frac{2\beta_{s_1}}{\pi}}\pts{1-\frac{c}{n}}^{\sk}\frac{\Tr\pts{\rmA_{-\pts{s_1:s_t}\cup[\sk]}^{-1}}}{2c_1\norm{\rmA_{-\pts{s_1:s_t}\cup[\sk]}^{-1}}_2 n^{\frac{1}{2}+\epsilon}} - 1}\label{eq:Q_1_t1_upper_bound},
    \end{align}
    where the second inequality follows by $\rmA_{-s_1:s_\ell}^{-1} \preceq \rmA_{-s_1:s_t}^{-1} \preceq \rmA_{-\pts{s_1:s_t}\cup[\sk]}^{-1}$ for all $\ell \leq t$, and the third inequality follows by the trace lower bound in Lemma~\ref{lem:trace_lower_bound}. Next, Lemma~\ref{lem:eig_constant} ensures eigenvalues of $\rmA_{-\pts{s_1:s_t}\cup[\sk]}^{-1}$ are identical up to a constant such that 
    \begin{align}\label{eq:trace_max_eig_tight}
        \frac{\Tr\pts{\rmA_{-\pts{s_1:s_t}\cup[\sk]}^{-1}}}{\norm{\rmA_{-\pts{s_1:s_t}\cup[\sk]}^{-1}}_2} \geq \frac{n}{c}.
    \end{align}
    Substitute~\Eqref{eq:trace_max_eig_tight} into~\Eqref{eq:Q_1_t1_upper_bound}, we get
    \begin{align*}
        T_1 \leq \frac{1}{\sqrt{\frac{2\beta_{s_1}}{\pi}}\pts{1-\frac{c}{n}}^{\sk}\frac{n^{\frac{1}{2}-\epsilon}}{c_2} - 1} = \frac{1}{c_3n^{\frac{1}{2}-\epsilon}-1}.
    \end{align*}
    For $T_2$, we use the sub-multiplicative matrix norm and get
    \begin{align}
        T_2 &\leq \frac{\norm{\rvz_{s_\ell}}_2\norm{\hrvy}_2\norm{\rmA_{-s_1:s_\ell}^{-1}}_2}{\rvz_{s_\ell}^\top\rmA_{-s_1:s_\ell}^{-1}\rvz_{s_\ell}}\nonumber\\
        &\leq \frac{\norm{\rvz_{s_\ell}}_2\norm{\hrvy}_2\norm{\rmA_{-s_1:s_\ell}^{-1}}_2}{\Tr\pts{\rmA_{-s_1:s_\ell}^{-1}} - c_1\norm{\rmA_{-s_1:s_\ell}^{-1}}_2\cdot n^{\frac{1}{2}+\epsilon}}\nonumber\\
        &\leq \frac{\norm{\rvz_{s_\ell}}_2\norm{\hrvy}_2\norm{\rmA_{-s_1:s_\ell}^{-1}}_2}{\pts{1-\frac{c}{n}}^{\sk+t-\ell}\Tr\pts{\rmA_{-\pts{s_1:s_t}\cup[\sk]}^{-1}} - c_1\norm{\rmA_{-s_1:s_\ell}^{-1}}_2\cdot n^{\frac{1}{2}+\epsilon}}\nonumber\\
        &\leq \frac{\norm{\rvz_{s_\ell}}_2\norm{\hrvy}_2\norm{\rmA_{-\pts{s_1:s_t}\cup[\sk]}^{-1}}_2}{\pts{1-\frac{c}{n}}^{\sk+t -l}\Tr\pts{\rmA_{-\pts{s_1:s_t}\cup[\sk]}^{-1}} - c_1\norm{\rmA_{-\pts{s_1:s_t}\cup[\sk]}^{-1}}_2\cdot n^{\frac{1}{2}+\epsilon}}\label{eq:Q_1_t2_upper_bound},
    \end{align}
    where the second inequality follows the Hanson-Wright inequality in Lemma~\ref{lem:hanson_wright_ineq}. In the third inequality, we use the trace lower bound in Lemma~\ref{lem:trace_lower_bound}. The last inequality follows by $\rmA_{-\pts{s_1:s_t}\cup[\sk]}^{-1} \succeq \rmA_{-s_1:s_\ell}^{-1}$ for $\ell \leq t$. Next, we apply the tightness of eigenvalues of $\rmA_{-\pts{s_1:s_t}\cup[\sk]}^{-1}$ in~\Eqref{eq:trace_max_eig_tight} again and get
    \begin{align}
        T_2 \leq \frac{\norm{\rvz_{s_\ell}}_2\norm{\hrvy}_2}{\pts{1-\frac{c}{n}}^{\sk+t-\ell}\frac{n}{c} - c_1 n^{\frac{1}{2}+\epsilon}} \leq \frac{cn}{\pts{1-\frac{c}{n}}^{\sk+t-\ell}\frac{n}{c} - c_1 n ^{\frac{1}{2}+\epsilon}}\leq{c_4}\nonumber,
    \end{align}
    where the last inequality follows Lemma~\ref{lem:z_bound} such that $\norm{\rvz_{s_\ell}}_2 \leq c\sqrt{n}$. Put together the upper bound of $T_1$ and $T_2$, the proof is complete.
\end{proof}
\begin{proof}(Lemma~\ref{lem:cn_const})\\
    We show that $\left|\frac{\lambda_{s_\ell}\rvz_{s_\ell}^\top\rmA_{-s_1:s_\ell}^{-1}\crvy_{s_{\ell-1}}}{1+\lambda_{s_\ell}\rvz_{s_\ell}^\top\rmA_{-s_1:s_\ell}^{-1}\rvz_{s_\ell}}\right| \leq c$ for $1 \leq \ell \leq t$ by induction. Recall the definition of $\crvy$ as $\crvy_{s_\ell} \coloneqq \crvy_{s_{\ell-1}} - \frac{\lambda_{s_\ell}\rvz_{s_\ell}^\top\rmA_{-s_1:s_\ell}^{-1}\crvy_{s_{\ell-1}}}{1+\lambda_{s_\ell}\rvz_{s_\ell}^\top\rmA_{-s_1:s_\ell}^{-1}\rvz_{s_\ell}}\rvz_{s_\ell}$ and $\crvy_{s_0} = \hrvy$ for $1\leq \ell \leq t$. For the base case $\ell=1$, we have
    \begin{align}
        \frac{\left|\lambda_{s_1}\rvz_{s_1}^\top\rmA_{-s_1:s_1}^{-1}\hrvy\right|}{1+\lambda_{s_1}\rvz_{s_1}^\top\rmA_{-s_1:s_1}^{-1}\rvz_{s_1}} &\leq \frac{\left|\rvz_{s_1}^\top\rmA_{-s_1:s_1}^{-1}\hrvy\right|}{\rvz_{s_1}^\top\rmA_{-s_1:s_1}^{-1}\rvz_{s_1}} \\
        &\leq \frac{\norm{\rvz_{s_1}}_2\norm{\rmA_{-s_1:s_1}^{-1}}_2\norm{\hrvy}_2}{\Tr\pts{\rmA_{-s_1:s_1}^{-1}}-c_1\norm{\rmA_{-s_1:s_1}^{-1}}_2\cdot n^{\frac{1}{2}+\epsilon}}\nonumber\\
        &\leq \frac{\norm{\rvz_{s_1}}_2\norm{\rmA_{-\pts{s_1:s_t}\cup[\sk]}^{-1}}_2\norm{\hrvy}_2}{\Tr\pts{\rmA_{-s_1:s_1}^{-1}}-c_1\norm{\rmA_{-\pts{s_1:s_t}\cup[\sk]}^{-1}}_2\cdot n^{\frac{1}{2}+\epsilon}}\label{eq:cy_base},
    \end{align}
    where in the second inequality, we use the sub-multiplicative of matrix norm in the numerator and Hanson-Wright inequality (Lemma~\ref{lem:hanson_wright_ineq}) in the denominator. In the last inequality, we use the fact that $\rmA_{-\pts{s_1:s_t}\cup[\sk]}^{-1} \succeq \rmA_{-s_1:s_1}^{-1}$. Next, we apply Lemma~\ref{lem:trace_lower_bound} and get $\Tr\pts{\rmA_{-s_1:s_1}^{-1}} \geq \pts{1-\frac{c}{n}}^{\sk+t-1} \Tr\pts{\rmA_{-\pts{s_1:s_t}\cup[\sk]}^{-1}}$. Therefore,
    \begin{align*}
        \pts{\ref{eq:cy_base}} &\leq \frac{\norm{\rvz_{s_1}}_2\norm{\rmA_{-\pts{s_1:s_t}\cup[\sk]}^{-1}}_2\norm{\hrvy}_2}{\pts{1-\frac{c}{n}}^{\sk+t-1} \Tr\pts{\rmA_{-\pts{s_1:s_t}\cup[\sk]}^{-1}}-c_1\norm{\rmA_{-\pts{s_1:s_t}\cup[\sk]}^{-1}}_2\cdot n^{\frac{1}{2}+\epsilon}}\\
        &= \frac{\norm{\rvz_{s_1}}_2\norm{\hrvy}_2}{\pts{1-\frac{c}{n}}^{\sk+t-1} \frac{\Tr\pts{\rmA_{-\pts{s_1:s_t}\cup[\sk]}^{-1}}}{\norm{\rmA_{-\pts{s_1:s_t}\cup[\sk]}^{-1}}_2}- c_1n^{\frac{1}{2}+\epsilon}} \\
        &\leq \frac{cn}{\pts{1-\frac{c}{n}}^{\sk+t-1} \frac{n}{c}- c_1n^{\frac{1}{2}+\epsilon}} \\
        &\leq c,
    \end{align*}
    where in the second inequality we apply Lemma~\ref{lem:z_bound} to get $\norm{\rvz_{s_1}}_2 \leq c\sqrt{n}$ and we apply Lemma~\ref{lem:eig_constant} to show eigenvalues of $\rmA_{-\pts{s_1:s_t}\cup[\sk]}^{-1}$ are identical up to a constant. 
    The base case $\ell=1$ is proved. Next, we assume $\left|\frac{\lambda_{s_j}\rvz_{s_j}^\top\rmA_{-s_1:s_j}^{-1}\crvy_{s_{j-1}}}{1+\lambda_{s_j}\rvz_{s_j}^\top\rmA_{-s_1:s_j}^{-1}\rvz_{s_j}}\right| \leq c$ is true for $1 \leq j \leq\ell-1$, and we show for $j=\ell$ the statement holds. We have the $j=\ell$ case
    \begin{align*}
        \left|\frac{\lambda_{s_\ell}\rvz_{s_\ell}^\top\rmA_{-s_1:s_\ell}^{-1}\crvy_{s_{\ell-1}}}{1+\lambda_{s_\ell}\rvz_{s_\ell}^\top\rmA_{-s_1:s_\ell}^{-1}\rvz_{s_\ell}}\right|&\leq \frac{\left|\rvz_{s_\ell}^\top\rmA_{-s_1:s_\ell}^{-1}\crvy_{s_{\ell-1}}\right|}{\rvz_{s_\ell}^\top\rmA_{-s_1:s_\ell}^{-1}\rvz_{s_\ell}} \\
        &= \frac{\left|\rvz_{s_\ell}^\top\rmA_{-s_1:s_\ell}^{-1}\pts{\hrvy - \sum_{j=1}^{l-1}\frac{\lambda_{s_j}\rvz_{s_j}^\top\rmA_{-s_1:s_j}^{-1}\crvy_{s_{j-1}}}{1+\lambda_{s_j}\rvz_{s_j}^\top\rmA_{-s_1:s_j}^{-1}\rvz_{s_j}}\rvz_{s_j} }\right|}{\rvz_{s_\ell}^\top\rmA_{-s_1:s_\ell}^{-1}\rvz_{s_\ell}}\\
        &\leq \frac{\left|\rvz_{s_\ell}^\top\rmA_{-s_1:s_\ell}^{-1}\hrvy\right|}{\rvz_{s_\ell}^\top\rmA_{-s_1:s_\ell}^{-1}\rvz_{s_\ell}} + \sum_{j=1}^{l-1} \frac{\left|\frac{\lambda_{s_j}\rvz_{s_j}^\top\rmA_{-s_1:s_j}^{-1}\crvy_{s_{j-1}}}{1+\lambda_{s_j}\rvz_{s_j}^\top\rmA_{-s_1:s_j}^{-1}\rvz_{s_j}}\rvz_{s_\ell}^\top\rmA_{-s_1:s_\ell}^{-1}\rvz_{s_j}\right|}{\rvz_{s_\ell}^\top\rmA_{-s_1:s_\ell}^{-1}\rvz_{s_\ell}}\\
        &\leq \frac{\left|\rvz_{s_\ell}^\top\rmA_{-s_1:s_\ell}^{-1}\hrvy\right|}{\rvz_{s_\ell}^\top\rmA_{-s_1:s_\ell}^{-1}\rvz_{s_\ell}} + \sum_{j=1}^{l-1} \frac{c\left|\rvz_{s_\ell}^\top\rmA_{-s_1:s_\ell}^{-1}\rvz_{s_j}\right|}{\rvz_{s_\ell}^\top\rmA_{-s_1:s_\ell}^{-1}\rvz_{s_\ell}},
    \end{align*}
    where we upper bound the term by taking absolute value individually, and in the last inequality we apply the induction assumption for $1\leq j \leq\ell-1$. For the first term, we can achieve a constant upper bound by following the exact procedure in the base case $\ell=1$. For the second term, we can use the Hanson-Wright inequality (Lemma~\ref{lem:hanson_wright_ineq}) and show
    \begin{align*}
        \frac{\left|\rvz_{s_\ell}^\top\rmA_{-s_1:s_\ell}^{-1}\rvz_{s_j}\right|}{\rvz_{s_\ell}^\top\rmA_{-s_1:s_\ell}^{-1}\rvz_{s_\ell}} &\leq \frac{c_1\norm{\rmA_{-s_1:s_\ell}^{-1}}_2\cdot n^{\frac{1}{2}+\epsilon}}{\Tr\pts{\rmA_{-s_1:s_\ell}^{-1}}-c_1\norm{\rmA_{-s_1:s_\ell}^{-1}}_2\cdot n^{\frac{1}{2}+\epsilon}}\\
        &\leq \frac{c_1\norm{\rmA_{-\pts{s_1:s_t}\cup[\sk]}^{-1}}_2\cdot n^{\frac{1}{2}+\epsilon}}{\Tr\pts{\rmA_{-s_1:s_\ell}^{-1}}-c_1\norm{\rmA_{-\pts{s_1:s_t}\cup[\sk]}^{-1}}_2\cdot n^{\frac{1}{2}+\epsilon}}\\
        &\leq \frac{c_1\norm{\rmA_{-\pts{s_1:s_t}\cup[\sk]}^{-1}}_2\cdot n^{\frac{1}{2}+\epsilon}}{\pts{1-\frac{c}{n}}^{\sk+t-\ell} \Tr\pts{\rmA_{-\pts{s_1:s_t}\cup[\sk]}^{-1}}-c_1\norm{\rmA_{-\pts{s_1:s_t}\cup[\sk]}^{-1}}_2\cdot n^{\frac{1}{2}+\epsilon}}\\
        &= \frac{1}{\pts{1-\frac{c}{n}}^{\sk+t-\ell}\frac{\Tr\pts{\rmA_{-\pts{s_1:s_t}\cup[\sk]}^{-1}}}{c_1\norm{\rmA_{-\pts{s_1:s_t}\cup[\sk]}^{-1}}_2\cdot n^{\frac{1}{2}+\epsilon}} - 1} \\
        &\leq \frac{1}{\pts{1-\frac{c}{n}}^{\sk+t-\ell}\frac{n^{\frac{1}{2}-\epsilon}}{c} - 1} \\
        &\leq \frac{c}{n^{\frac{1}{2}-\epsilon}},
    \end{align*}
    where we use the fact that $\rmA_{-\pts{s_1:s_t}\cup[\sk]}^{-1} \succeq \rmA_{-s_1:s_\ell}^{-1}$ for $\ell \leq t$ in the second inequality and we use Lemma~\ref{lem:trace_lower_bound} to get the trace lower bound in the third inequality. Finally, we again apply Lemma~\ref{lem:eig_constant} to show eigenvalues in $\rmA_{-\pts{s_1:s_t}\cup[\sk]}^{-1}$ are identical up to a constant. Putting together the bounds for the first term and the second term, we complete the induction proof for $j =\ell-1$ case. The proof is done.

\end{proof}

\section{Background lemmas}\label{app:background_lemmas}
In this section, we provide statements and/or proofs of some miscellaneous lemmas.

The first lemma is the Hanson-Wright inequality, which demonstrates the quadratic term $\rvz^\top\rmM\rvz$ of sub-Gaussian random vector $\rvz$ concentrates around its expectation.
\begin{lemma}(Hanson-Wright Inequalities,~\cite{rudelson2013hanson})\label{lem:hanson_wright_ineq}
Let $\rvz$ be a random vector with i.i.d. sub-Gaussian entries $\rvz_i$ such that $\E\mts{\rvz_i} = 0$ and $\norm{\rvz_i}_{\psi_2}\leq 1$. There exists a universal constant $c > 0$ such that for any positive semi-definite matrix $\rmM$ and for every $t \geq 0$, we have
\begin{align*}
    \sP\mts{\left|\rvz^\top\rmM\rvz - \E\mts{\rvz^\top\rmM\rvz}\right| > t} \leq 2 \exp{\left\{-c\min\left\{\frac{t^2}{\norm{\rmM}_{\mathsf{F}}^2},\frac{t}{\norm{\rmM}_2}\right\}\right\}}.
\end{align*}
Note that $\norm{\rmM}_{\mathsf{F}}^2 \leq n\norm{\rmM}_2^2$ and we substitute $t=c_1\norm{\rmM}_2\cdot n^{\frac{1}{2} + \epsilon}$ where $c_1^2 = \frac{1}{c}$ and $\epsilon\in (0,\frac{1}{4})$ to get
\begin{align*}
    \left|\rvz^\top\rmM\rvz - \E\mts{\rvz^\top\rmM\rvz}\right| \leq c_1\norm{\rmM}_2\cdot n^{\frac{1}{2} + \epsilon},
\end{align*}
with probability at least $1-2e^{-n^{2\epsilon}}$. Again, note that $\norm{\rmM}_2 \leq \Tr\pts{\rmM}$ and $\norm{\rmM}_{\mathsf{F}}^2=\Tr\pts{\rmM^2}\leq \pts{\Tr\pts{\rmM}}^2$, and we substitute $t=\frac{1}{c}\cdot\Tr\pts{\rmM}\cdot\pts{\ln n}$ to get
\begin{align*}
    \rvz^\top\rmM\rvz \leq \E\mts{\rvz^\top\rmM\rvz} + \frac{1}{c}\Tr\pts{\rmM}\cdot \pts{\ln n} \leq \pts{1+\frac{1}{c}}\cdot\Tr\pts{\rmM}\cdot\pts{\ln n},
\end{align*}
with probability at least $1-\frac{1}{n}$. Note that the probabilities are over $\rvz$ only and $\rmM$ is positive semi-definite and is independent to $\rvz$.   
\end{lemma}
The next lemma restates a bound on the squared norm of a Gaussian random vector.
\begin{lemma}\label{lem:z_bound}
    Let $\rvz\sim\gN(\vmu, \rmI_n)$ and for $\delta\in\pts{0, 1}$ and $c>1$, we have
    \begin{align*}
        \frac{n}{c} = n\pts{1-\delta} \leq \norm{\rvz}_2^2 \leq n\pts{1+\delta} = cn,
    \end{align*}
    with probability at least $1 - 2e^{-n\delta^2}$.
\end{lemma}
The following lemma guarantees that if $\cov$ exhibits a heavy tail such that $r_{\sk}\pts{\cov} \geq bn$, it retains a heavy tail even after removing $t \ll n^{\frac{1}{2}}$ components from the tail.
\begin{lemma}\label{lem:balance_extension}
    For any data covariance matrix $\cov$ satisfying $\sk \leq \frac{n}{c}$ such that $r_{\sk}\pts{\cov} \geq bn$, for any set of indices $\gS$ such that $\gS =\{j\mid \sk < j\leq d\}$ and $|\gS|=t \ll n^{\frac{1}{2}}\ll d$, we have $r_{\sk}\pts{\cov_{-s_1:s_t}} \geq bn - t\geq b_2n$.
\end{lemma}
\begin{proof} By the definition of effective ranks, $\cov$ satisfies
    \begin{align*}
        r_{\sk}\pts{\cov} = \frac{\sum_{j=\sk+1}^d\lambda_j}{\lambda_{\sk+1}} \geq bn.
    \end{align*}
    By removing $t$ components whose index is larger than $\sk$, and denote $j^* = \min\{j\mid \sk< j \leq \sk+1+t, j\in\gS^\mathsf{c}\}$, we have
    \begin{align*}
        r_{\sk}\pts{\cov_{-s_1:s_t}} = \frac{\sum_{j=\sk+1, j\in\gS^\mathsf{c}}^d\lambda_j}{\lambda_{j^*}} \geq \frac{\sum_{j=\sk+1, j\in\gS^\mathsf{c}}^d\lambda_j}{\lambda_{\sk+1}} = \frac{\sum_{j=\sk+1}^d\lambda_j}{\lambda_{\sk+1}} - \frac{\sum_{j\in\gS}\lambda_j}{\lambda_{\sk+1}} \geq bn - t,
    \end{align*}
    where the first inequality follows $\lambda_{k*+1} \geq \lambda_{j^*}$ and the last inequality follows the lemma assumption on $\gS$ and $\lambda_{k*+1} \geq \lambda_j$ for all $j\in\gS$. As a result, since $t \ll n^{\frac{1}{2}}$, the proof is complete. 
\end{proof}
Next, we apply Lemma~\ref{lem:balance_extension} to demonstrate that the eigenvalues of the tail of $\cov$, after removing $t \ll n^{\frac{1}{2}}$ components, remain identical up to a constant factor.
\begin{lemma}\label{lem:eig_constant}
    For any data covariance matrix $\cov$ satisfying $\sk \leq \frac{n}{c}$ such that $r_{\sk}\pts{\cov} \geq bn$, for any set of indices $\gS$ such that $\gS =\{j\mid \sk < j\leq d\}$ and $|\gS|=t \ll n^{\frac{1}{2}-\epsilon}\ll d$, we have $r_0\pts{\cov_{-[\sk]\cup\pts{s_1:s_t}}} = r_{\sk}\pts{\cov_{-s_1:s_t}} \geq bn$. Therefore, we have
    \begin{align*}
        \frac{1}{cn} \leq \frac{\norm{\rmA_{-[\sk]\cup\pts{s_1:s_t}}^{-1}}_2}{\Tr\pts{\rmA_{-[\sk]\cup\pts{s_1:s_t}}^{-1}}} \leq \frac{c}{n},
    \end{align*}
    for $c \geq 1$ with probability at least $1-2e^{-\frac{n}{\sqrt{c}}}$.
\end{lemma}
\begin{proof} According to Lemma~\ref{lem:balance_extension}, we have $r_{\sk}\pts{\cov_{-s_1:s_t}} = \frac{\sum_{j=\sk+1,j\in\gS^\mathsf{c}}^d\lambda_j}{\lambda_{j^*}} \geq bn$, where we denote $j^* = \min\{j\mid \sk< j \leq \sk+1+t, j\in\gS^\mathsf{c}\}$. Furthermore, by re-indexing eigenvalues, we denote $\{\tlambda_j\}_{j=1}^{d - \sk - t}$ the eigenvalues of the leave-$\sk$ and $t$-out covariance matrix $\cov_{-\pts{s_1:s_t}\cup[\sk]}$, and we have $r_0\pts{\cov_{-\pts{s_1:s_t}\cup[\sk]}} = \frac{\sum_{j=1}^{d-\sk-t}\tlambda_j}{\tlambda_1} = r_{\sk}\pts{\cov_{-s_1:s_t}} \geq bn$. Based on Lemma 10 in~\cite{bartlett2020benign}, for $\cov_{-\pts{s_1:s_t}\cup[\sk]}$, we have
    \begin{align*}
        \frac{1}{c}\tlambda_1 r_0\pts{\cov_{-\pts{s_1:s_t}\cup[\sk]}} \leq \mu_n\pts{\rmA_{-\pts{s_1:s_t}\cup[\sk]}} \leq \mu_1\pts{\rmA_{-\pts{s_1:s_t}\cup[\sk]}} \leq c\tlambda_1 r_0\pts{\cov_{-\pts{s_1:s_t}\cup[\sk]}},
    \end{align*}
    with probability at least $1-2e^{-\frac{n}{c}}$. Therefore, we have the bounds for $\norm{\rmA_{-[\sk]\cup\pts{s_1:s_t}}^{-1}}_2$ as
    \begin{align*}
        \norm{\rmA_{-[\sk]\cup\pts{s_1:s_t}}^{-1}}_2 &= \frac{1}{\mu_n\pts{\rmA_{-[\sk]\cup\pts{s_1:s_t}}}} \leq \frac{c}{\tlambda_1 r_0\pts{\cov_{-\pts{s_1:s_t}\cup[\sk]}}}\\
        \norm{\rmA_{-[\sk]\cup\pts{s_1:s_t}}^{-1}}_2 &= \frac{1}{\mu_n\pts{\rmA_{-[\sk]\cup\pts{s_1:s_t}}}}\geq \frac{1}{\mu_1\pts{\rmA_{-[\sk]\cup\pts{s_1:s_t}}}} \geq \frac{1}{c\tlambda_1 r_0\pts{\cov_{-\pts{s_1:s_t}\cup[\sk]}}}.
    \end{align*}
    Similarly, for $\Tr\pts{\rmA_{-[\sk]\cup\pts{s_1:s_t}}^{-1}}$, we have
    \begin{align*}
        \Tr\pts{\rmA_{-[\sk]\cup\pts{s_1:s_t}}^{-1}} &= \sum_{i=1}^n\frac{1}{\mu_i\pts{\rmA_{-[\sk]\cup\pts{s_1:s_t}}}} \leq \frac{n}{\mu_n\pts{\rmA_{-[\sk]\cup\pts{s_1:s_t}}}}\leq \frac{cn}{\tlambda_1 r_0\pts{\cov_{-\pts{s_1:s_t}\cup[\sk]}}}\\
        \Tr\pts{\rmA_{-[\sk]\cup\pts{s_1:s_t}}^{-1}} &= \sum_{i=1}^n\frac{1}{\mu_i\pts{\rmA_{-[\sk]\cup\pts{s_1:s_t}}}}\geq \frac{n}{\mu_1\pts{\rmA_{-[\sk]\cup\pts{s_1:s_t}}}}\geq \frac{n}{c\tlambda_1 r_0\pts{\cov_{-\pts{s_1:s_t}\cup[\sk]}}}.
    \end{align*}
    By substituting these bounds into $\frac{\norm{\rmA_{-[\sk]\cup\pts{s_1:s_t}}^{-1}}_2}{\Tr\pts{\rmA_{-[\sk]\cup\pts{s_1:s_t}}^{-1}}}$, the proof is done.
\end{proof}
The following lemma extends Lemma 25 from~\cite{wang2023benign} to show the trace bounds when removing distinct components in $\cov$, whereas the original result only provided the lower bound for removing the top $k$ components.
\begin{lemma}[From Lemma 25 in~\cite{wang2023benign}]\label{lem:trace_lower_bound} For any data covariance matrix $\cov$ satisfying $\sk \leq \frac{n}{c}$ such that $r_{\sk}\pts{\cov} \geq bn$, for any set of indices $\gS$ such that $\gS =\{j\mid \sk < j\leq d\}$ and $|\gS|=t \ll n^{\frac{1}{2}}\ll d$, for any $0 \leq \ell_1 \leq \ell_2 \leq t$ and $0 \leq k_1 \leq k_2 \leq \sk$ and sufficiently large $n$, we have
    \begin{align*}
        \Tr\pts{\rmA_{-\pts{s_1:s_{\ell_2}}\cup[k_2]}^{-1}} \geq \Tr\pts{\rmA_{-\pts{s_1:s_{\ell_1}}\cup[k_1]}^{-1}} \geq \pts{1-\frac{c}{n}}^{k_2-k_1 +\ell_2 -\ell_1}\Tr\pts{\rmA_{-\pts{s_1:s_{\ell_2}}\cup[k_2]}^{-1}},
    \end{align*}
    with a probability at least $1-2\pts{k_2-k_1+\ell_2-\ell_1}e^{-\frac{n}{c}}$. 
\end{lemma}
\begin{proof} For the first inequality, it directly holds since $\rmA_{-\pts{s_1:s_{\ell_1}}\cup[k_1]} \succeq \rmA_{-\pts{s_1:s_{\ell_2}}\cup[k_2]}$ implies $\rmA_{-\pts{s_1:s_{\ell_2}}\cup[k_2]}^{-1} \succeq \rmA_{-\pts{s_1:s_{\ell_1}}\cup[k_1]}^{-1}$. Next, according to Lemma 25 in~\cite{wang2023benign}, by removing the top component in $\cov$, we have
    \begin{align*}
        \Tr\pts{\rmA_{-\pts{s_1:s_{\ell_1}}\cup[k_1]}^{-1}} \geq \pts{1-\frac{c}{n}}^{k_2-k_1}\Tr\pts{\rmA_{-\pts{s_1:s_{\ell_1}}\cup[k_2]}^{-1}}.
    \end{align*}
    Next, following the proof steps in Lemma 25 in~\cite{wang2023benign}, by the Sherman-Morrison-Woodbury identity, we have
    \begin{align}
        \Tr\pts{\rmA_{-\pts{s_1:s_{\ell_1}}\cup[k_2]}^{-1}} &= \Tr\pts{\rmA_{-\pts{s_1:s_{\ell_1+1}}\cup[k_2]}^{-1}} - \frac{\lambda_{s_{\ell_1+1}}\Tr\pts{\rmA_{-\pts{s_1:s_{\ell_1+1}}\cup[k_2]}^{-1}\rvz_{s_{\ell_1+1}}\rvz_{s_{\ell_1+1}}^\top\rmA_{-\pts{s_1:s_{\ell_1+1}}\cup[k_2]}^{-1}}}{1+\lambda_{s_{\ell_1+1}}\rvz_{s_{\ell_1+1}}^\top\rmA_{-\pts{s_1:s_{\ell_1+1}}\cup[k_2]}^{-1}\rvz_{s_{\ell_1+1}}}\nonumber\\
        &= \Tr\pts{\rmA_{-\pts{s_1:s_{\ell_1+1}}\cup[k_2]}^{-1}} - \frac{\lambda_{s_{\ell_1+1}}\rvz_{s_{\ell_1+1}}^\top\rmA_{-\pts{s_1:s_{\ell_1+1}}\cup[k_2]}^{-2}\rvz_{s_{\ell_1+1}}}{1+\lambda_{s_{\ell_1+1}}\rvz_{s_{\ell_1+1}}^\top\rmA_{-\pts{s_1:s_{\ell_1+1}}\cup[k_2]}^{-1}\rvz_{s_{\ell_1+1}}}\nonumber\\
        &\geq \Tr\pts{\rmA_{-\pts{s_1:s_{\ell_1+1}}\cup[k_2]}^{-1}} - \norm{\rmA_{-\pts{s_1:s_{\ell_1+1}}\cup[k_2]}^{-1}}_2 \cdot \frac{\lambda_{s_{\ell_1+1}}\rvz_{s_{\ell_1+1}}^\top\rmA_{-\pts{s_1:s_{\ell_1+1}}\cup[k_2]}^{-1}\rvz_{s_{\ell_1+1}}}{1+\lambda_{s_{\ell_1+1}}\rvz_{s_{\ell_1+1}}^\top\rmA_{-\pts{s_1:s_{\ell_1+1}}\cup[k_2]}^{-1}\rvz_{s_{\ell_1+1}}}\nonumber\\
        &\geq \Tr\pts{\rmA_{-\pts{s_1:s_{\ell_1+1}}\cup[k_2]}^{-1}} - \norm{\rmA_{-\pts{s_1:s_{\ell_1+1}}\cup[k_2]}^{-1}}_2,\label{eq:trace_lower_bound}
    \end{align}
    where in the first inequality, we use the property that $\rvx^\top\rmM^2\rvx \leq \norm{\rmM}_2\cdot\rvx^\top\rmM\rvx$. Next, we have
    \begin{align*}
        \frac{\norm{\rmA_{-\pts{s_1:s_{\ell_1+1}}\cup[k_2]}^{-1}}_2}{\Tr\pts{\rmA_{-\pts{s_1:s_{\ell_1+1}}\cup[k_2]}^{-1}}} \leq \frac{\norm{\rmA_{-\pts{s_1:s_{\ell_1+1}}\cup[\sk]}^{-1}}_2}{\Tr\pts{\rmA_{-\pts{s_1:s_{\ell_1+1}}\cup[k_2]}^{-1}}} \leq \frac{\norm{\rmA_{-\pts{s_1:s_{\ell_1+1}}\cup[\sk]}^{-1}}_2}{\pts{1-\frac{c}{n}}^{\sk-k_2}\Tr\pts{\rmA_{-\pts{s_1:s_{\ell_1+1}}\cup[\sk]}^{-1}}} \leq \frac{c}{\pts{1-\frac{c}{n}}^{\sk-k_2}n},
    \end{align*}
    where in the first inequality, we apply the property that $\rmA_{-\pts{s_1:s_{\ell_1+1}}\cup[k_2]}^{-1} \preceq \rmA_{-\pts{s_1:s_{\ell_1+1}}\cup[\sk]}^{-1}$, in the second inequality, we use Lemma 25 in~\cite{wang2023benign}, and in the last inequality, we apply Lemma~\ref{lem:eig_constant}. As a result, from~\Eqref{eq:trace_lower_bound}, by dividing $\Tr\pts{\rmA_{-\pts{s_1:s_{\ell_1+1}}\cup[k_2]}^{-1}}$ on both sides, we have
    \begin{align*}
        \frac{\Tr\pts{\rmA_{-\pts{s_1:s_{\ell_1}}\cup[k_2]}^{-1}}}{\Tr\pts{\rmA_{-\pts{s_1:s_{\ell_1+1}}\cup[k_2]}^{-1}}} \geq 1 - \frac{\norm{\rmA_{-\pts{s_1:s_{\ell_1+1}}\cup[k_2]}^{-1}}_2}{\Tr\pts{\rmA_{-\pts{s_1:s_{\ell_1+1}}\cup[k_2]}^{-1}}}\geq 1 - \frac{c}{\pts{1-\frac{c}{n}}^{\sk-k_2}n} \geq 1 -\frac{c_1}{n},
    \end{align*}
    By applying the steps $\ell_2-\ell_1$ times, the proof is complete.
\end{proof}

\section{Additional simulations} \label{app:simulation}
In this section, we present additional simulations of our few-shot postprocessing algorithm. In general, we find that our postprocessing algorithm can recover $t$-sparse signal in cases where the classification and regression tasks fail --- even including the worst-case scenario of isotropic covariance.

Our simulations were run on an Nvidia A5000 GPU with 24GB VRAM, but this level of compute is not necessary. Our code is available at \url{https://github.com/tmlabonte/taskshift}.
\begin{figure*}[h]
    \centering
    \begin{subfigure}[b]{0.32\textwidth}
        \includegraphics[scale=0.33]{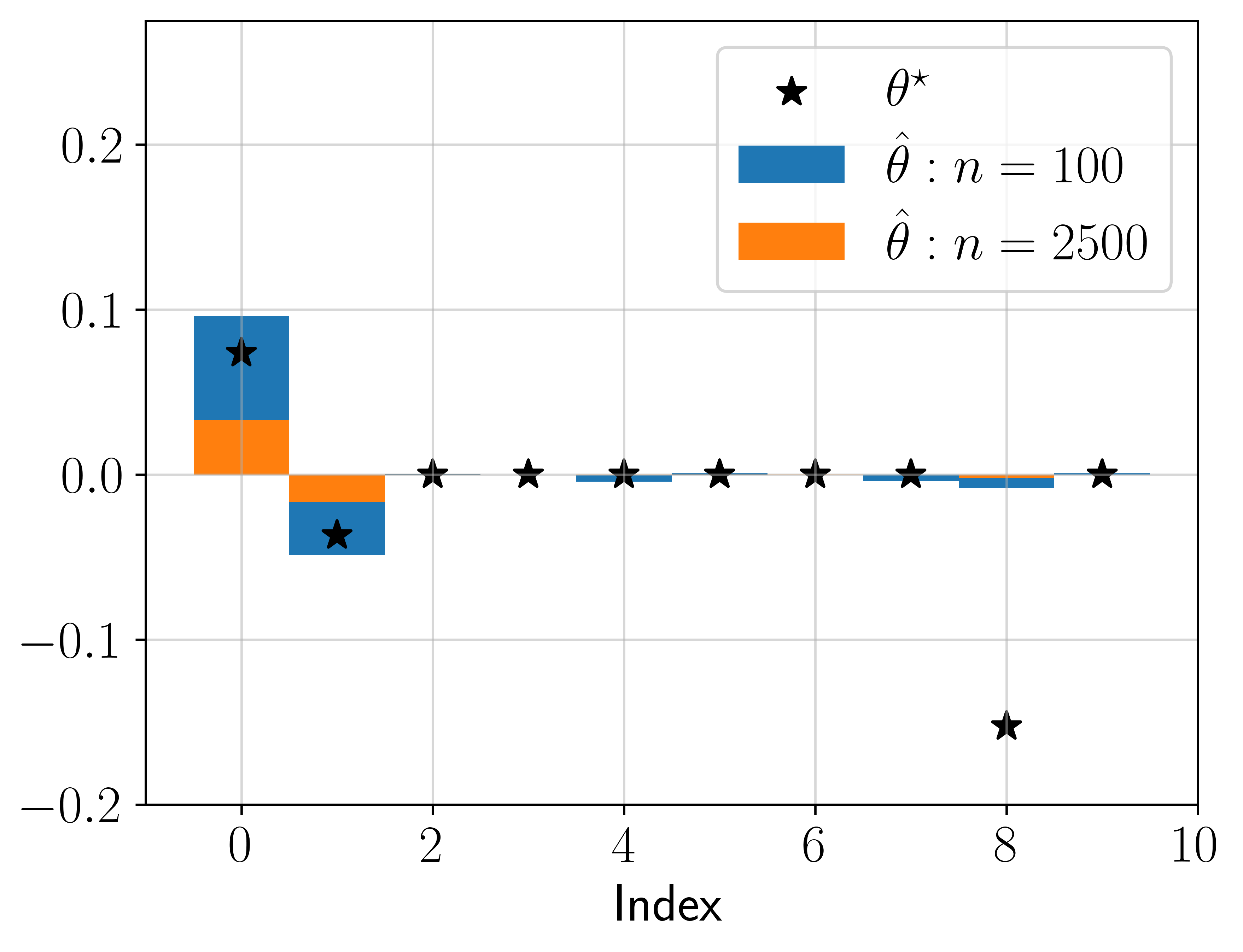}
        \subcaption{Support recovery}
    \end{subfigure}
    \hfill
    \begin{subfigure}[b]{0.32\textwidth}
        \includegraphics[scale=0.33]{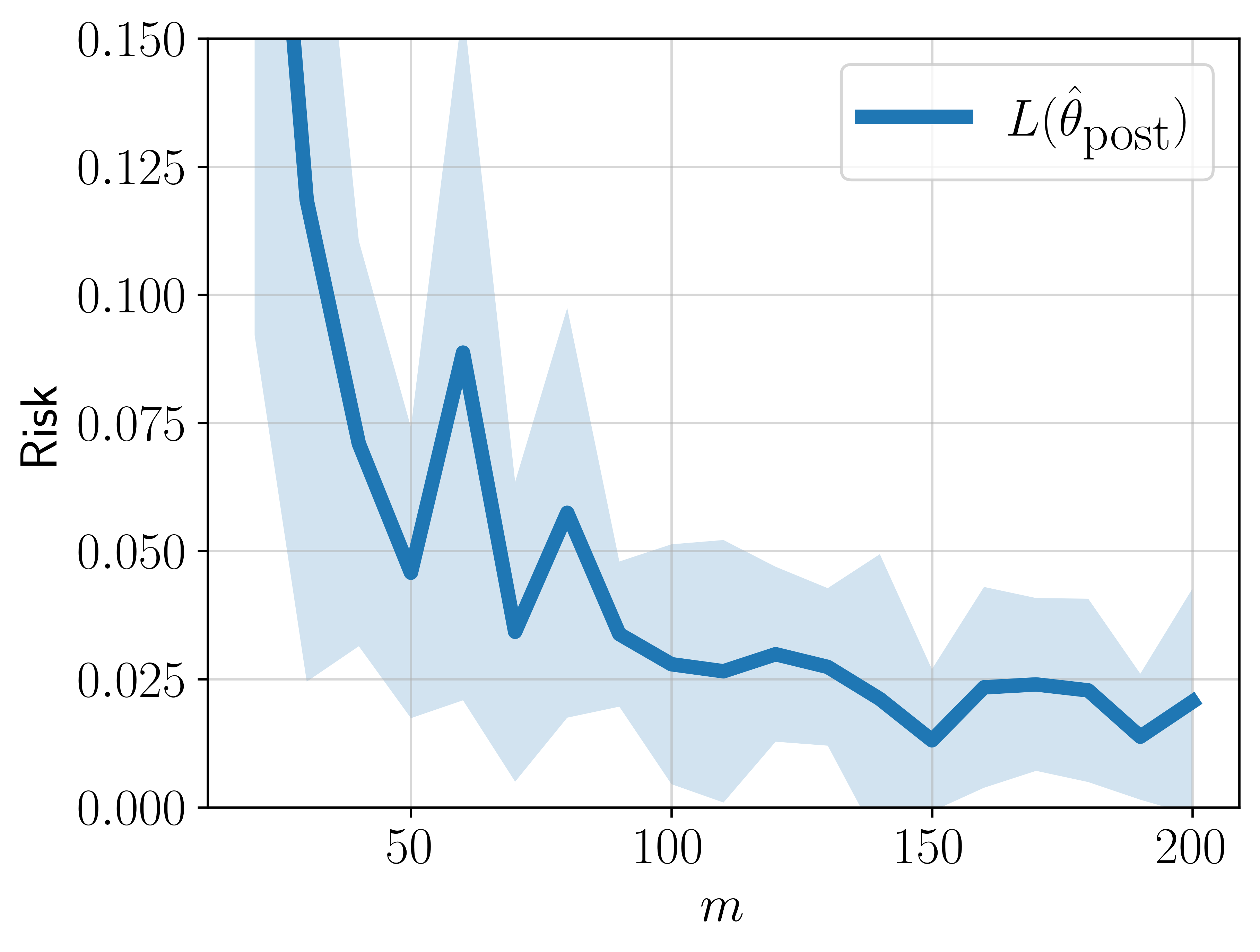}
        \subcaption{Few-shot postprocessing}
    \end{subfigure}
    \hfill
    \begin{subfigure}[b]{0.32\textwidth}
        \includegraphics[scale=0.33]{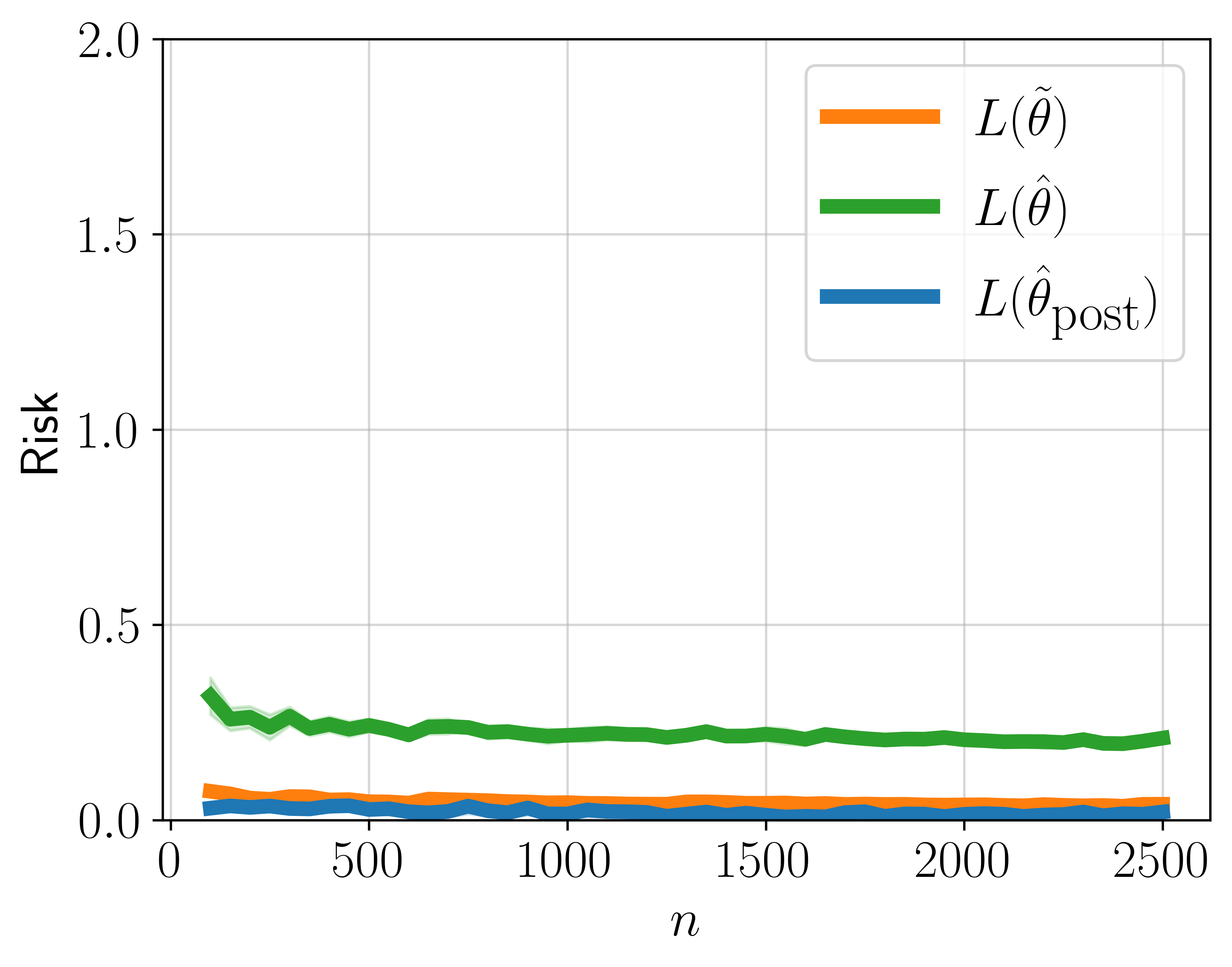}
        \subcaption{Risk curves}
    \end{subfigure}
    \hfill
    \caption{\textbf{Task shift for spiked covariance with $q < 1-r$ and signal outside the spike.} We set $p=1.5$, $q=0.5$, and $r=0.25$ so that $q < 1-r$. Moreover, we add an additional signal component which lies outside the covariance spike for any $n\leq 2500$.  Our task shift algorithm correctly recovers the support and generalizes well; note that the decay of the component outside the spike (index 8) is faster than those in the spike (indices 1-2), but still slower than those outside the support (indices 3-7 and 9-10). The true signal $\svtheta$ is $3$-sparse with $a_1=1$, $a_2=-0.5$, and $a_8=-0.15$ (see Assumption~\ref{asm:k_sparse}). Our postprocessing algorithm uses top-$t$ support recovery and least-squares on noisy $m$-shot regression data. We plot the mean and standard deviation over $10$ draws of the training dataset $\rmX$.}
    \label{fig:spiked_outside}
\end{figure*}

\begin{figure*}[h]
    \centering
    \begin{subfigure}[b]{0.32\textwidth}
        \includegraphics[scale=0.33]{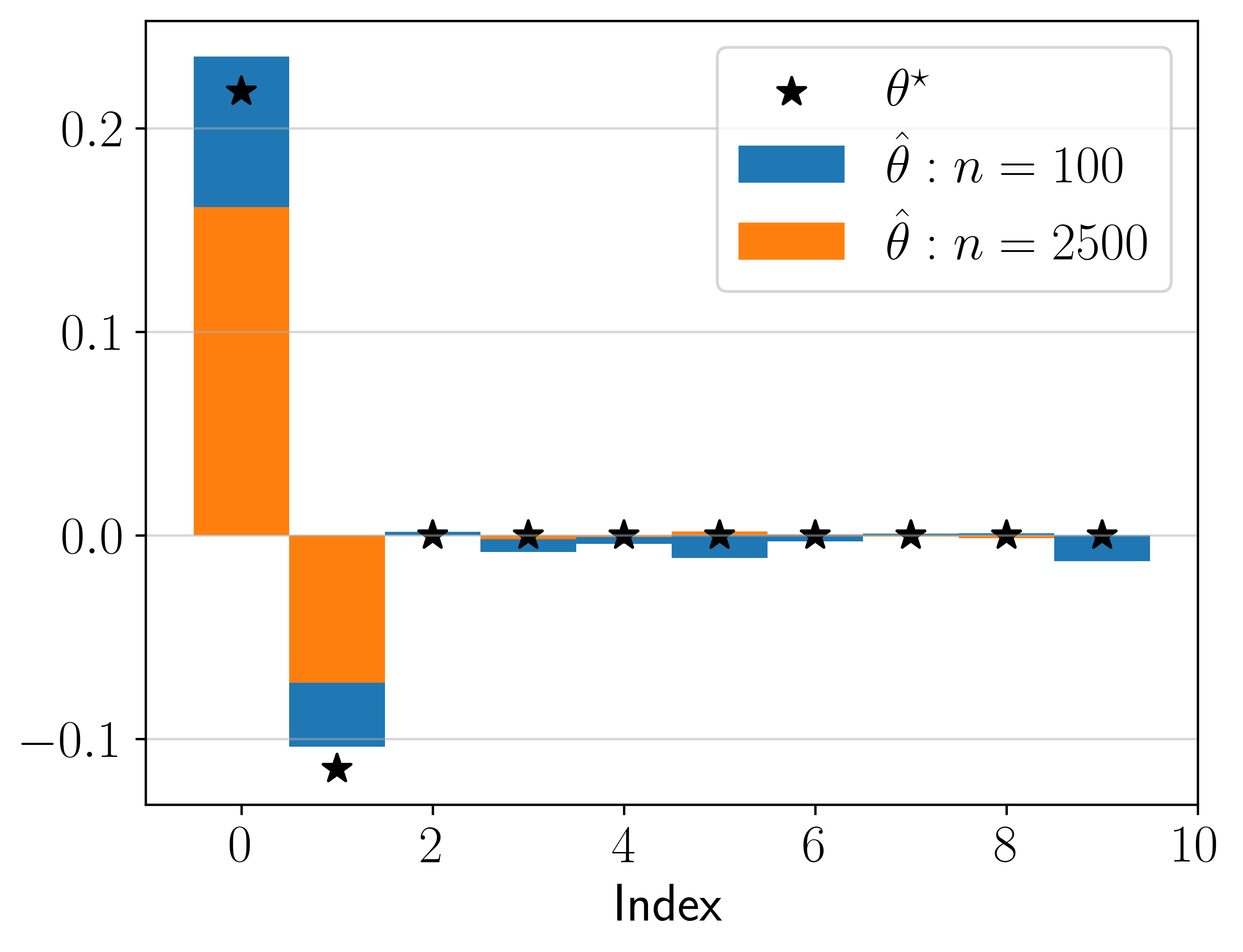}
        \subcaption{Support recovery}
    \end{subfigure}
    \hfill
    \begin{subfigure}[b]{0.32\textwidth}
        \includegraphics[scale=0.33]{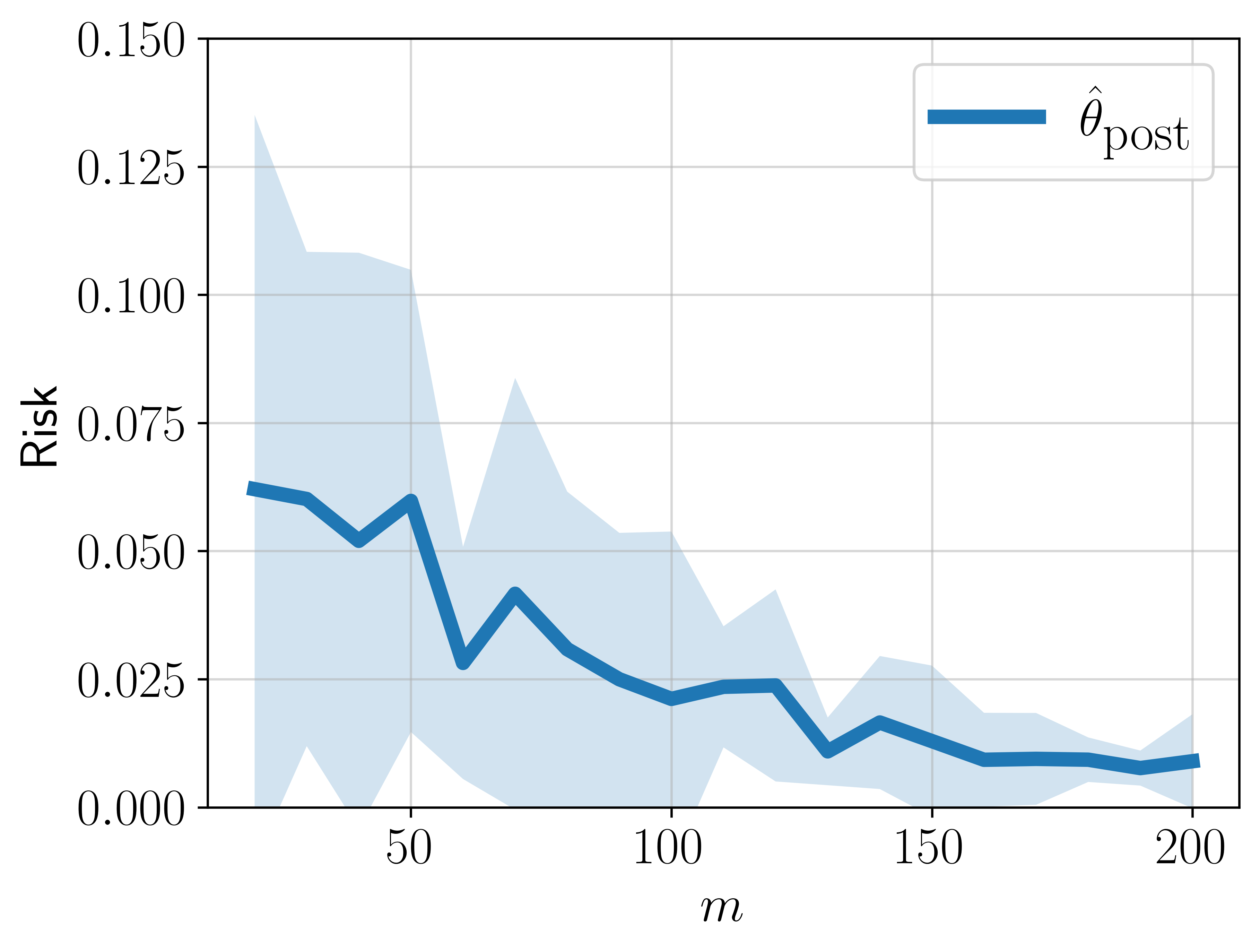}
        \subcaption{Few-shot postprocessing}
    \end{subfigure}
    \hfill
    \begin{subfigure}[b]{0.32\textwidth}
        \includegraphics[scale=0.33]{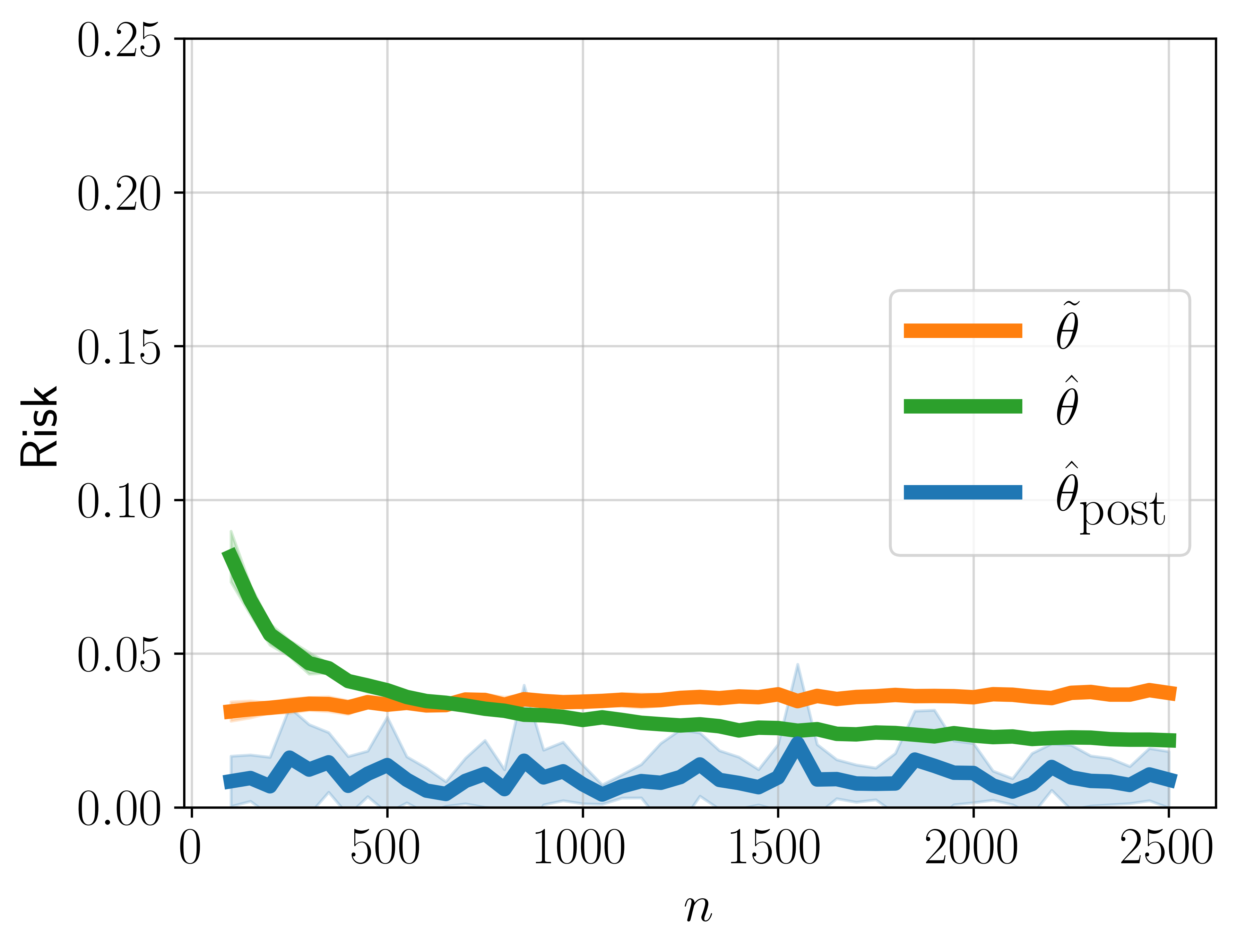}
        \subcaption{Risk curves}
    \end{subfigure}
    \hfill
    \caption{\textbf{Task shift for polynomial covariance with $u=0.25,v=0$}. Our task shift estimator generalizes for polynomial covariance models. The true signal $\svtheta$ is $2$-sparse with $a_1=0.2$ and $a_2=-0.1$ (see Assumption~\ref{asm:k_sparse}), and we set $d=n^{1.5}$. Note that this parameterization satisfies the conditions of Corollary~\ref{cor:poly_support_identification2}. Our postprocessing algorithm uses top-$t$ support recovery and least-squares on noisy $m$-shot regression data. We plot the mean and standard deviation over $10$ draws of the training dataset $\rmX$.}
    \label{fig:poly}
\end{figure*}

\begin{figure*}[h]
    \centering
    \begin{subfigure}[b]{0.32\textwidth}
        \includegraphics[scale=0.33]{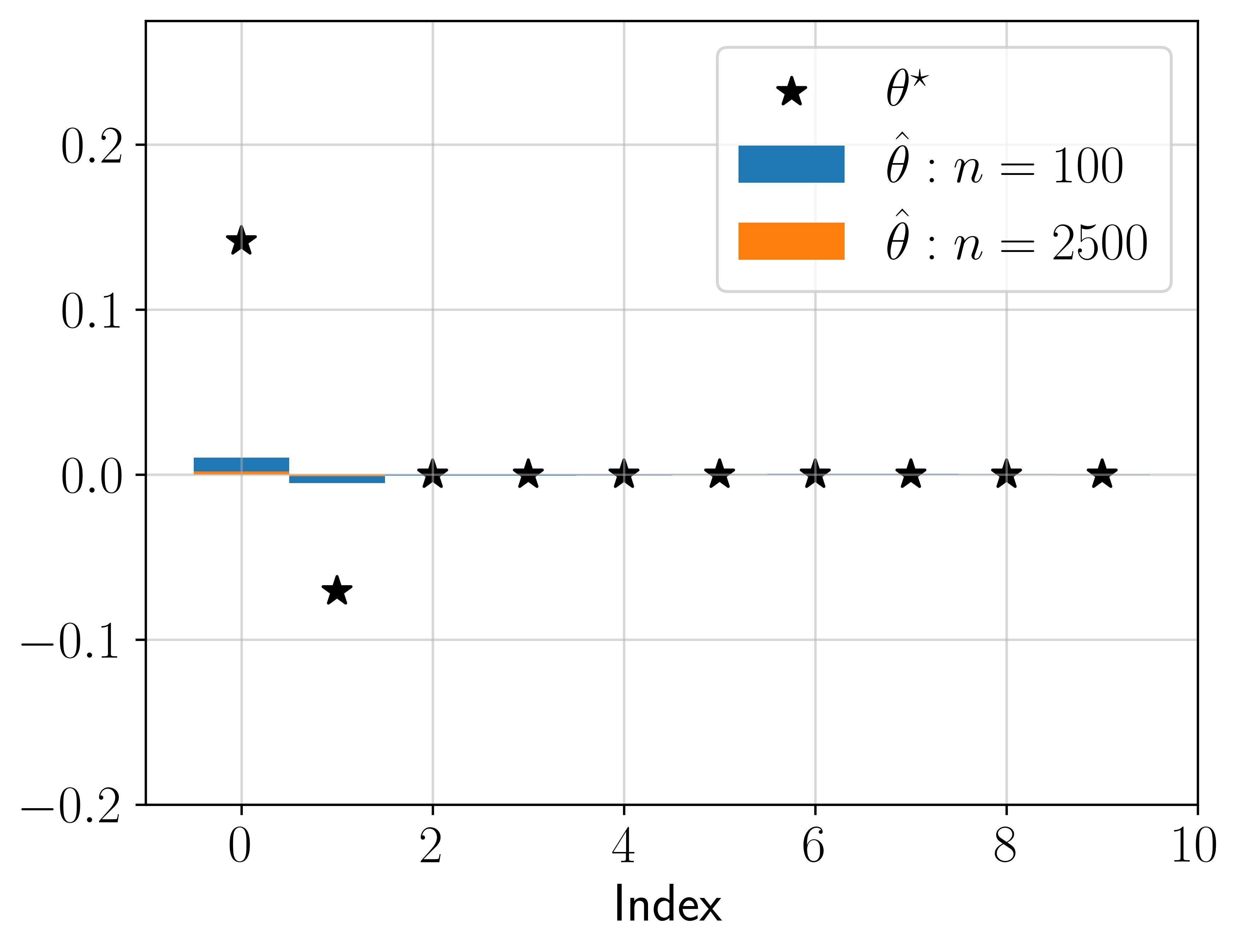}
        \subcaption{Support recovery}
    \end{subfigure}
    \hfill
    \begin{subfigure}[b]{0.32\textwidth}
        \includegraphics[scale=0.33]{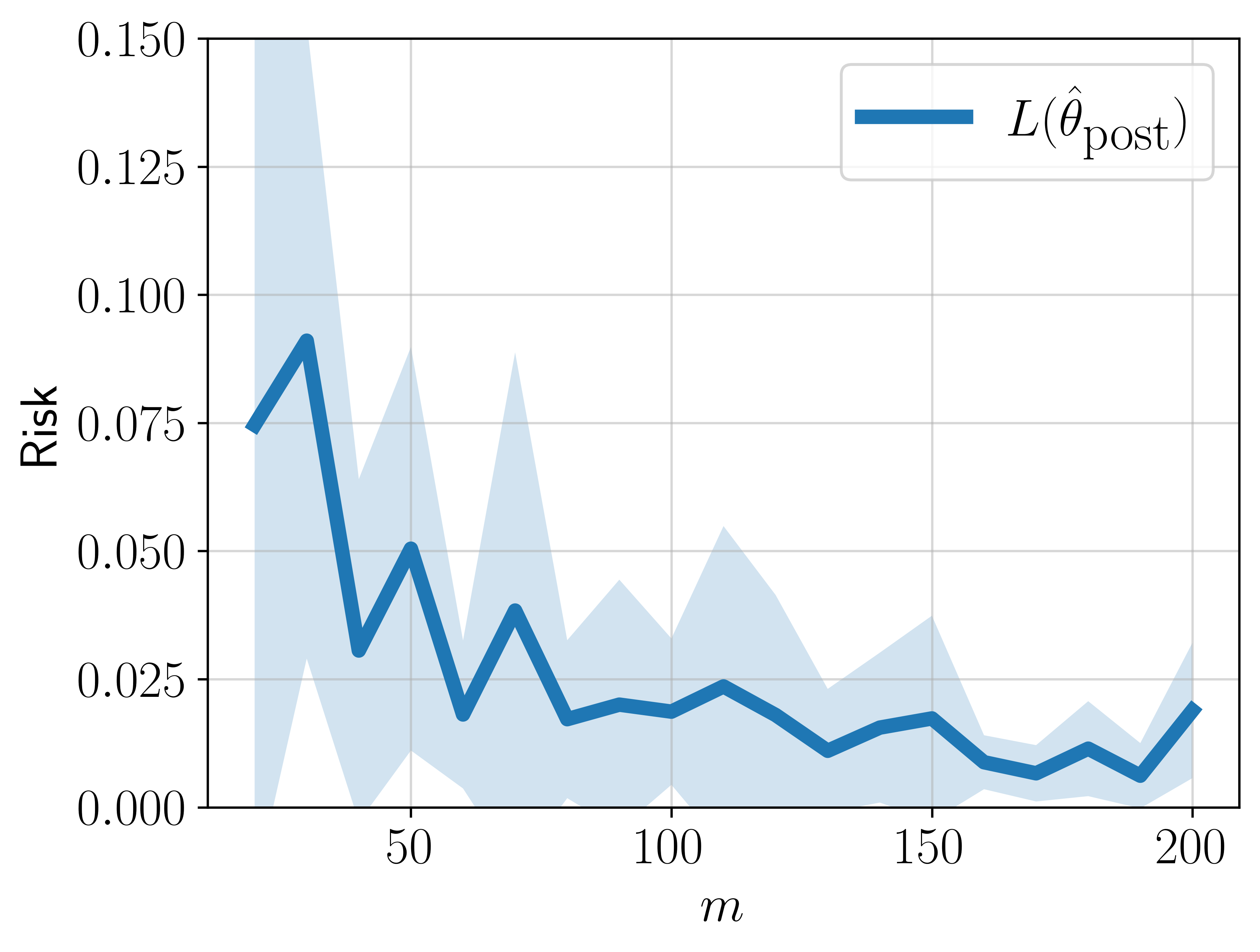}
        \subcaption{Few-shot postprocessing}
    \end{subfigure}
    \hfill
    \begin{subfigure}[b]{0.32\textwidth}
        \includegraphics[scale=0.33]{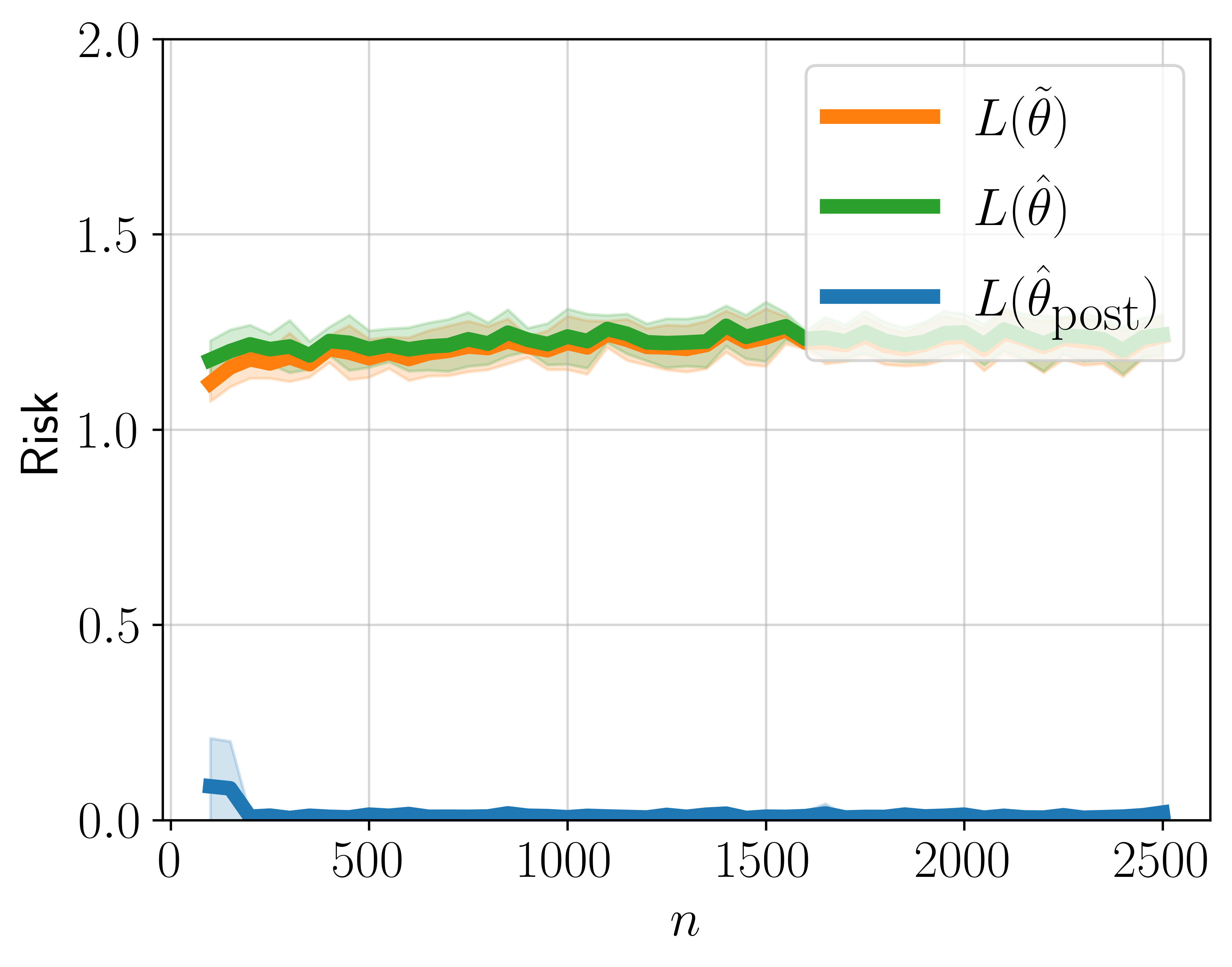}
        \subcaption{Risk curves}
    \end{subfigure}
    \hfill
    \caption{\textbf{Task shift for isotropic covariance $\cov=50\mI$}. Our task shift estimator generalizes even in worst-case scenarios for minimum $\ell_2$-norm interpolation such as isotropic covariance. The true signal $\svtheta$ is $2$-sparse with $a_1=1$ and $a_2=-0.5$ (see Assumption~\ref{asm:k_sparse}), and we set $d=n^{1.5}$. Our postprocessing algorithm uses top-$t$ support recovery and least-squares on noisy $m$-shot regression data. We plot the mean and standard deviation over $10$ draws of the training dataset $\rmX$.}
    \label{fig:isotropic}
\end{figure*}

\begin{figure*}[ht]
    \begin{subfigure}[b]{\textwidth}
        \centering
        \begin{subfigure}[b]{0.32\textwidth}
            \includegraphics[scale=0.33]{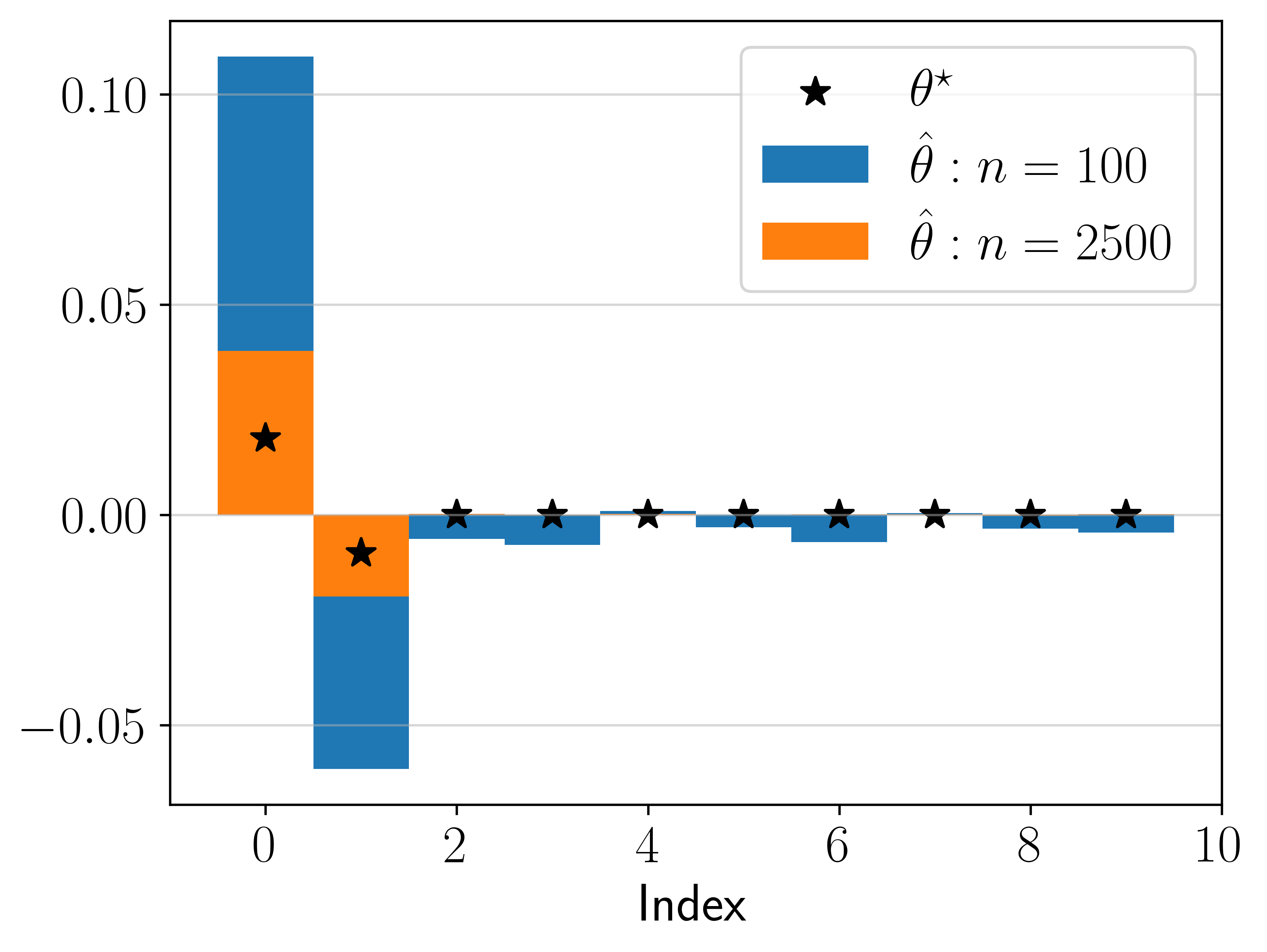}
            \subcaption{Support recovery}
        \end{subfigure}
        \hfill
        \begin{subfigure}[b]{0.32\textwidth}
            \includegraphics[scale=0.33]{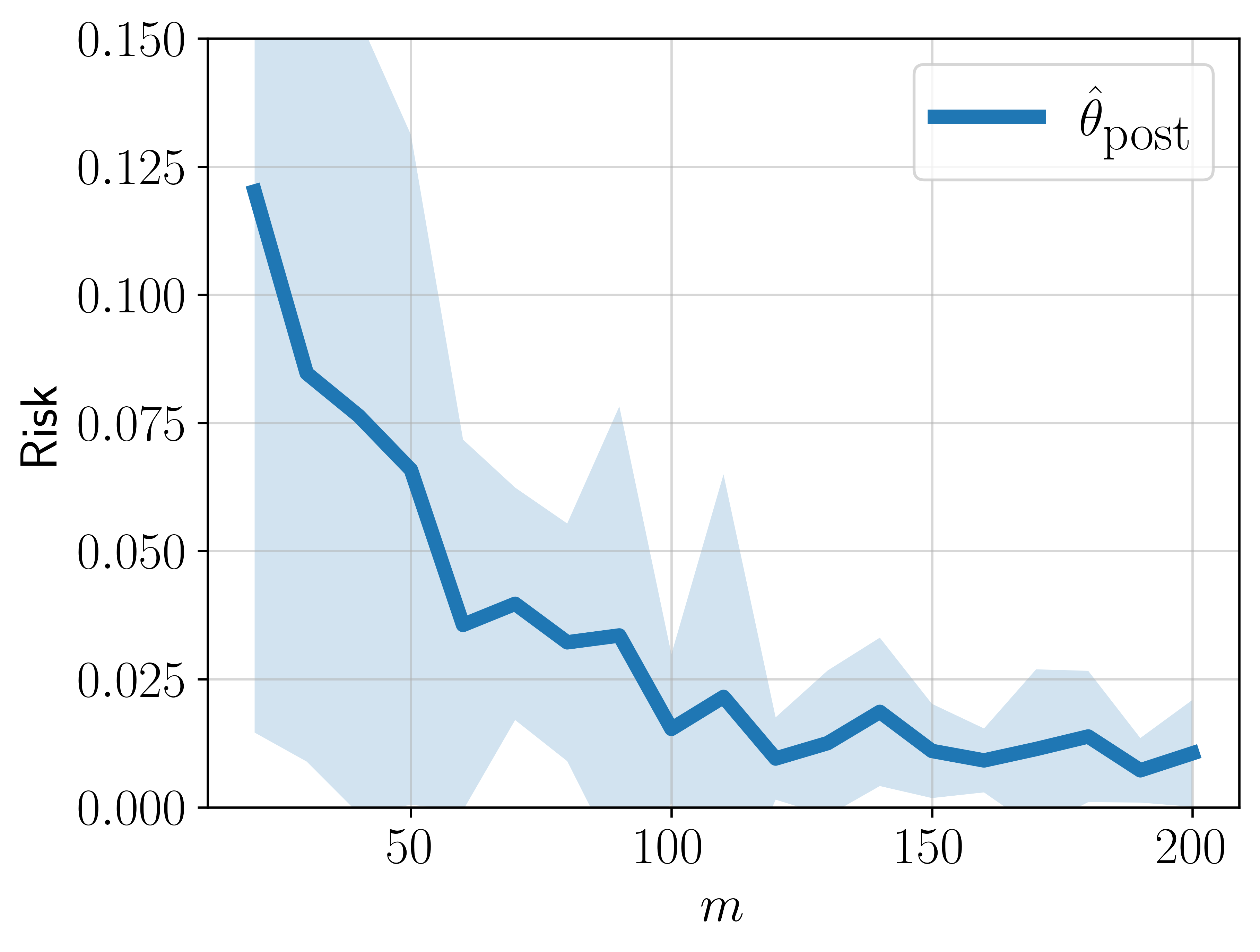}
            \subcaption{Least-squares with dimension reduction}
        \end{subfigure}
        \hfill
        \begin{subfigure}[b]{0.32\textwidth}
            \includegraphics[scale=0.33]{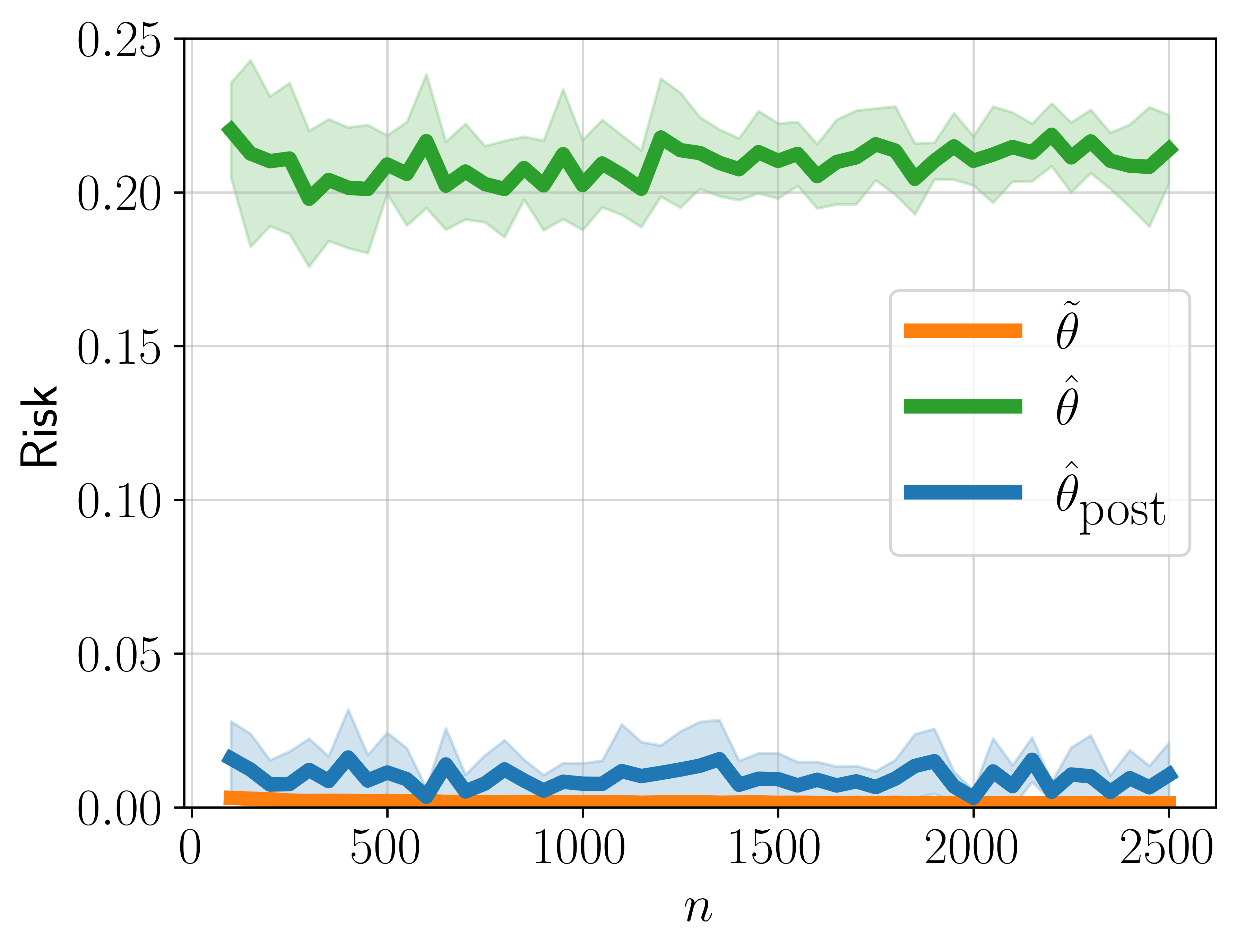}
            \subcaption{Regression risk}
        \end{subfigure}
        \hfill
        \caption*{\textbf{(i) Task shift for spiked covariance when classification and regression generalize.} We set $p=1.5$, $q=0.3$, and $r=0.5$ so that $0<q< 1-r$. In this regime, the classification MNI $\hvtheta$ generalizes on the original classification problem and the regression MNI $\tvtheta$ generalizes on the original regression problems.}
        \label{fig:spiked_q3}
    \end{subfigure}
    \addtocounter{subfigure}{-3}
    \begin{subfigure}[b]{\textwidth}
        \centering
        \begin{subfigure}[b]{0.32\textwidth}
            \includegraphics[scale=0.33]{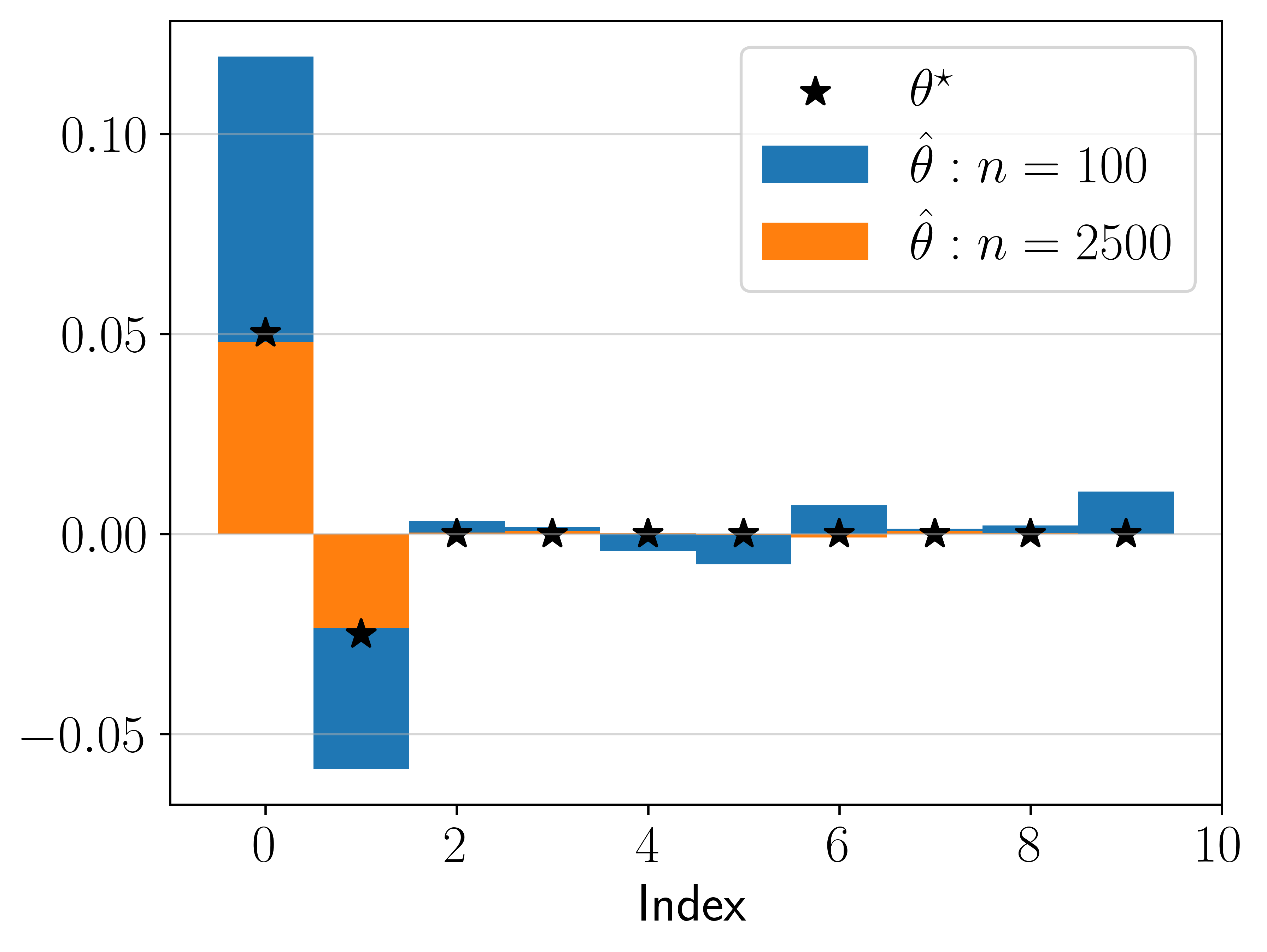}
            \subcaption{Support recovery}
        \end{subfigure}
        \hfill
        \begin{subfigure}[b]{0.32\textwidth}
            \includegraphics[scale=0.33]{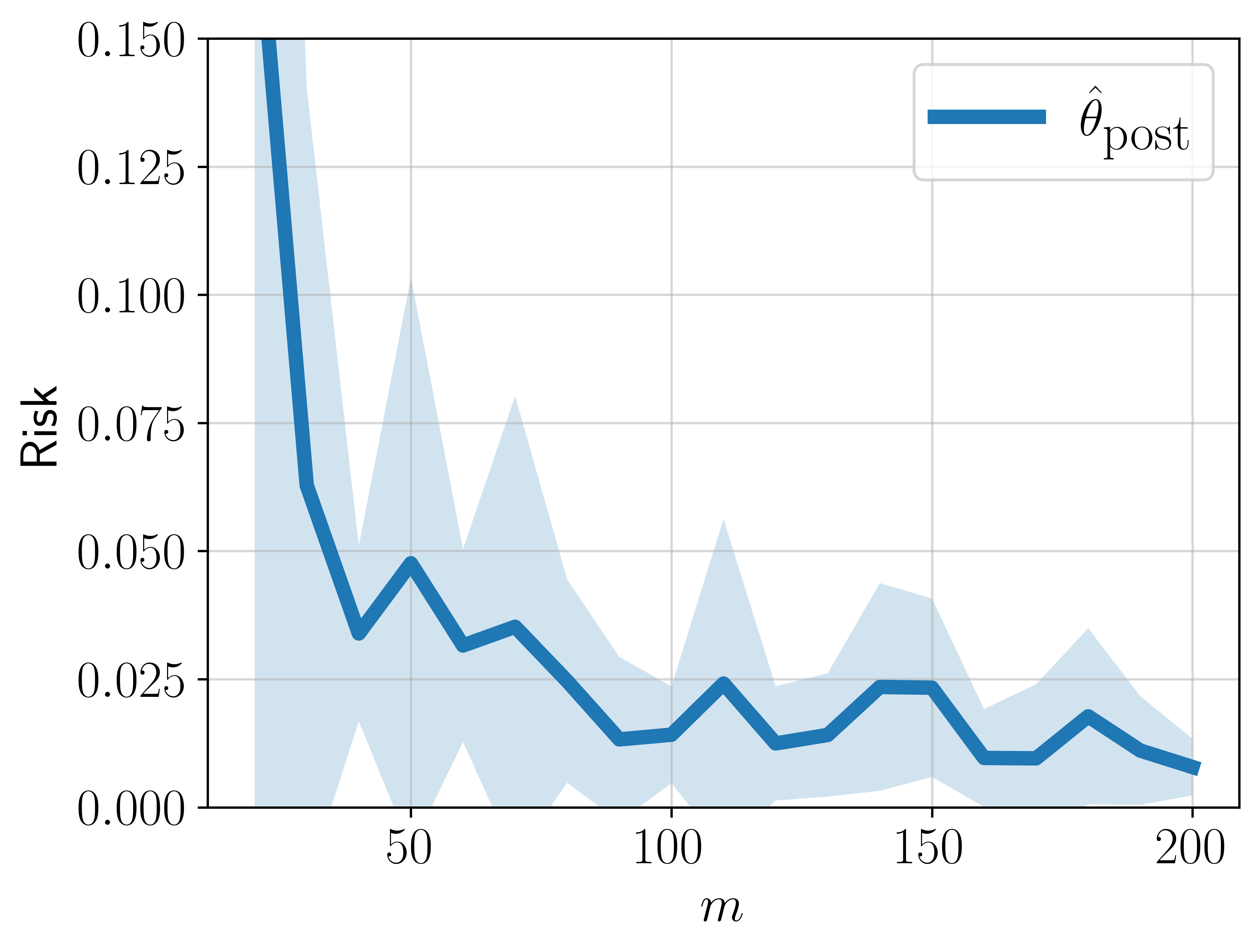}
            \subcaption{Least-squares with dimension reduction}
        \end{subfigure}
        \hfill
        \begin{subfigure}[b]{0.32\textwidth}
            \includegraphics[scale=0.33]{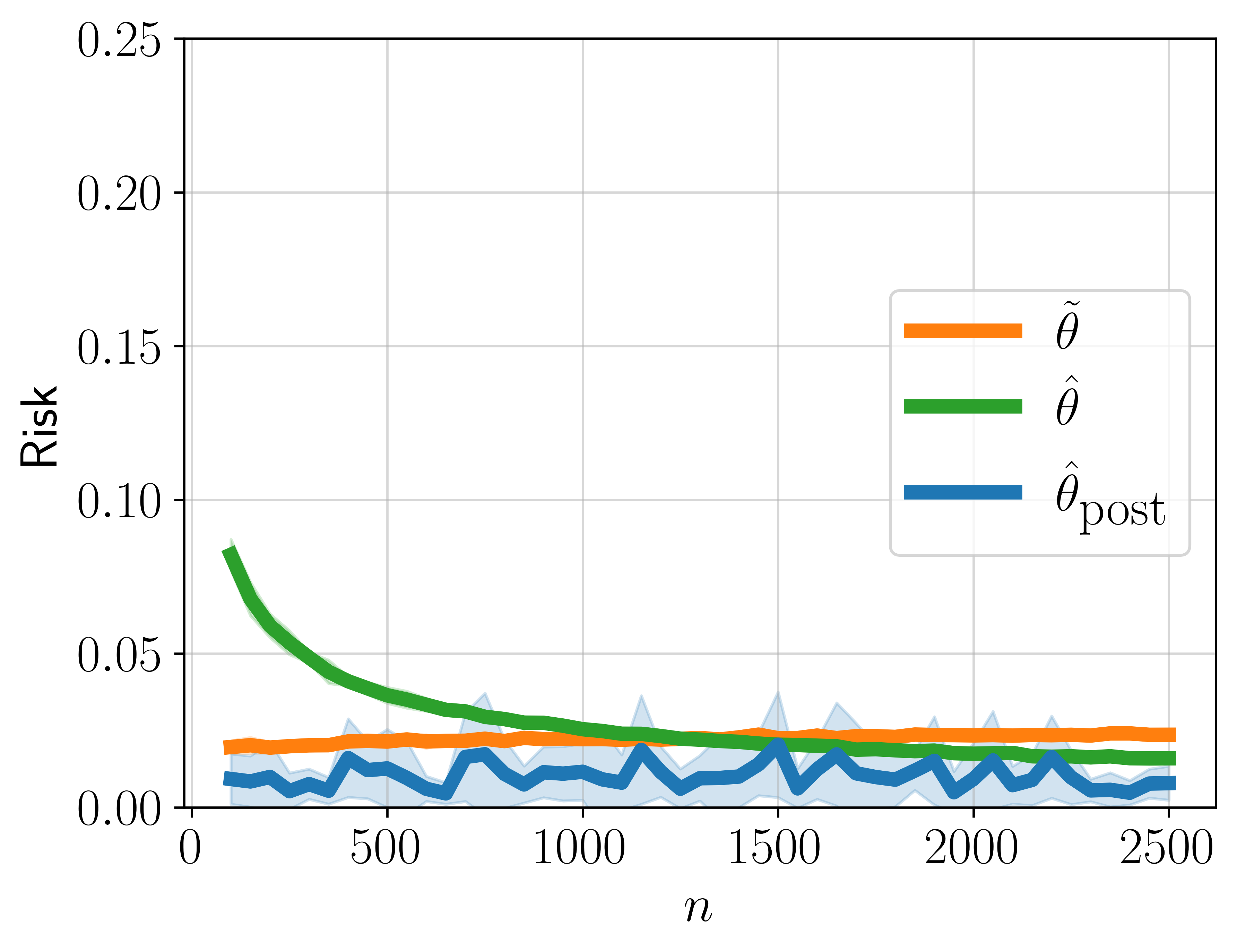}
            \subcaption{Regression risk}
        \end{subfigure}
        \hfill
        \caption*{\textbf{(ii) Task shift for spiked covariance when classification generalizes but regression does not.} We set $p=1.5$, $q=0.6$, and $r=0.5$ so that $1-r<q<(1-r)+(p-1)/2$. In this regime, the classification MNI $\hvtheta$ generalizes on the original classification problem, but the regression MNI $\tvtheta$ does not generalize on the original regression problems.}
        \label{fig:spiked_q6}
    \end{subfigure}
    \begin{subfigure}[b]{\textwidth}
        \centering
        \begin{subfigure}[b]{0.32\textwidth}
            \includegraphics[scale=0.33]{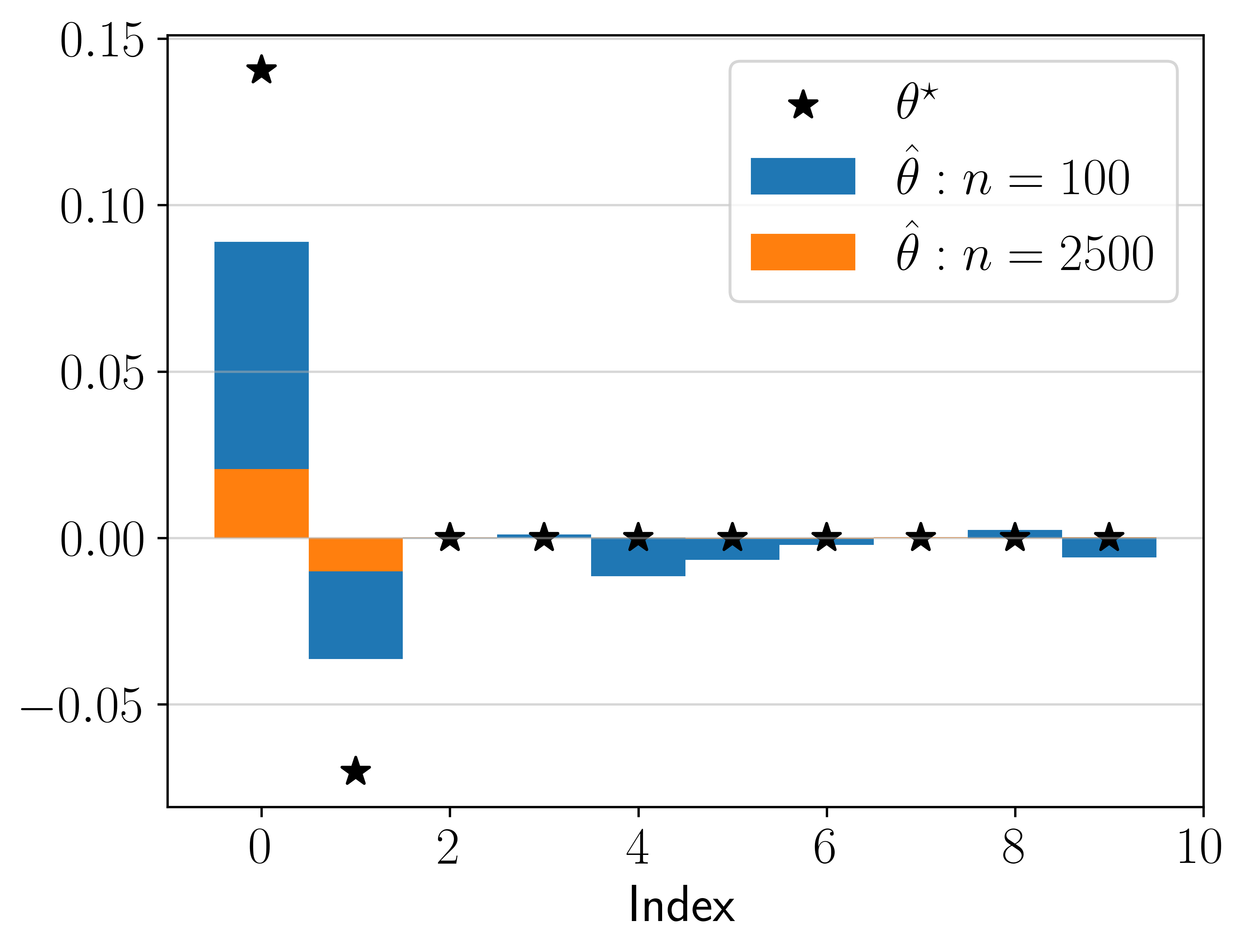}
            \subcaption{Support recovery}
        \end{subfigure}
        \hfill
        \begin{subfigure}[b]{0.32\textwidth}
            \includegraphics[scale=0.33]{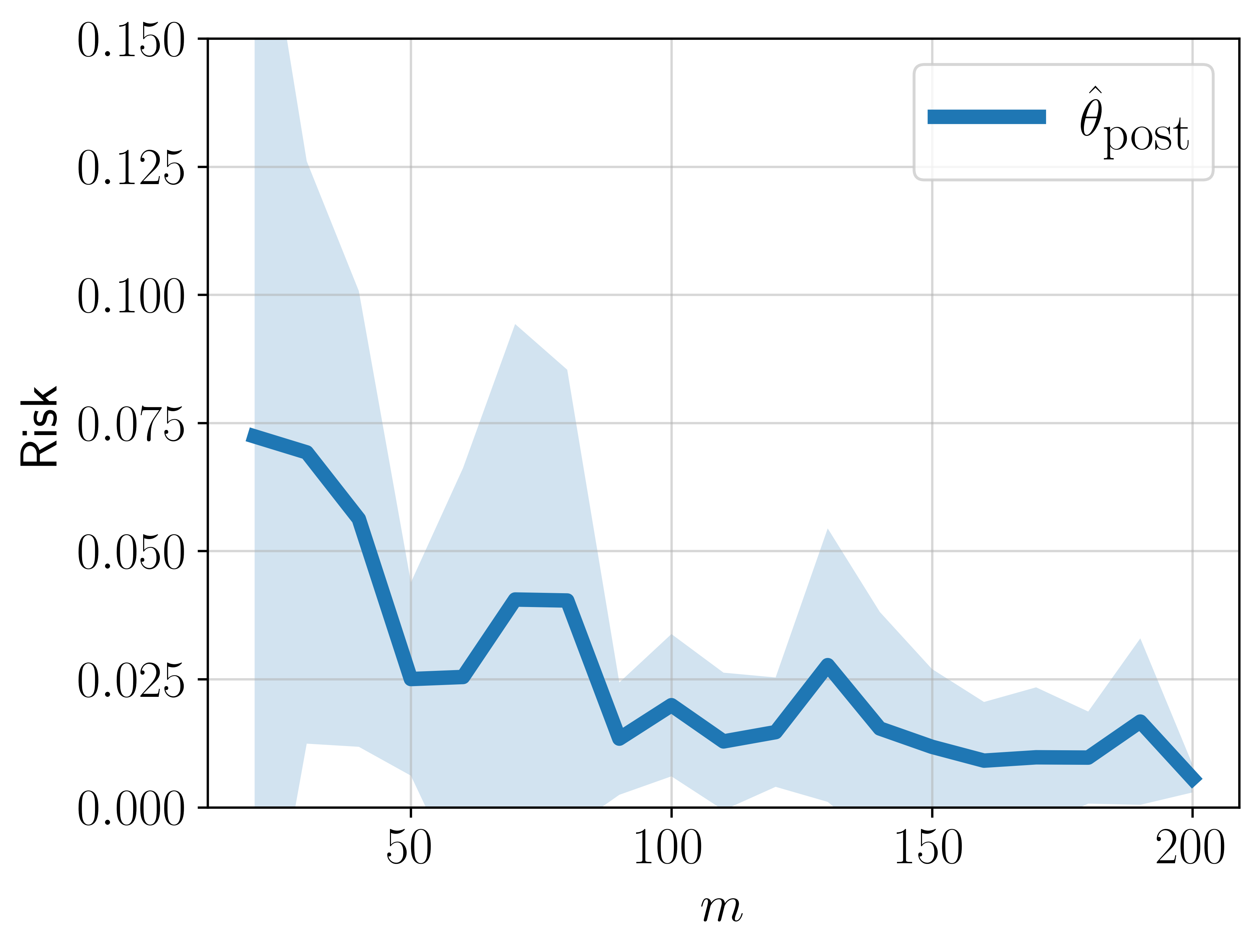}
            \subcaption{Least-squares with dimension reduction}
        \end{subfigure}
        \hfill
        \begin{subfigure}[b]{0.32\textwidth}
            \includegraphics[scale=0.33]{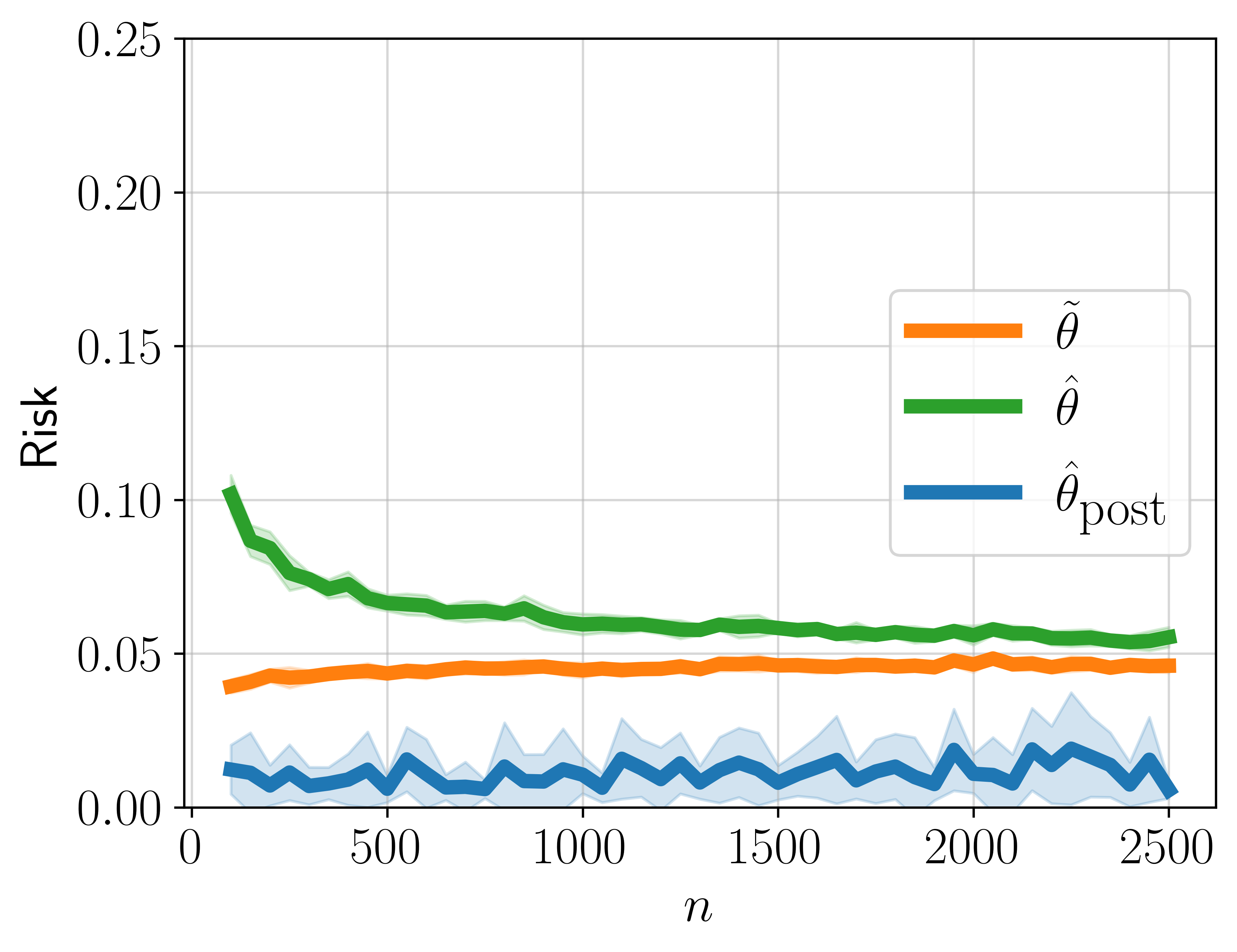}
            \subcaption{Regression risk}
        \end{subfigure}
        \hfill
        \caption*{\textbf{(iii) Task shift for spiked covariance when neither classification nor regression generalize.} We set $p=1.5$, $q=0.9$, and $r=0.5$ so that $(1-r)+(p-1)/2<q<p-r$. In this regime, the classification MNI $\hvtheta$ does not on the original classification problem, and the regression MNI $\tvtheta$ does not generalize on the original regression problems.}
        \label{fig:spiked_q9}
    \end{subfigure}
    \caption{\textbf{Postprocessing achieves task shift in the three regimes of~\cite{muthukumar2021classification}.} The left column demonstrates the survival of $t$-sparse signal support components in the classification MNI $\hvtheta$ while non-support components decay quickly. The middle column shows the $O(\frac{t}{m})$ regression error of least-squares with reduction to $t$ dimensions using $m$ regression samples under standard Gaussian noise. Finally, the right column displays the regression risk of the classification MNI, regression MNI, and our postprocessed predictor. The signal $\svtheta$ is $2$-sparse with $a_1=0.2$ and $a_2=-0.1$ (see Assumption~\ref{asm:k_sparse}). The middle column fixes $n=2500$. We plot the mean and standard deviation over $10$ draws of the training dataset $\rmX$.}
    \label{fig:spike2}
\end{figure*}

\end{document}